%% file: H2T_main.tex
\documentclass[10pt,journal,compsoc]{IEEEtran}

\usepackage[utf8]{inputenc} 
\usepackage[T1]{fontenc}    

\usepackage{algorithm,algorithmic}

\usepackage{makecell}
\usepackage{threeparttable}

\usepackage{subfigure}
\usepackage{epsfig}

\usepackage{booktabs}       
\usepackage{enumitem}

\usepackage[dvipsnames, svgnames, x11names, table]{xcolor}
\usepackage{arydshln}
\usepackage{colortbl}
\usepackage{amsmath}
\usepackage{amssymb}
\usepackage{amsthm}
\usepackage{url}
\usepackage{ulem}
\usepackage{tikz}
\usepackage{multirow}

\definecolor{red2}{RGB}{240,191,211}
\definecolor{deepred}{RGB}{219,95,146}
\definecolor{blue2}{RGB}{205,226,247}
\definecolor{deepblue}{RGB}{88,161,230}

\theoremstyle{plain}
\newtheorem{theorem}{Theorem} 

\newtheorem{definition}{Definition} 
\newtheorem{remark}{Remark}

\newcommand{\redtext}[1]{\textcolor{black}{#1}} 

\definecolor{mygray}{gray}{.9}

\makeatletter
\def\hlinew#1{%
\noalign{\ifnum0=`}\fi\hrule \@height #1 \futurelet
\reserved@a\@xhline}
\makeatother%

\begin{document}
\title{PI-H2T: Enhancing Long-Tailed Visual Recognition with Permutation-Invariant and Head-to-Tail Feature Fusion}

\author{Mengke Li,
        Zhikai Hu,
        Yang Lu,~\IEEEmembership{Member,~IEEE,}
        Weichao Lan,
        Yiu-ming~Cheung,~\IEEEmembership{Fellow,~IEEE,}\\
        Hui Huang*,~\IEEEmembership{Senior Member,~IEEE}
\IEEEcompsocitemizethanks{
\IEEEcompsocthanksitem This work was supported in part by Guangdong S\&T Program (2024B0101050004), NSFC (62306181), Guangdong Basic and Applied Basic Research Foundation (2024A1515010163), National Key Laboratory of Radar Signal Processing (JKW202403), Shenzhen S\&T
Program (RCBS20231211090659101, KJZD20240903100022028), RGC GRF (12201323, 12202924), and Guangdong Provincial Key Laboratory of Visual Media and Multidimensional Intelligence.
\IEEEcompsocthanksitem Mengke Li is with the Guangdong Provincial Key Laboratory of Visual Media and Multidimensional Intelligence, CSSE, Shenzhen University, China (E-mail: csmengkeli@gmail.com). 

\IEEEcompsocthanksitem Yang Lu is with Fujian Key Laboratory of Sensing and Computing for Smart City, School of Informatics, Xiamen University, China (E-mail: luyang@xmu.edu.cn).

\IEEEcompsocthanksitem Zhikai Hu and Yiu-ming Cheung are with the Department of Computer Science, Hong Kong Baptist University, China (E-mail: cszkhu@comp.hkbu.edu.hk; ymc@comp.hkbu.edu.hk). 

\IEEEcompsocthanksitem Weichao Lan is with Inspur Smart City Technology Co., Ltd., China (E-mail: lweichao@aliyun.com).

\IEEEcompsocthanksitem Hui Huang is the corresponding author with the Guangdong Provincial Key Laboratory of Visual Media and Multidimensional Intelligence, CSSE, Shenzhen University, China (E-mail: hhzhiyan@gmail.com).}}

\markboth{Submitted to IEEE Transactions on Pattern Analysis and Machine Intelligence}{M. Li \MakeLowercase{\textit{et al.}}: Long-Tailed Visual Recognition via Permutation-Invariant Head-to-Tail Feature Fusion}

\IEEEtitleabstractindextext{%
\begin{abstract}
The imbalanced distribution of long-tailed data presents a significant challenge for deep learning models, causing them to prioritize head classes while neglecting tail classes. 
Two key factors contributing to low recognition accuracy are the deformed representation space and a biased classifier, stemming from insufficient semantic information in tail classes.
To address these issues, we propose permutation-invariant and head-to-tail feature fusion (PI-H2T), a highly adaptable method. 
PI-H2T enhances the representation space through permutation-invariant representation fusion (PIF), yielding more clustered features and automatic class margins.
Additionally, it adjusts the biased classifier by transferring semantic information from head to tail classes via head-to-tail fusion (H2TF), improving tail class diversity.
Theoretical analysis and experiments show that PI-H2T optimizes both the representation space and decision boundaries. 
Its plug-and-play design ensures seamless integration into existing methods, providing a straightforward path to further performance improvements. 
Extensive experiments on long-tailed benchmarks confirm the effectiveness of PI-H2T.
The source code is available at \url{https://github.com/Keke921/H2T}.
\end{abstract}
}

\maketitle

\input{Sec/1_Intro}
\input{Sec/2_RW}

\input{Sec/3_Method}

\input{Sec/4_Experiment}

\input{Sec/5_Conclusion}

\normalem
\bibliographystyle{IEEEbib}
\bibliography{H2T_ref}

\begin{IEEEbiography}[{\includegraphics[width=1in,height=1.25in,clip,keepaspectratio]{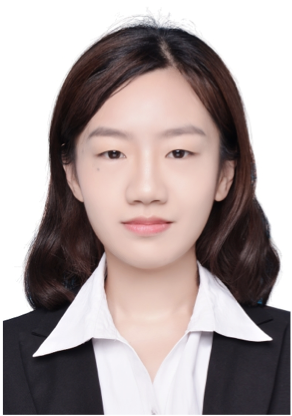}}]{Mengke Li}
received the B.S. degree in Communication Engineering from Southwest University, Chongqing, China, in 2015, the M.S. degree in electronic engineering from Xidian University, Xi’an, China, in 2018, and the Ph.D. degree from the Department of Computer Science, Hong Kong Baptist University, Hong Kong SAR, China, under the supervision of Prof. Yiu-ming Cheung, in 2022. She is currently an Assistant Professor with College of Computer Science and Software Engineering, Shenzhen University, Shenzhen, China. Her current research interests include imbalanced data learning, long-tail learning, and pattern recognition.
\end{IEEEbiography}

\begin{IEEEbiography}[{\includegraphics[width=1in,height=1.25in,clip,keepaspectratio]{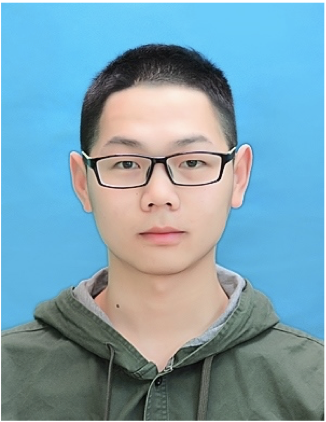}}]{Zhikai Hu}
received his Ph.D. degree from the Department of Computer Science, Hong Kong Baptist University, Hong Kong SAR, China, under the supervision of Prof. Yiu-ming Cheung, in 2025. He is the awardee of RGC Junior Research Fellow Scheme (2025/26). He is currently a post-doctoral research fellow with the Department of Computer Science, Hong Kong Baptist University, Hong Kong SAR, China. He has published over 20 articles in high-quality conferences and journals, including TPAMI, TNNLS, TCSVT, AAAI, and so on. His present research interests include information retrieval and AI for Art \& Culture.
\end{IEEEbiography}

\vspace{-0.5in}
\begin{IEEEbiography}[{\includegraphics[width=1in,height=1.25in,clip,keepaspectratio]{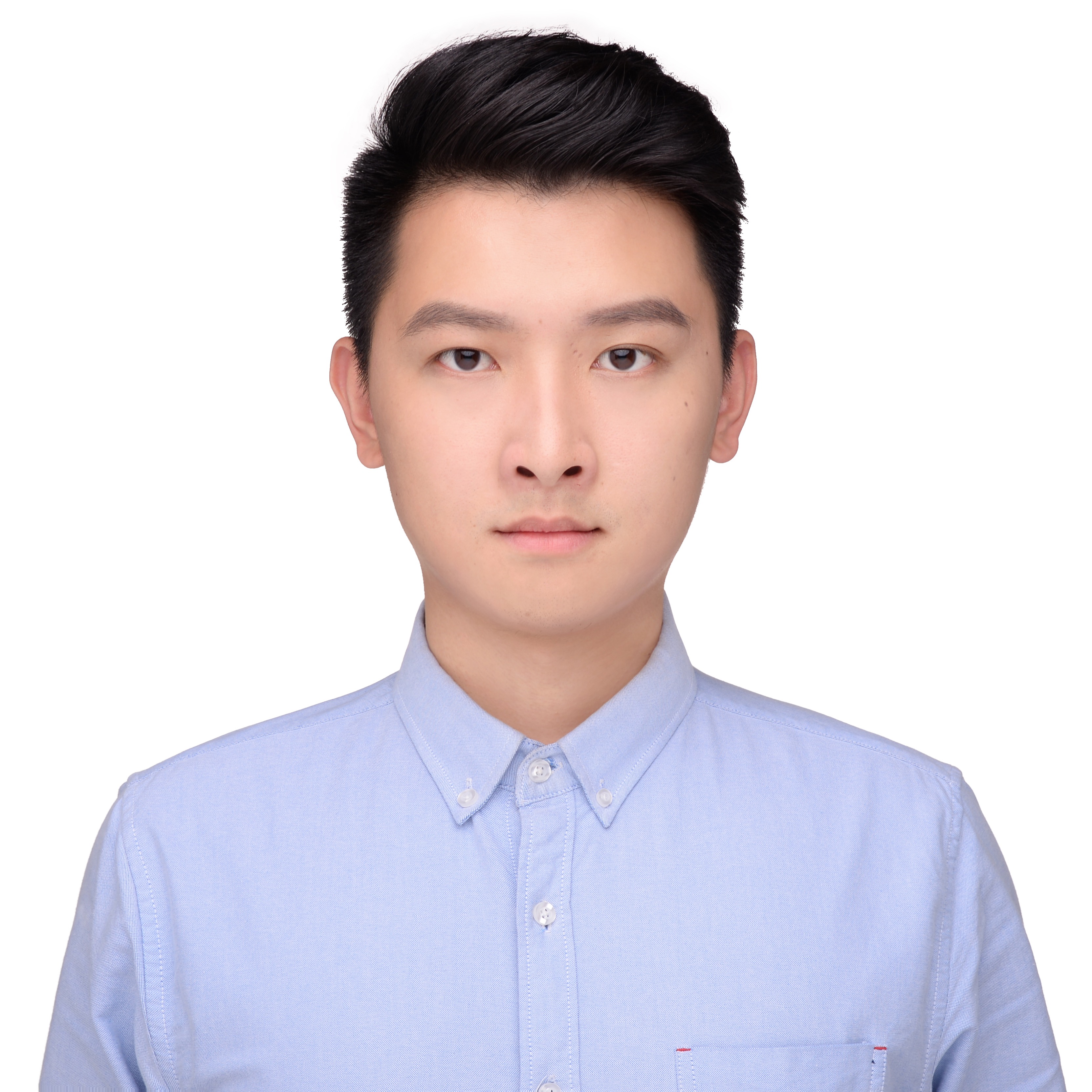}}]{Yang Lu}
received the B.Sc. and M.Sc. degrees in Software Engineering from the University of Macau, Macau, China, in 2012 and 2014, respectively, and the Ph.D. degree in computer science from Hong Kong Baptist University, Hong Kong, China, in 2019. He is currently an Assistant Professor with the Department of Computer Science, School of Informatics, Xiamen University, Xiamen, China. His current research interests include deep learning, federated learning, long-tail learning, and meta-learning.
\end{IEEEbiography}

\vspace{-0.5in}
\begin{IEEEbiography}[{\includegraphics[width=1in,height=1.25in,clip,keepaspectratio]{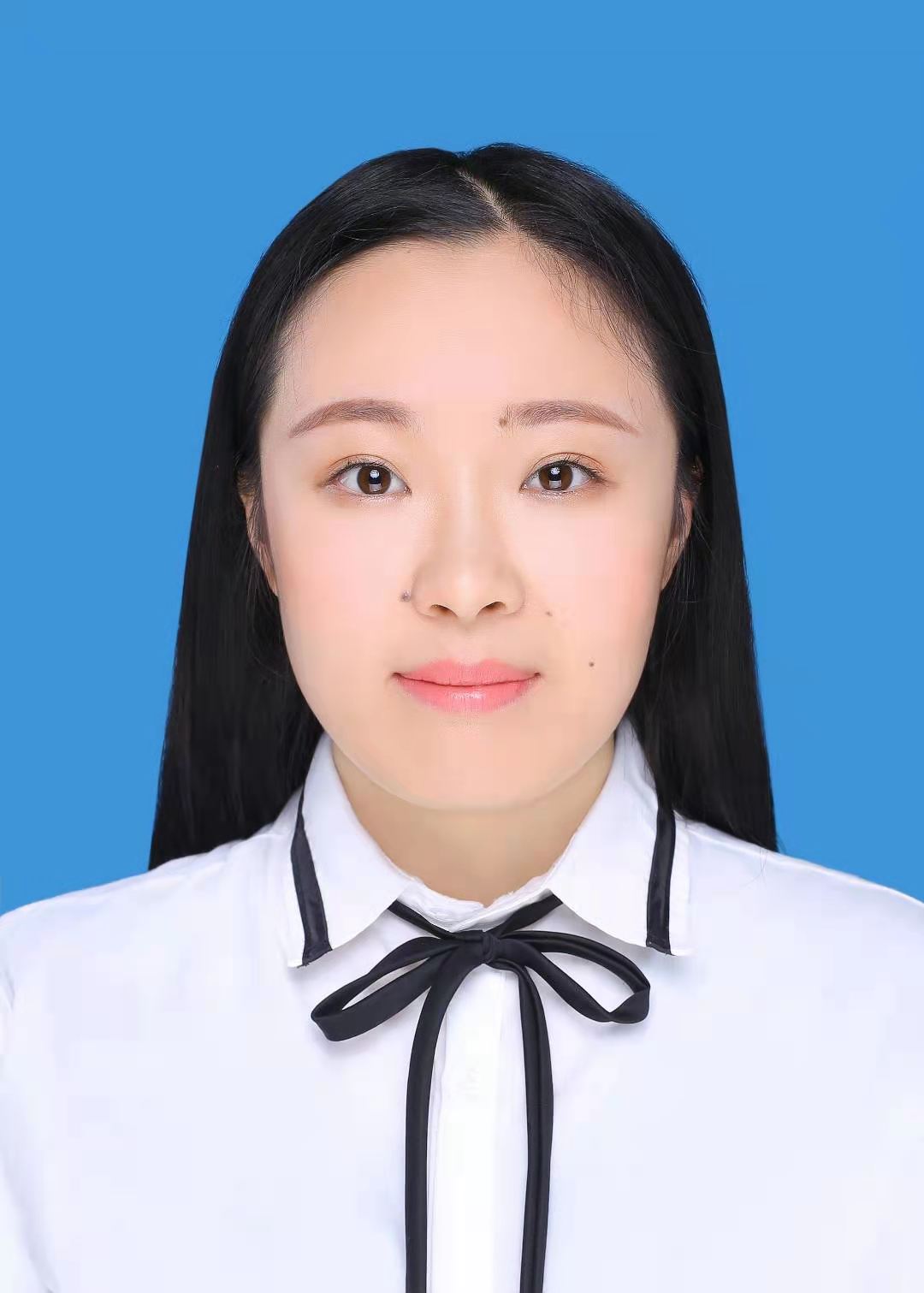}}]{Weichao Lan}
received her B.S. degree in Electronics and Information Engineering from Sichuan University in 2019, and the Ph.D degree in Computer Science from Hong Kong Baptist University, under the supervision of Prof. Yiu-ming Cheung, in 2024. Her present research interests include network compression and acceleration, and lightweight models.
\end{IEEEbiography}

\vspace{-0.5in}
\begin{IEEEbiography}[{\includegraphics[width=1in,height=1.25in,clip,keepaspectratio]{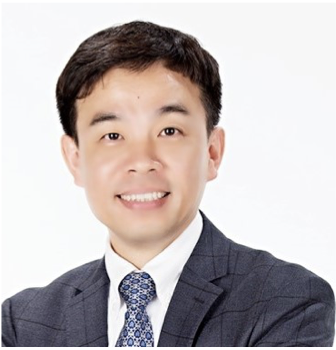}}]{Yiu-ming Cheung}
(Fellow'18, IEEE) received the Ph.D. degree in 2000 from the Department of Computer Science and Engineering, The Chinese University of Hong Kong, Hong Kong. He is currently a Chair Professor with the Department of Computer
Science, Hong Kong Baptist University, Hong Kong. His research interests include Machine Learning and Visual Computing, as well as their applications in Data Science, Pattern Recognition, and Information Security. Prof. Cheung is a Member of the European Academy of Sciences and Arts, and a Fellow of the American Association for the Advancement of Science (AAAS), the International Association for Pattern Recognition (IAPR), the Institution of Engineering and Technology (IET), and the British Computer Society (BCS). Also, he is a recipient of RGC Senior Research Fellow
Award. He is the Editor-in-Chief of the IEEE Transactions on Emerging Topics in Computational Intelligence, and is an Associate Editor for several prestigious journals, including IEEE Transactions on Cybernetics, IEEE Transactions on Cognitive and Developmental Systems, Pattern Recognition, Neurocomputing, to name a few. For more details, please refer to: \url{https://www.comp.hkbu.edu.hk/~ymc}.
\end{IEEEbiography}

\vspace{-0.5in}
\begin{IEEEbiography}[{\includegraphics[width=1in,height=1.25in,clip,keepaspectratio]{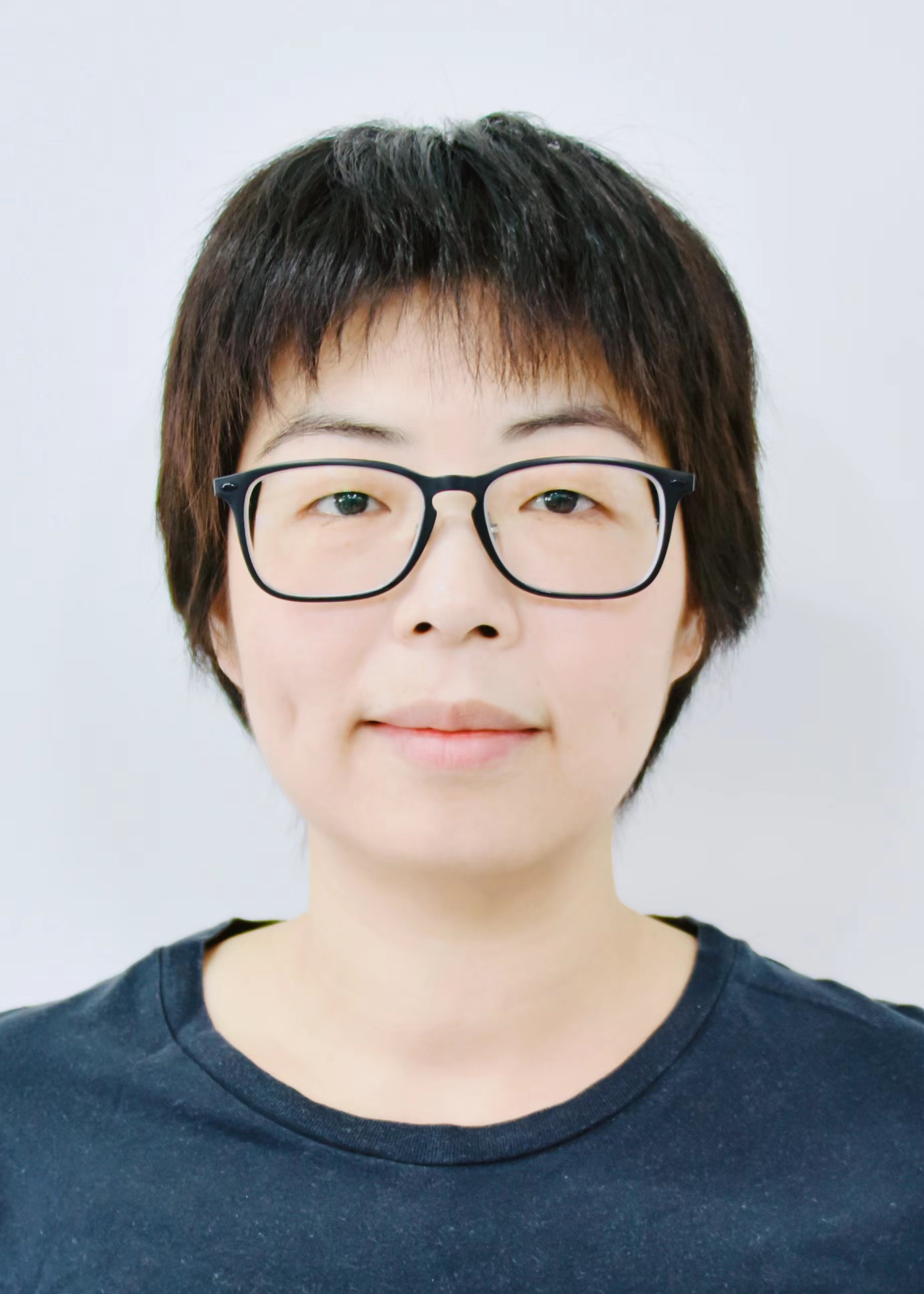}}]{Hui Huang}
received the Ph.D. in Math from the University of British Columbia. She is currently a Chair Professor at Shenzhen University, serving as the Dean of College of Computer Science and Software Engineering while also directing the Visual Computing Research Center. Her research encompasses computer graphics, computer vision and visual analytics, focusing on geometry, points, shapes and images. She was on the editorial board of ACM TOG and IEEE TVCG.
\end{IEEEbiography}


\clearpage
\appendices

\section{Detailed Analysis of the Problem Discussed in Sec.~\ref{sec:motivation}}
\label{app:motivation}

The long-tailed problem stems from two complementary issues.
Firstly, the \textbf{representation issue} is caused by the sparsity of tail class samples, which results in an inadequate representation that renders the model susceptible to overfitting on irrelevant patterns. 
While methods such as multi-expert models~\cite{WangXD21RIDE,liJ2022nested,Jin2023shike} or pretrained priors~\cite{Dong23LPT,shi2024LIFT} can enhance diversity, they are structure-specific and cannot be generalized to different frameworks.
Moreover, the limited sampling of tail classes makes their features more sensitive to spurious cues, such as the order of elements in the input. This order dependency introduces hidden confounders that cause the model to overfit to position-specific or irrelevant semantic patterns.
To address this, we propose incorporating the permutation invariance property into the feature representation, which inherently eliminates reliance on input order and mitigates order-induced biases.
Based on this, we propose a \textbf{Permutation-Invariant Feature Fusion (PIF)} strategy that integrates these features with the original model features to improve representation quality.
This fusion not only improves feature robustness but also provides additional flexibility for subsequent classifier calibration.
The proposed strategy is architecture-agnostic and can be seamlessly integrated into a wide range of backbones with merely two additional trainable parameters.

Secondly, \textbf{classifier bias} arises from the disproportionate number of head and tail class samples in long-tailed data, leading to a tendency to favor head classes during prediction.
The straightforward solution involves increasing the importance or frequencies of tail classes~\cite{Kaidi2019, bbn20,shin2017deep}.
They can increase the performance of tail classes, nevertheless, they also entail an elevated risk of overfitting. 
To mitigate overfitting while refining the decision boundary, it is crucial to enhance the diversity of tail class samples~\cite{zhang2021survey, yang2022survey}. 
Unfortunately, directly obtaining more samples is generally infeasible.
An alternative is to enrich tail classes by fully exploiting the existing feature space.
Typically, misclassified samples are unseen instances of the training set, and these instances tend to be interspersed in the vicinity of the decision boundary. 
This pattern offers an opportunity to enhance tail class diversity by populating class margins with semantically meaningful samples, thereby facilitating the formation of more optimal decision boundaries.
To this end, we propose \textbf{Head-to-Tail Fusion (H2TF)}, which simulates potential unseen samples by directly borrowing semantic information from head classes to augment tail classes. 
This strategy is employed exclusively during the stage-2 classifier adjustment phase, enabling the filling of inter-class margins with more diverse and abundant tail samples.
This, in turn, facilitates classifier calibration within a more balanced and well-structured feature space.

\section{Algorithm of PI-H2T}
\label{app:alg}

The proposed training algorithm is presented in Algorithm~\ref{alg:h2t}, where the key procedures for PIF and H2TF are outlined in lines 11 and 13–14, respectively.

\input{tab/algorithm}

\input{fig/ablation/PIF-tsne}

\section{Additional Visualizations of Decision Boundaries for Baseline Methods}
\label{app:boundary}

\begin{figure}[!htpb]
\begin{minipage}[c]{1.\linewidth} 
    \centering  
    \subfigure[Decision boundary by DR]{\label{fig:CE09} 
    \includegraphics[width=.48\linewidth, height=.42\linewidth]{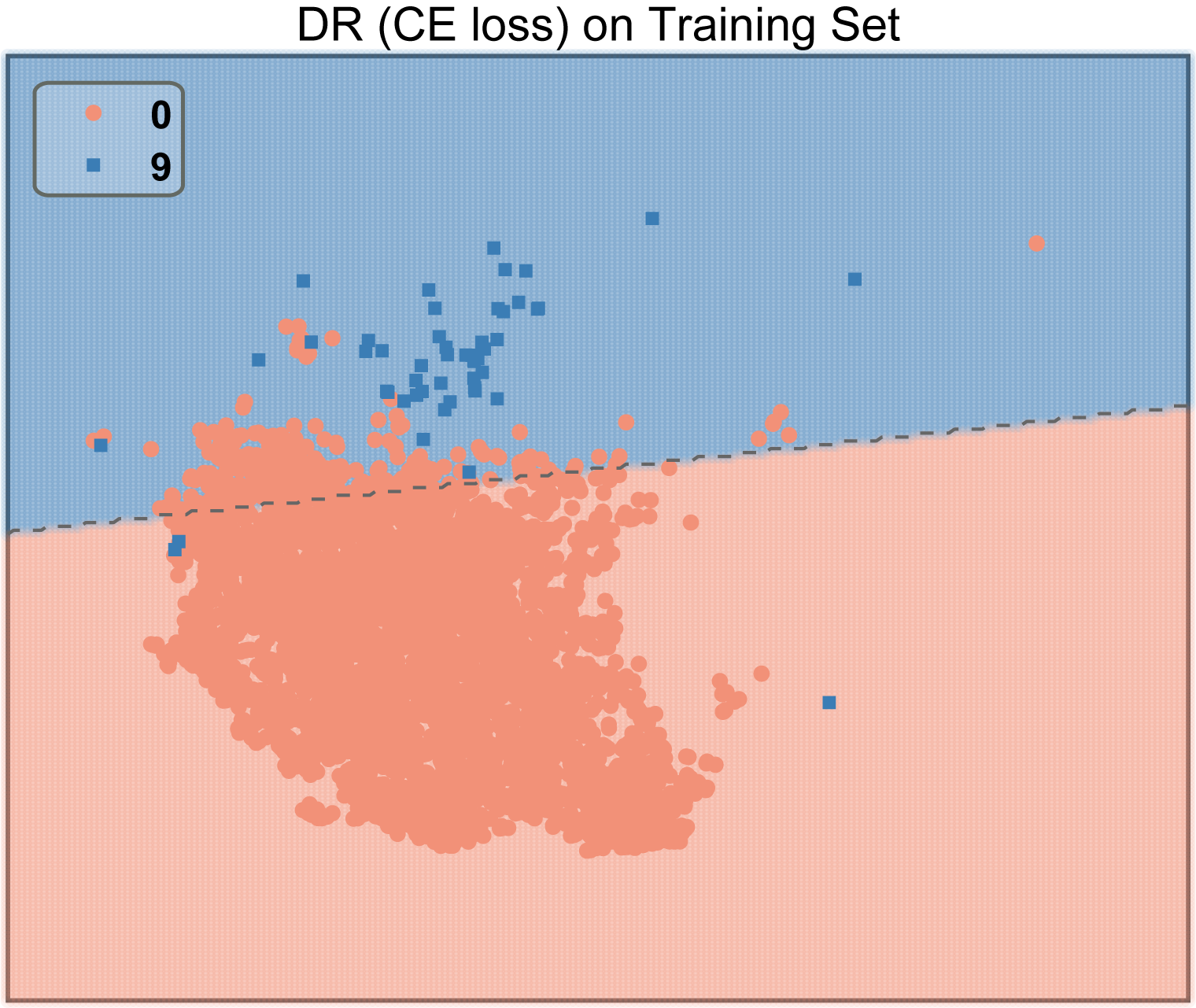} \label{fig:CE-train09} 
    \includegraphics[width=.48\linewidth, height=.42\linewidth]{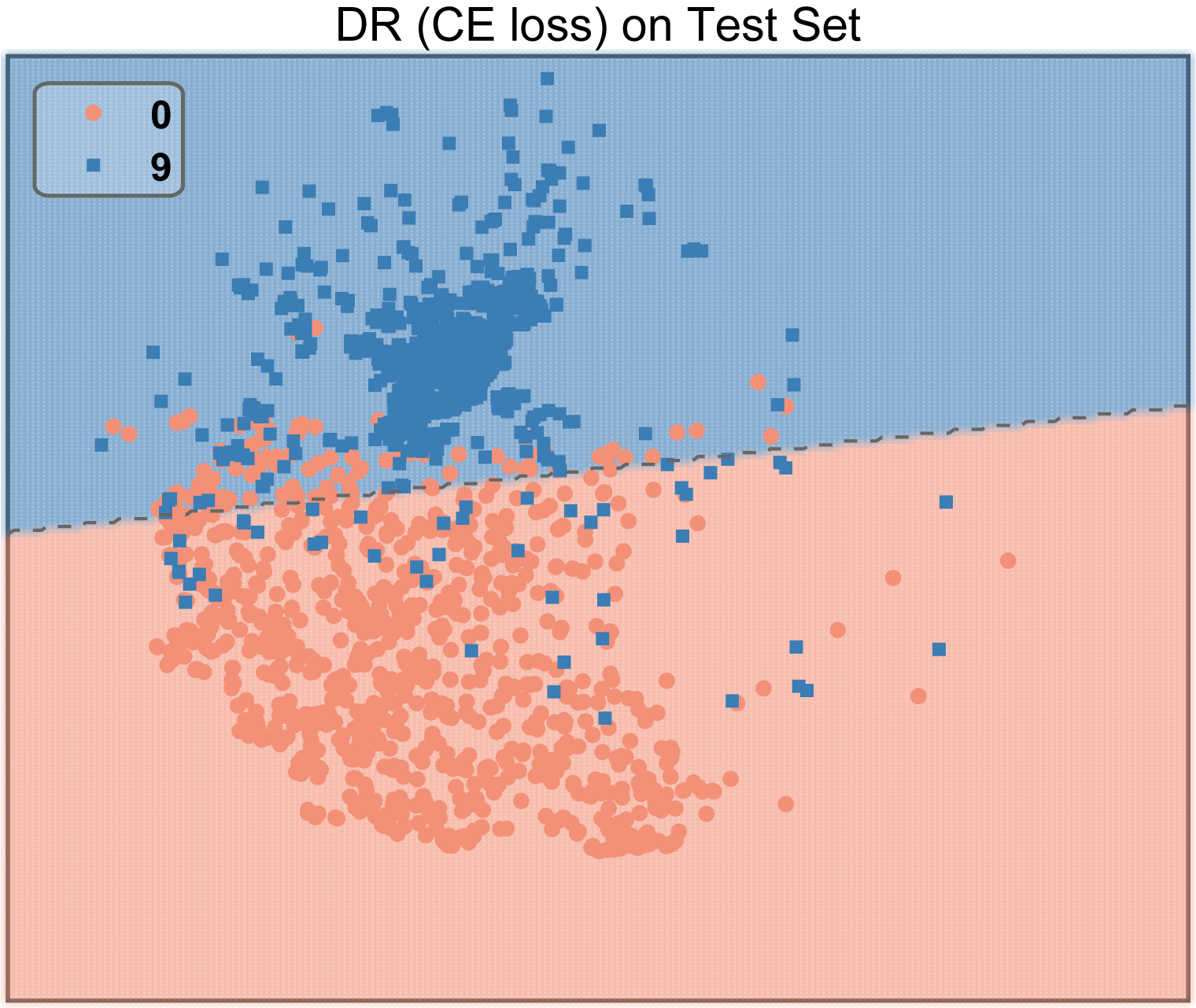} \label{fig:CE-test09}
    }
    \\
    \subfigure[Decision boundary by DR+H2T]{\label{fig:H2T09} 
    \includegraphics[width=.48\linewidth, height=.42\linewidth]{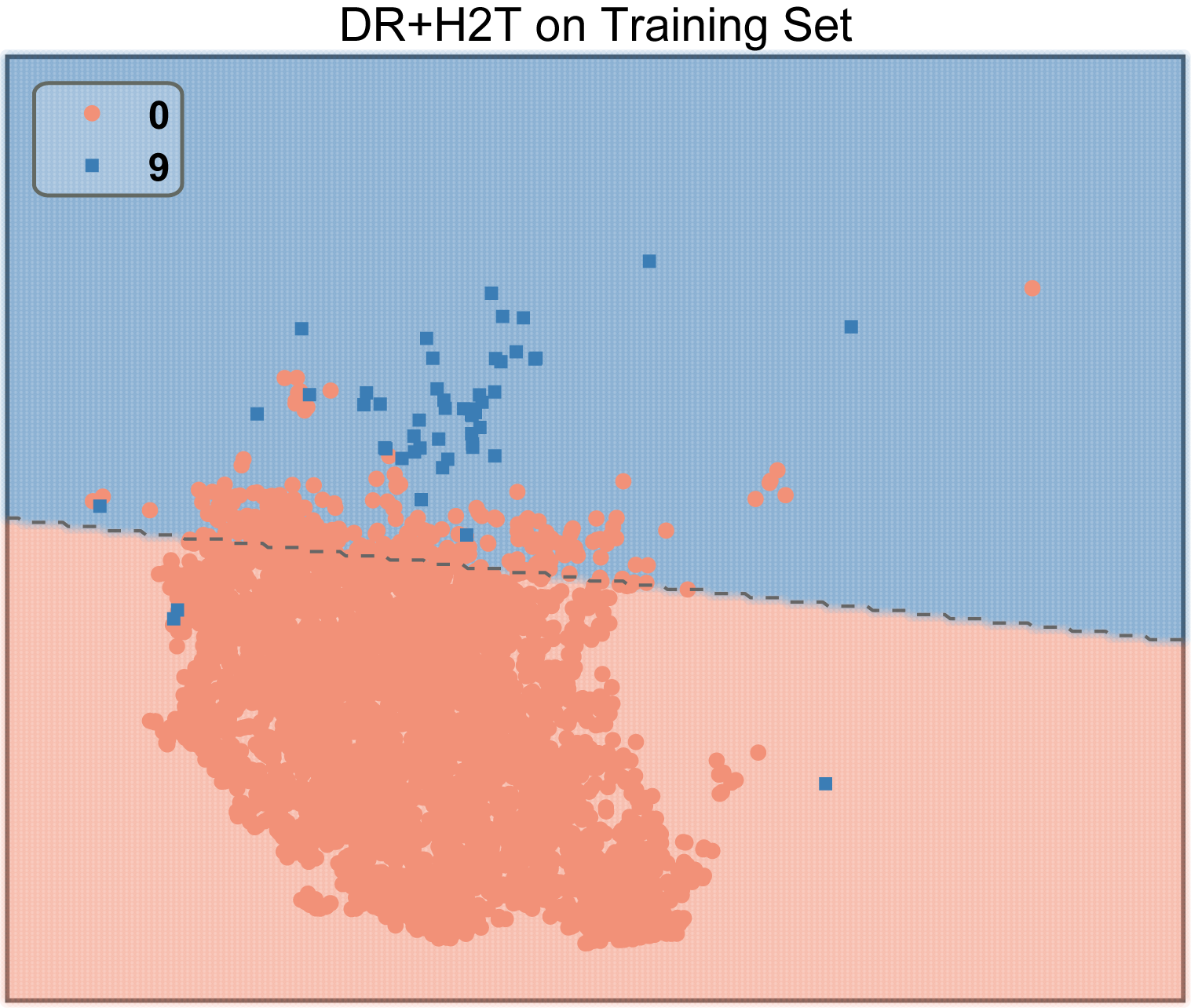} \label{fig:H2T-train09}  
    \includegraphics[width=.48\linewidth, height=.42\linewidth]{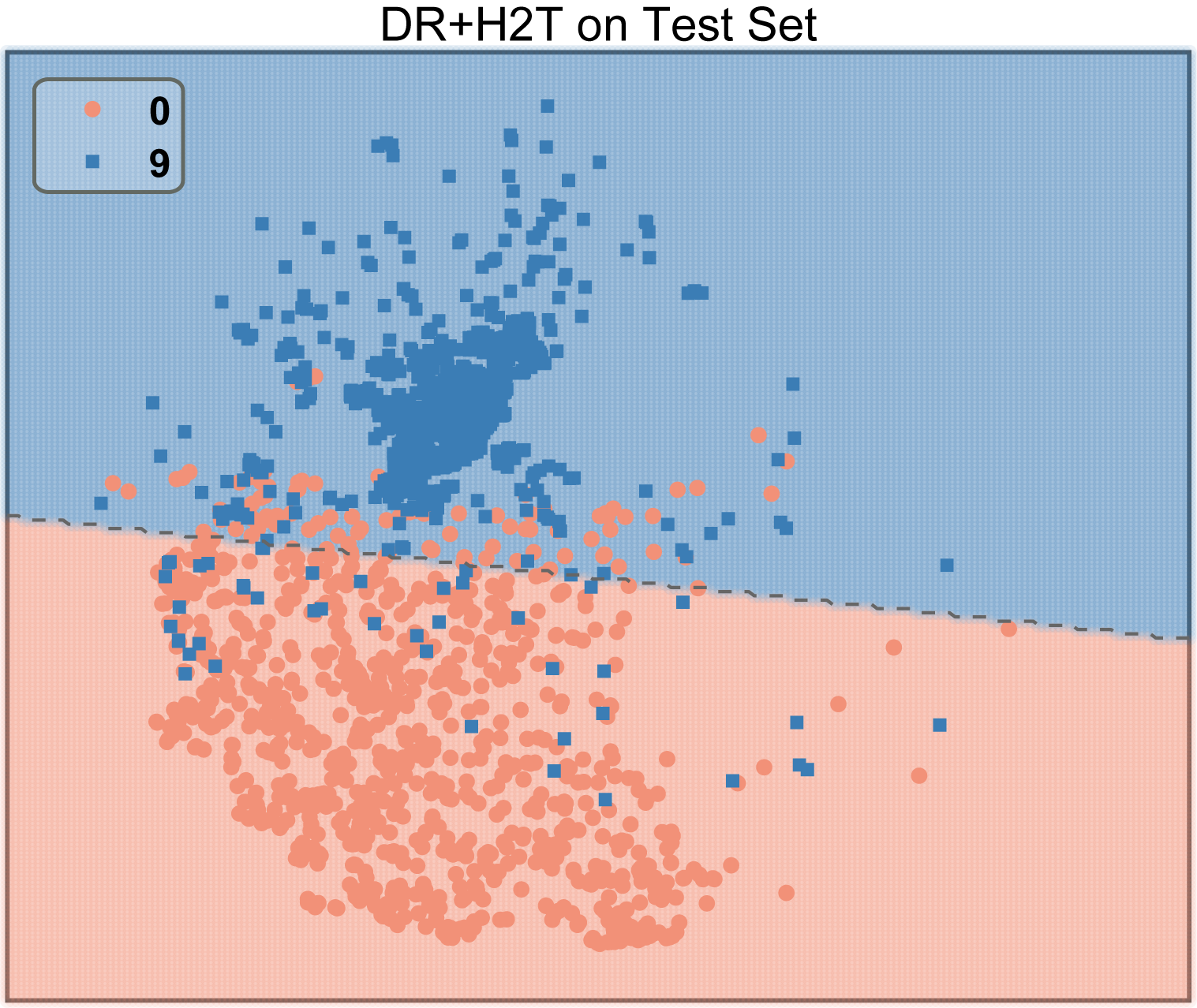} \label{fig:H2T-test09}
    }
\caption{Supplementary visualization of decision boundaries between Class 0 and Class 9 on CIFAR10-LT (imbalance factor 100), corresponding to the discussion in Sec.~\ref{sec:motivation}.}
\label{appfig:boundary09}
\end{minipage}
\end{figure}

H2T does not change the feature distribution, and its strengths and limitations are equally apparent. 
Its primary advantage lies in its simplicity, as it only requires further calibration of the classifier. 
H2T re-optimizes the class boundary distribution within the fixed feature space.
For example, when comparing Fig.~\ref{fig:CE09} and Fig.~\ref{fig:H2T09}, the feature space remains the same, but it is evident that H2T corrects the classification boundary, shifting it in a more optimal direction.

\section{More Feature Visualization for H2TF}
\label{app:vis_tsne_PIH2T}

Fig.~\ref{appfig:tsne_PIH2T} presents additional visualizations of feature distributions under different H2TF fusion configurations. 
We illustrate how varying the fusion ratio affects the representations of head, medium, and tail classes. 
As the fusion ratio decreases, tail-class features become more dispersed, reducing the dominance of head-class features. 
However, overly reducing the self-feature component can introduce semantic noise and local confusion. 
It is important to emphasize that H2TF is applied exclusively during the classifier rectification phase.
During inference, the backbone features are kept fixed, and no feature fusion occurs, thus preserving the learned representations.

\section{Effectiveness of PIF on Balanced Datasets}

\input{tab/cifar-bal}

\label{app:balance_data}
To further validate the effect of PIF, we also conducted experiments on balanced datasets (Table~\ref{tab:cifar_bal}). 
It can be seen that PIF also improves the performance of the model on balanced datasets, demonstrating the effect of the automatic margin introduced by PIF.

\end{document}

%% file: Sec/1_Intro.tex
\section{Introduction} \label{sec:intro}
Deep learning confronts the dilemma of training models on data with substantial class imbalance, named long-tailed data, which is characterized by an abundance of samples from a few classes (head classes) and a scarcity of samples from others (tail classes)~\cite{shi2024LIFT}.
This imbalance impedes model performance, particularly affecting the recognition accuracy of tail classes. 
Consequently, this issue has emerged as a critical bottleneck constraining the progress of deep learning models.
As long-tailed data are more applicable to the empirical data distributions encountered in the real world~\cite{reed2001pareto, LiuZW19LTOW, zhang2021survey}, long-tailed visual recognition has attracted considerable attention recently.

To alleviate the issue of severe class imbalance within long-tailed data, numerous methods have been proposed in recent years. 
These approaches can primarily be classified into three categories~\cite{shi2024LIFT, li2022advances}: 1) data manipulation~\cite{Nitesh2002SMOTE, Huang2016CVPR,kim2020m2m, Wang2021RSG, Li2021MetaSAug, ParkS2022Majority}, 2) model modification~\cite{bbn20,decouple20,WangXD21RIDE, Wang2020the, LiBL2022Trustworthy, Jin2023shike} and 3) output adjustment~\cite{Tsung2020Focal,cui2019class,RenJW2020Balanced,adjustment21,Hong2021CVPR, LiMK2022GCL}. 
Besides training deep neural networks (DNNs) from scratch, recent studies, including BALLAD~\cite{ma2021BALLAD}, VL-LTR~\cite{tian2022vlltr}, LPT~\cite{Dong23LPT} and LIFT~\cite{shi2024LIFT}, to name a few, introduce prior knowledge from foundation models trained on large scale datasets, such as CLIP~\cite{radford2021clip} and ImageNet-21K~\cite{imagenet21k}, to aid the poor performance of models on tail classes.
These methods have yielded remarkable advancements in addressing the challenges of long-tailed data.

\begin{figure}
    \centering
    \includegraphics[width=0.9\linewidth]{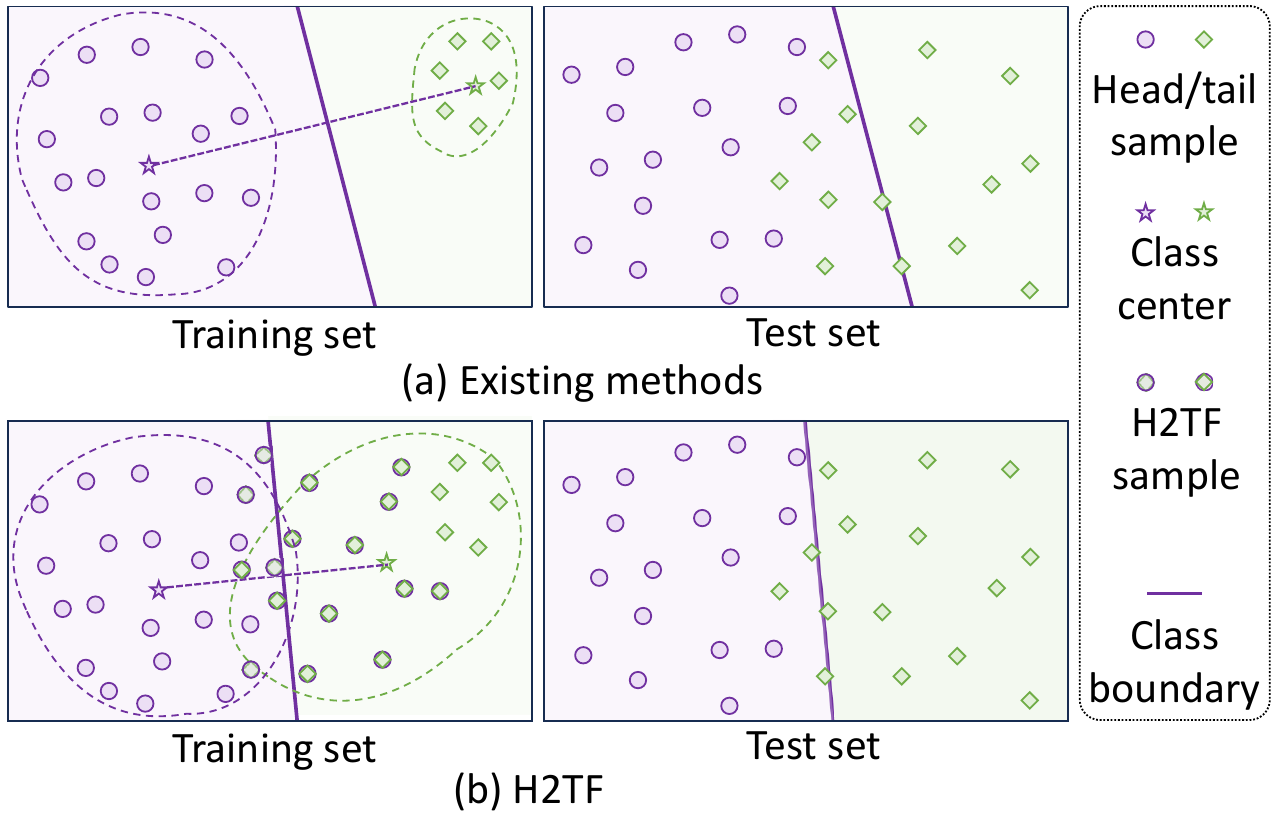}  
    \caption{Comparison between decision boundaries produced by (a) existing methods, and (b) proposed H2TF (stage 2 of PI-H2T).}
\label{fig:dec_bnd}
\vspace{-1em}
\end{figure}

Despite the success of existing work in achieving robust predictions, these approaches generally optimize representation learning and classifier calibration separately, lacking a unified, continuous optimization framework.
Consequently, one aspect of either feature learning or classifier calibration is always left suboptimal.
In the representation space, increasing inter-class margins~\cite{Kaidi2019, Deng2019ArcFace} and/or refining the compactness of the intra-class feature distribution~\cite{wang2018cosface, Wang2017NormFace} have been shown to improve model generalization capabilities~\cite{Zhikai2021MTFH, HuZK23Joint}.
While a clear margin is effective for balanced data, it does not address the scarcity of tail class samples. 
Consequently, the deformation of the representation space and the biased decision boundary in the context of long-tailed data may still persist.
Parameter-efficient fine-tuning (PEFT) techniques leverage extensive pre-trained data to aid the sparse representations of tail classes. 
Nonetheless, many of these large datasets, such as CLIP~\cite{radford2021clip} and JFT-300M~\cite{sun2017revisiting}, are often not publicly accessible~\cite{imagenet21k}.
There remains uncertainty regarding model biases and preferences after the post-fine-tuning, and the decision boundary, primarily influenced by head classes, remains uncalibrated.
Taking the binary case in the embedding space of the obtained backbone as an example, the decision boundary is usually the midline connecting the two class centroids. 
The head classes have been sufficiently sampled, and their embedding space is fully occupied~\cite{xiao2021does}. 
In contrast, the tail class suffers from a scarcity of samples, leading to sparsely distributed semantic regions. 
The bias in the tail centroid persists due to head squeeze, even if the existence of clear margins. 
As a result, during the inference stage, a considerable number of samples that differ from the training set emerge, causing the erosion of well-defined margins.
Consequently, numerous tail class samples are misclassified as head classes, as illustrated in Fig.~\ref{fig:dec_bnd}(a).

In this paper, we propose two progressive and modular enhancements that compact feature distributions and automatically expand class margins, while alleviating the over-compression of the feature space of tail classes by the classifier. 
We term the overall approach Permutation-Invariant and Head-to-Tail Feature Fusion (PI-H2T).
The first component, Permutation-Invariant Feature Fusion (PIF), implicitly enlarges the inter-class margins by refining the feature representation in stage 1, which in turn lays a solid foundation for effective classifier calibration in stage 2.
We incorporate a permutation-invariant (PI) operation that promotes inner-set interactions and reduces feature redundancy. 
The PI features are then fused with the original features to preserve essential order-specific cues (e.g., distinguishing ocean spray from clouds), while enhancing the overall representation quality. 
This margin-aware design ensures that the classifier has sufficient flexibility for effective boundary calibration in stage 2.
At this stage, the second component, Head-to-Tail Fusion (H2TF), is applied to mitigate classifier bias by explicitly augmenting tail class features with semantically rich representations from head classes.
This is achieved by partially substituting tail features with those from head samples, under the assumption that rare tail instances are prone to misclassification due to their sparse and noisy representation. 
Transferring the head semantics can effectively fill the tail semantic area and the category overlap, compelling the decision boundary to shift toward a more balanced and optimal position, as illustrated in Fig.~\ref{fig:dec_bnd}(b). 
As a result, the generalization on tail classes can be improved.

PI-H2T is designed to be lightweight and architecture-agnostic, introducing only negligible additional parameters and computational overhead, thus allowing seamless integration into existing frameworks.
We implement both PIF and H2TF using simple and modular designs.
Specifically, PIF employs channel-wise mean pooling as a permutation-invariant operation, which is parameter-free, and only introduces a minimal number of learnable weights when fusing the PI features with the original backbone features. 
This ensures enhanced representation while maintaining compatibility with diverse architectures.
For H2TF, we propose a straightforward strategy that reuses extracted features without retraining the backbone. 
A fraction of features from class-balanced data is selectively fused with those from head-biased data based on their distances to the class center, allowing tail class features to benefit from enriched head semantics. 
Importantly, H2TF is only applied during training for classifier calibration and is excluded during inference, ensuring that the feature extraction process remains unchanged and free from semantic confusion.
Overall, PI-H2T serves as a practical and flexible module for enhancing long-tailed recognition across various existing methods.

The main contributions of this paper are as follows:
\begin{itemize}[leftmargin=*, labelsep=5pt]
    \item We propose permutation-invariant feature fusion (PIF), which adds merely two parameters to the model, significantly enhancing feature representation and enabling more optimal decision boundaries with minimal computational overhead.
    \item We propose head-to-tail fusion (H2TF), which transfers semantic information from head classes to augment tail class representations without requiring additional data or network parameters. This mechanism effectively leverages the capacity of a well-trained backbone and facilitates the learning of more discriminative decision boundaries.
    \item We devise simple fusion strategies. 
    PIF extracts and integrates PI features with the original ones, while H2TF selectively combines a fraction of the features from an instance-wise sampling branch with those of a class-balanced sampling branch. 
    These strategies enable PI-H2T to be implemented with ease to the original backbone.     
    \item The integration of the proposed PI-H2T module into existing architectures, including single models, multi-expert frameworks, and PEFT methods, leads to substantial performance improvements. 
    Extensive experiments on standard long-tailed benchmarks further demonstrate the versatility and broad applicability of the proposed method across various model paradigms.
\end{itemize}

A preliminary conference version of this work has been published in AAAI 2024~\cite{LiMK2024H2T}. 
In this extended version, we present three major improvements:
\begin{itemize}[leftmargin=*, labelsep=5pt]
\item First, to address the limitation of H2T, which only adjusts the decision boundary without altering the feature distribution, we propose PIF to enhance the representation space, enabling better calibration of decision boundaries.
An insightful analysis of the rationale behind PIF is presented, which offers theoretical support for its utility.
\item Second, in the H2TF stage of PI-H2T, we refine the randomly selected feature maps by fusing them with features after pooling to eliminate the need for explicit feature map selection. 
Additionally, unlike the manual hyperparameter selection of the fusion ratio in the conference version of H2T, the fusion ratio in H2TF is automatically determined by the distance between the feature and its corresponding class center, eliminating the need for manual tuning.
\item Third, compared to H2T, we generalize PI-H2T from deep neural networks (DNNs) to the multi-head self-attention (MHSA) mechanism, demonstrating its versatility.
\end{itemize}
\redtext{Section~\ref{sec:resluts} presents extensive experiments on multiple long-tailed benchmarks under CE and GCL settings, providing a comprehensive evaluation of the proposed framework and demonstrating overall advantages over H2T.
}



%% file: Sec/2_RW.tex
\section{Related Work}
\label{sec:related_work}
\subsection{DNN-based Long-tailed Learning}
Deep neural networks (DNNs) have achieved substantial progress in long-tailed visual recognition over the past decade. 
From the perspective of data processing within deep models, existing methods can be broadly categorized into three levels: input-level, representation-level, and output-level approaches.
At the \textbf{input level}, data manipulation is implemented through re-weighting/sampling~\cite{cui2019class} and/or data augmentation techniques~\cite{CubukED19AutoAugment, CubukED20Randaugment, chu2020feature,Perrett2023CVPR}. 
These methods have been proven to be effective in increasing the classification accuracy for DNNs and also have demonstrated their efficacy on long-tailed data~\cite{RenJW2020Balanced, CuiJQ2021parametric, liJ2022nested}.
At the \textbf{representation level}, modifications to the model structure are undertaken to yield a superior representation of features.
For example, decoupling representation based methods~\cite{decouple20, mislas21} and BBN based methods~\cite{bbn20,DisAli21} decouple representation learning and classifier training. 
These methods initially acquire representations from the original long-tailed data. 
Subsequently, they retrain the classifier utilizing either class-balanced sampling data~\cite{decouple20} or reverse sampling data~\cite{bbn20}.
Ensembling learning encompasses redundant ensembling~\cite{WangXD21RIDE, LiBL2022Trustworthy,liJ2022nested, Cai2021ACE}, which aggregates separate classifiers or networks in a framework of multiple experts, and complementary ensembling~\cite{bbn20, Cui2022reslt, XiangLY2020LFME}, which involves statistical selection of different data divisions. 
Studies have demonstrated that ensembling methods, particularly redundant ensembling, can achieve state-of-the-art performance and generate more robust predictions by reducing model variance~\cite{LiBL2022Trustworthy, WangXD21RIDE} and/or increasing data diversity~\cite{liJ2022nested, LiY2020Overcoming, XiangLY2020LFME}. 
At the \textbf{output level}, existing methods enhance model representation and adjust the classifier by calibrating the logit of the model according to certain criteria.
For example, Re-margining methods~\cite{Kaidi2019, LiMK2022KPS, LiMK2022GCL, Wang2021Seesaw} introduce class-size-based constants to create larger margins for tail classes compared to head classes, improving the separability of tail classes. 
This helps alleviate overfitting in minority classes and enhances model generalization.

\subsection{MHSA-based Long-tailed Learning}
Multi-head self-attention (MHSA) mechanisms have recently gained significant traction, leading to substantial performance improvements in various computer vision applications.
Visual Transformers (ViT)~\cite{Dosovitskiy21vit} is the first to introduce MHSA-based transformer architectures~\cite{vaswani2017attention} into the field of computer vision.
Building on these advancements, LiVT~\cite{Xu2023Learning} further investigates the training of ViTs from scratch on long-tailed datasets by incorporating masked generative pertaining in the first stage and a balanced binary cross-entropy (Bal-BCE) loss in the second stage.
Recent advances in CV have leveraged the power of \textbf{pre-trained MHSA-based models}, exemplified by ViT pre-trained on the visual-only dataset ImageNet-21K~\cite{imagenet21k}, and the visual-linguistic model CLIP~\cite{radford2021learning}.
Different from the traditional paradigm of training models from scratch, the recently proposed PEFT techniques~\cite{jia2022visual, chen2022adaptformer,hu2021lora} have been adopted in long-tailed learning, such as RAC~\cite{Long2022RAC}, VL-LTR~\cite{tian2022vl}, LPT~\cite{Dong23LPT}, and LIFT~\cite{shi2024LIFT}, to name a few.
These methods demonstrate that fine-tuning pre-trained models with informative priors can significantly improve the performance on long-tailed visual recognition tasks. 
For instance, 
LPT fine-tunes the ViT model pre-trained on ImageNet-21K using visual prompt tuning~\cite{jia2022visual}, implementing a two-stage training strategy for class-specific and class-shared prompts.
LIFT reveals that heavy fine-tuning can actually degrade performance on tail classes.
It proposes initializing the classifier with linguistic representation from the text encoder in CLIP and fine-tuning ViT with a few epochs. 
Nevertheless, it is worth noting that their performance in tail classes still exhibits inferior results compared to that in head classes.

\redtext{
\subsection{General Long-Tailed Techniques}
General long-tailed techniques can be applied to various backbone architectures and generally include two types.
\textbf{Data augmentation} increases data diversity, thereby improving model representation for tail classes. 
Classical techniques includes flipping, rotating, cropping, and padding~\cite{he2016deep}. 
Recent attempts adapt augmentation strategies based on class distributions in long-tailed settings.
For example, CUDA~\cite{ahn2023cuda} learns class-specific augmentation strengths.
DODA~\cite{wang2024kill} maintains an augmentation distribution to select strategies per class. 
MixUp~\cite{Hongyi2018} and CutMix~\cite{YunSD2019cutmix} combine image–label pairs and are widely used. 
Bag of tricks~\cite{zhang2021bag} and MiSLAS~\cite{mislas21} applies MixUp during representation learning to enhance feature quality. 
GLMC~\cite{du2023global} synthesizes mixed head–tail samples utilizing both MixUp and CutMix. 
Learning from Neighbors~\cite{zhao2025learning} enhances dataset granularity by leveraging LLMs to find auxiliary categories and retrieve relevant images online.
}
\redtext{
\noindent\textbf{Logit adjustment} methods~\cite{LiMK2022GCL, adjustment21, xu2021towards,Kaidi2019, Aimar2023BalPoE} modify class logit to balance feature space distribution and reduce classifier bias. 
For example, LDAM loss~\cite{Kaidi2019} subtracts a class-specific margin from the ground-truth logit, assigning larger margins to tail classes. 
BalPoE~\cite{Aimar2023BalPoE} combines multiple adjusted logits with different strengths to collaboratively optimize the model.
Such methods are widely adopted in both DNN- and MHSA-based long-tailed learning, demonstrating strong performance. 
For example, LIFT~\cite{shi2024LIFT} uses Logit Adjustment~\cite{adjustment21} as its loss function, while LPT~\cite{Dong23LPT} and GNM-PT~\cite{Li2024GNM} employ GCL~\cite{LiMK2022GCL}. 
DBM~\cite{son2025difficulty} extends this approach by assigning larger margins to more difficult samples, functioning as an additive method compatible with various backbones.
}

%% file: Sec/3_Method.tex
\section{Methodology}
\label{sec:method}
\subsection{Preliminaries}
We formally define the basic notation used in this paper before going into detail about our proposed method.
Let $\{x,y\}$ represent one input image and its corresponding label. 
The total number of classes is denoted by $C$, therefore, we have $y \in \{0, 1, \cdots, C-1\}$. 
The training set includes $N$ samples. 
Suppose that class $i$ has $n_i$ training samples. 
Then $N=\sum_i n_i$. 
For simplicity, we suppose $n_0 \geq n_1 \geq \cdots \geq n_{C-1}$. 
Feeding $x$ into the backbone, we can obtain its feature maps before the last pooling layer, denoted $\mathcal{F} = \left[F_{0}, F_{1}, \cdots, F_{d-1} \right] \in \mathbb{R}^{w_{F} \times h_{F} \times d}$, where $w_{F}$ and $ h_{F}$ represent the width and height of the feature map, respectively.
$d$ is the dimension of features in the embedding space. 
The representation after the last pooling layer is ${\mathbf{f}} \in \mathbb{R}^{d} $.
The weight of the linear classifier is represented as $\mathbf{W}=[\mathbf{w}_0, \mathbf{w}_1, \cdots, \mathbf{w}_{C-1}] \in \mathbb{R}^{d \times C} $, where $\mathbf{w}_i$ represents the classifier weight for class $i$. 
We use the subscripts $h, m$, and $t$ to indicate head, medium, and tail classes, respectively. 
$z_i=\mathbf{w}_i^T\mathbf{f}$ represents the predicted logit of class $i$, where the subscript $i=y$ denotes the target logit and $i\neq y$ denotes the non-target logit.
An overview of the notational conventions is presented in Table~\ref{tab:symbols}.
\input{tab/sym_tab}

\subsection{Problem Analysis and Method Overview} \label{sec:motivation}
We progressively solve two pivotal challenges in long-tailed learning: 1) the \textbf{deformed representation space}, and 2) the \textbf{biased classifier}.
In the representation learning stage, we learn representative features while automatically reserving an adequate and appropriate margin to facilitate the consequent correction of classifier bias.

\redtext{
First, sparse tail-class samples lead to poor representations that are prone to overfit spurious patterns, including order-dependent artifacts. 
To address this, we introduce Permutation-Invariant Feature fusion (PIF), an architecture-agnostic module that removes input-order dependence and fuses permutation-invariant features with the original embeddings. 
PIF improves feature robustness and provides a lightweight (two-parameter) mechanism that aids downstream classifier calibration.
}

\redtext{
Second, classifier bias arises from the imbalance between head and tail classes, which drives predictions toward head classes. 
Simply up-weighting tail classes can improve tail accuracy but risks overfitting. 
Instead, we increase effective tail diversity by filling inter-class margins with semantically meaningful samples borrowed from head classes. 
Concretely, Head-to-Tail Fusion (H2TF) augments tail-class support during the stage-2 classifier adjustment, thereby producing better-calibrated decision boundaries without changing the backbone.
}

\redtext{
A detailed analysis is provided in Appendix A.} 

\subsection{Permutation-Invariant Feature Fusion}
\label{sec:PIF_method}
We incorporate the concept of permutation invariance, as introduced in prior works~\cite{NIPS2017DeepSets,NIPS2019DeepSetNet}, to enhance the robustness and generality of feature representations.
\begin{definition} A function $g$: $\mathbb{R}^ {d_1 \times d} \rightarrow \mathbb{R}^{d_2}$ is \textbf{permutation-invariant} iff it satisfies: 
\begin{equation}
    g(\mathbf{X}) = g(\mathbf{PX}),
\end{equation}
for any permutation matrices $\mathbf{P}  \in \mathbb{R}^{d_1 \times d_1}$.
\end{definition}

In the proposed framework, we begin by reshaping the input feature map $\mathcal{F} \in \mathbb{R}^{w_F \times h_F \times d}$ into a 2D matrix $\mathcal{F}'\in \mathbb{R}^ {w_F \cdot h_F \times d}$.
We then apply a permutation-invariant function $g_{PI}: \mathbb{R}^ {w_F \cdot h_F \times d} \rightarrow \mathbb{R}^{w_F \cdot h_F} $ to obtain a compact set of permutation-invariant (PI) features $F_{PI} $:
\begin{equation}
    F'_{PI} = g_{PI} (\mathcal{F'}).
\end{equation}
Subsequently, we reshape $F'_{PI}$ to make $F_{PI} \in \mathbb{R}^{w_F \times h_F }$ and then fuse the obtained PI features with the original feature maps $\mathcal{F}$. 
To achieve more stable and rapid convergence, we model the fusion of PI features by mimicking the residual structure of ResNet~\cite{he2016deep}. 
The fusion process is as follows:
\begin{equation}\label{eq:ini_fuse}
    \mathcal{F}_{fuse} = a \cdot (\mathcal{F}-F_{PI})+b\cdot \mathcal{F},
\end{equation}
where $a$ and $b$ are learnable parameters. 
\redtext{The fused feature $\mathcal{F}_{fuse}$ is subsequently passed through a pooling layer for classification.}
It is important to note that the pooling operation is also permutation-invariant.
However, this permutation-invariant applies to the spatial elements within each feature map rather than across channels.

For the selection of $g_{PI}$, we utilize the mean operation guided by the following theorem. 
\begin{theorem} 
The mean operation across channels on feature maps is permutation-invariant. 
\end{theorem}
\begin{proof}
The mean operation across the channel on $\mathcal{F'} = \left[F'_{0}, F'_{1}, \cdots, F'_{d-1} \right] \in \mathbb{R}^{w_{F} \cdot h_{F} \times d} $ is defined as:
\begin{equation}
\mathcal{F}_{mean} = \frac{1}{d} \sum_{i=0}^{d-1} F'_{i}.
\end{equation}
Given any permutation $\pi$ of the indices of $F'_{i}$, due to the commutative law of addition, the mean remains unchanged, which can be expressed as:
\begin{equation}
\frac{1}{d} \sum_{i=0}^{d-1} F'_{\pi(i)} = \frac{1}{d} \sum_{i=0}^{d-1} F'_{i} .
\end{equation}
\end{proof}

The learnable parameters in Eq.~(\ref{eq:ini_fuse}) can be implemented using a 1D convolution kernel, allowing the module to be seamlessly integrated into existing models as a standalone layer.
Accordingly, Eq.~(\ref{eq:ini_fuse}) can be reformulated as:
\begin{equation}\label{eq:PIF}
    \mathcal{F}_{fuse} = \texttt{Conv1D}\left( \texttt{Concat}\{\mathcal{F}- F_{PI}, \mathcal{F}\} \right),
\end{equation}
where $\texttt{Conv1D}$ denotes 1D convolution operation, and $\texttt{Concat}$ represents the concatenation of the inputs. 
\redtext{In practice, $\texttt{Conv1D}$ is implemented with a kernel size of 1, an input channel size of 2 (corresponding to the concatenated features), and an output channel size of 1, resulting in only two learnable parameters without bias (or three if the bias is included), thereby explicitly realizing the parameters in Eq.~(\ref{eq:ini_fuse}).
}
The only trainable parameters are those in $\texttt{Conv1D}$.
Since the convolution kernel is lightweight, it introduces negligible additional parameters and incurs minimal computational overhead, making the module efficient and scalable.
Finally, $\mathcal{F}_{fuse}$ is fed into the final pooling layer to obtain the feature representation $\mathbf{f}$.

\subsection{Head-to-Tail Fusion}\label{sec:H2TF_method}
During the stage-2 classifier calibration phase, we fuse features from head classes with those of tail classes to leverage the abundant and closely related semantic information.
This fusion enriches the tail class representations and effectively expands their embedding space, facilitating more balanced classification.
The fusion process is formulated as:
\begin{equation}\label{eq:fh2t}
    \tilde{\mathbf{f}} = r \cdot \mathbf{f}_t  + (1-r)  \cdot \mathbf{f}_h ,
\end{equation}
where $r$ is the fusion ratio. 
$\tilde{\mathbf{f}}$ is then passed through the pooling layer and the classifier to predict the corresponding logits $\mathbf{z} = [z_0, z_1, \cdots, z_{C-1}]$. 
The ground truth label corresponds to the label of the fused sample\footnotemark $f_t$.  
Different loss functions, such as CE loss, MisLAS~\cite{mislas21}, or GCL~\cite{LiMK2022GCL}, to name a few, can be employed. 
The backbone $\phi$ can be single DNN-based models~\cite{he2016deep}, the multi-expert models~\cite{WangXD21RIDE, XiangLY2020LFME} as well as ViT~\cite{Dosovitskiy21vit} based models. 
We leverage the two-stage training paradigm~\cite{decouple20}, with H2TF applied in stage II. 
The execution of H2TF involves two key factors: determining the fusion ratio $r$, and employing a simple yet effective fusion strategy to avoid the additional complexity and computational overhead introduced by explicit sample selection during training.
\footnotetext{We refer to samples with the same label as the fused samples, while samples with different labels that participate in the fusion are referred to as fusing samples.
As shown in Eq.~(\ref{eq:fh2t}), $f_t$ is the fused sample and 
$f_h$ denotes the fusing sample.}

\begin{figure}[t]
    \centering
    \includegraphics[width=1.\linewidth]{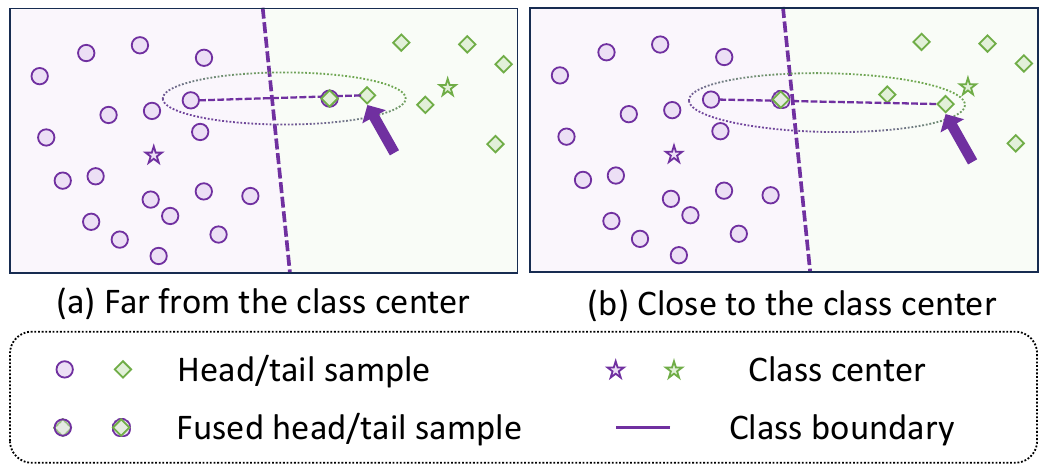}  
    \caption{Illustration of the relationship between the fusion ratio and the distance of a sample from its class center. 
The arrows point to fused samples.
The legend is identical to that in Fig.~\ref{fig:dec_bnd}.
\redtext{H2TF moves H2T-fused samples closer to samples (a) that are originally farther from the class center, while pushing them relatively away from samples (b) that are originally closer to the class center.}}
\label{fig:FusionRatro}
\end{figure}

\vspace{0.5em}
\noindent\textbf{Fusion Ratio in H2TF}. 
\cite{LiMK2024H2T} sets the fusion ratio $r$ as a hyperparameter, which lacks theoretical guidance and relies on empirical experience.
We obtain $r$ based on the distance ($d_x$) of a sample to its class center to determine the fusion ratio.
Since Eq.~(\ref{eq:fh2t}) is a linear combination of two feature vectors, the H2TF fused feature will be located between the two samples, as shown in Fig.~\ref{fig:FusionRatro}.  
A larger $r$ causes the H2TF fused sample to be closer to the sample being fused and farther away from the fusing sample \footnotemark[1], as shown in Fig.~\ref{fig:FusionRatro}~(a).
Samples situated far from the class center are located near the decision boundary. 
Being excessively far from such samples can lead to an overly broadened distribution of the feature distribution, which may adversely impact the model’s performance.
Accordingly, we set $r$ to be proportional to $d_x$:
\begin{equation}\label{eq:r_proto_d}
    r = Norm\left[ d_x \right]_0^1,  
\end{equation}
where $d_x = dist(\mathbf{f}_x, \mathbf{w}_y)$. 
$\mathbf{f}_x$ represents the feature of input $x$, and $\mathbf{w}_y$ denotes the corresponding classifier weights of the ground truth class.
$Norm\left[ \cdot \right]_0^1$ represents a normalization function used to ensure that the fusion ratio $r$ is bounded within $[0,1]$. 
Common normalization techniques include min-max normalization, max normalization, and \texttt{tanh} normalization, all of which introduce no additional parameters\footnote{We adopt \redtext{linear min-max normalization based on theoretical bounds rather than batch statistics} in our experiments.}.
$dist$ is chosen as cosine distance because it has a defined range, which facilitates the normalization of $r$.


\vspace{0.5em}
\noindent\textbf{Fusion Strategy in H2TF}. 
The selection of features to be fused poses a tedious task during training since visual recognition tasks often involve a wide range of classes. 
It is not always practical to guarantee that each mini-batch encompasses all the requisite categories for fusion. 
We thereby devise a simple strategy to facilitate an easy-to-execute fusion process.
This strategy involves sampling two versions of data:
1) class-balanced data $\mathcal{T}^B$ utilized to balance the empirical/structural risk minimization (ERM/SRM) of each class are fed into the fused branch, and
2) instance-wise data $\mathcal{T}^I$ has a high probability of obtaining head class samples, which are fed into the fusing branch. 
The sampling rates $p^B_i$ and $p^I_i$ for class $i$ within the sets $\mathcal{T}^B$ and $\mathcal{T}^I$, respectively, are calculated by:
\begin{equation}\label{eq:sam_rate}
     p^B_i = \frac{1}{C}, 
     p^I_i = \frac{n_i}{N}.
\end{equation}
Balanced sampling data ensures that each class is sampled with equal probability $\frac{1}{C}$. 
Consequently, classes with fewer samples have a higher probability of being resampled multiple times.
In contrast, the instance-wise sampling branch follows the original data distribution, resulting in head-biased sampling where frequent classes are more frequently selected. 
Using the feature maps $\mathcal{F}^B$ and $\mathcal{F}^I$ obtained from $\mathcal{T}^B$ and $\mathcal{T}^I$, respectively, we substitute $\mathcal{F}_t$ and $\mathcal{F}_h$ in Eq.~(\ref{eq:fh2t}). 
By doing so, features from the repeatedly sampled tail classes will be fused with the head class features with a higher probability. 
Unlike the common practice of linearly combining a pair of inputs and their labels for augmentation, we use only the labels of the fused branch, namely the balance-sampled data, as the ground truth.
This strategy leverages the semantic richness of head classes to enhance the representation of tail classes.
H2TF is exclusively employed during the Stage 2 training phase to achieve a balanced feature distribution and more effectively facilitate classifier calibration.
This process does not affect the feature extraction. 
During inference, H2TF is omitted; instead, the feature extractor calibrated in Stage 1 via PIF, together with the subsequently calibrated classifier, is utilized.

\begin{figure}[t]
    \centering
    \includegraphics[width=1.\linewidth]{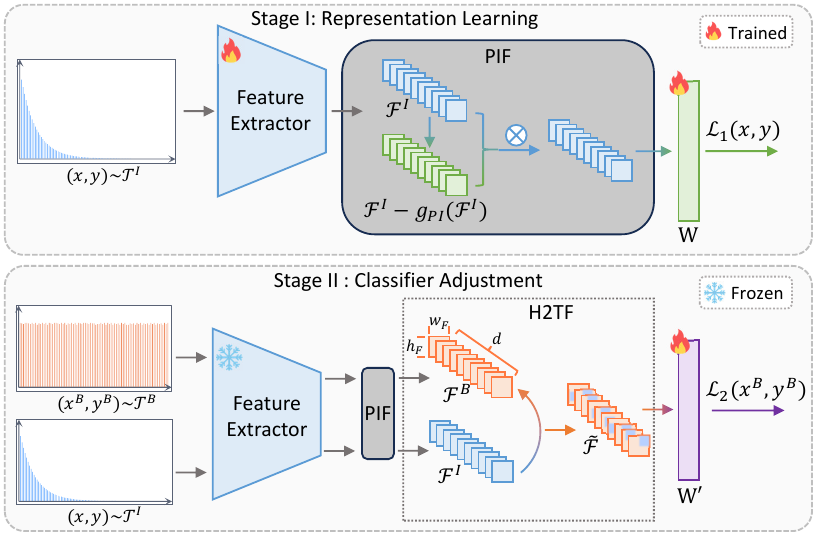}  
    \caption{Overall structure of permutation-invariant and head-to-tail feature fusion (PI-H2T).}
\label{fig:framework}
\vspace{-0.5em}
\end{figure}

The proposed training framework is illustrated in Fig.~\ref{fig:framework}
A summary of the algorithm is provided in Appendix~B . 

\subsection{Rationale Analysis}\label{sec:ana}

\begin{remark}\label{re:PIF}
    PIF is equivalent to automatically assigning classification margins. 
\end{remark}

\begin{proof}
According to Eq.~(\ref{eq:ini_fuse}), the output features following the pooling layer can be represented as:
\begin{equation}\label{eq:ineq1}
    \mathbf{f}_{fuse} = a \cdot \left(  \mathbf{f}-\mathbf{f}_{PI} \right).
\end{equation}
Our optimization objective can ensure that:
\begin{equation}\label{eq:ineq2}
   \mathbf{w}_y^T \mathbf{f}_{fuse} > \mathbf{w}_i^T \mathbf{f}_{fuse}, \text{ for } i\neq y.
\end{equation}
Take Eq.~(\ref{eq:ineq1}) into Eq.~(\ref{eq:ineq2}), we can have:
\begin{gather}
    a\cdot \mathbf{w}_y^T \mathbf{f}- a\cdot \mathbf{w}_y^T \mathbf{f}_{PI}+b \cdot \mathbf{w}_y^T \mathbf{f} > \\ 
    a\cdot \mathbf{w}_i^T \mathbf{f}- a\cdot \mathbf{w}_i^T \mathbf{f}_{PI}+b \cdot \mathbf{w}_i^T \mathbf{f} \Rightarrow \\
    \mathbf{w}_y^T \mathbf{f} > \mathbf{w}_i^T \mathbf{f} + \frac{a}{a+b} \cdot \left( \mathbf{w}_y^T \mathbf{f}_{PI} - \mathbf{w}_i^T \mathbf{f}_{PI}\right).\label{eq:PIF_margin}
\end{gather}
\redtext{
We define the class margin as
$\dfrac{a}{a+b} \cdot \left( \mathbf{w}_y^T \mathbf{f}_{PI} - \mathbf{w}_i^T \mathbf{f}_{PI}\right)$.
where $\mathbf{w}_y$ is the center of the $y$-th class, learned from the data distribution of class $y$.
$\mathbf{f}_{PI}$ is the PI feature corresponding to the target class.
The feature $\mathbf{f}_{PI}$, obtained from the backbone with learnable parameters, should be aligned more closely with its true class center $\mathbf{w}_y$ than with any other class center $\mathbf{w}_i$.
This implies that $\left( \mathbf{w}_y^T \mathbf{f}_{PI} - \mathbf{w}_i^T \mathbf{f}_{PI}\right) > 0$.
Accordingly, the margin is guaranteed to be positive and can be adaptively estimated for each class, depending on both its PI feature and the overall dataset distribution.
}
\end{proof}
\redtext{
Further empirical validation is presented in Section~\ref{sec:vis_feat}.
}

\begin{remark}\label{re:H2TF}
    H2TF generates two mutually restraining forces. 
\end{remark}
To simplify the analysis, we select two classes, namely the head and tail classes, without loss of generality.

\begin{proof}
We first analyze the inequality relationships for correctly classified tail samples and H2TF head samples.
For correctly classified tail samples, the inequality $z_t > z_h$ holds. 
For H2TF head samples, we have $\tilde{z}_h > \tilde{z}_t$, where $\tilde{z}_i = w_i\tilde{f}_i$ represents the H2TF logit of class $i$.
Consequently, the following inequality can be established:
\begin{equation}\label{eq:correct_tail}
\begin{split}
    r \cdot \mathbf{w}_{t}^T \mathbf{f}_{t}+ \uwave{(1-r) \cdot \mathbf{w}_{t}^T \mathbf{f}_{t}}  > r \cdot \mathbf{w}_{h}^T \mathbf{f}_{t}+\dotuline{(1-r) \cdot \mathbf{w}_{h}^T \mathbf{f}_{t}} & \\ 
     \triangleright \text{ for correct tail}, &\\ 
\end{split}  
\end{equation}
\begin{equation}\label{eq:H2TF_head}
\begin{split}
    r \cdot \mathbf{w}_{h}^T\mathbf{f}_{h}+\dotuline{(1-r) \cdot \mathbf{w}_{h}^T \mathbf{f}_{t}}  > r \cdot  \mathbf{w}_{t}^T \mathbf{f}_{h}+ \uwave{(1-r) \cdot \mathbf{w}_{t}^T \mathbf{f}_{t}}& \\ 
     \triangleright \text{ for H2TF head}.& 
\end{split}  
\end{equation}
By adding Eq.~(\ref{eq:correct_tail}) to Eq.~(\ref{eq:H2TF_head}), we derive
\begin{equation}\label{eq:fuse1}
    \mathbf{w}_t^T (\mathbf{f}_t-\mathbf{f}_h) > \mathbf{w}_h^T (\mathbf{f}_t-\mathbf{f}_h).
\end{equation}
Let $\theta_*$ (where $*$ is $h$ or $t$) denote the angle between $\mathbf{w}_*$ and $\mathbf{f}_t-\mathbf{f}_h$. 
Eq.~(\ref{eq:fuse1}) can then be simplified as
\begin{equation}\label{eq:theta1}
    |\mathbf{w}_t| \cos \theta_t > |\mathbf{w}_h| \cos \theta_h.
\end{equation}

In H2TF, $\mathbf{f}_t$ and $\mathbf{f}_h$ are fixed and only $\mathbf{w}_t$ and $\mathbf{w}_h$ are adjustable because we only finetune the classifier by freezing the feature extractor. 
Therefore, Eq.~(\ref{eq:theta1}) encourages $\mathbf{w}_t$ to be in closer proximity to tail samples while pushing $\mathbf{w}_t$ further away from them, thus increasing the distribution span of head classes, as illustrated by force {\footnotesize\textcircled{\scriptsize{1}}} in Fig.~\ref{fig:rationale}.

\redtext{Next, we consider wrongly classified tail samples and H2TF tail samples.}
For wrongly classified tail samples, the inequality $z_h > z_t$ holds for a given tail sample.
For H2TF tail samples, $\tilde{z}_t > \tilde{z}_h$ is satisfied.
Therefore, the following inequalities are established:
\begin{equation}\label{eq:wrong_tail}
\begin{split}
    \uwave{r \cdot \mathbf{w}_{h}^T \mathbf{f}_{t} } + (1-r) \cdot \mathbf{w}_{h}^T \mathbf{f}_{t} > 
    \dotuline{r \cdot \mathbf{w}_{t}^T \mathbf{f}_{t} } + (1-r) \cdot \mathbf{w}_{t}^T \mathbf{f}_{t}& \\ 
     \triangleright \text{ for wrong tail}. &\\ 
\end{split}  
\end{equation}
\begin{equation}\label{eq:H2TF_tail}
\begin{split}
    \dotuline{r \cdot \mathbf{w}_{t}^T \mathbf{f}_{t}} + (1-r) \cdot \mathbf{w}_{t}^T \mathbf{f}_{h}  > \uwave{r \cdot  \mathbf{w}_{h}^T \mathbf{f}_{t}} + (1-r) \cdot \mathbf{w}_{h}^T \mathbf{f}_{h} & \\ 
     \triangleright \text{ for H2TF tail}.& 
\end{split}  
\end{equation}
Analogous to the derivation from Eq.~(\ref{eq:theta1}), we can have
\begin{equation}\label{eq:theta2}
     |\mathbf{w}_h| \cos \theta_h > |\mathbf{w}_t| \cos \theta_t. 
\end{equation}
\redtext{Eq.~(\ref{eq:theta2}) applies a force opposing the effect of Eq.~(\ref{eq:theta1}), as depicted by force {\footnotesize\textcircled{\scriptsize{2}}} in Fig.~\ref{fig:rationale}}.
\redtext{However, because the number of correctly classified samples typically exceeds that of misclassified samples in the training set, the effect of Eq.~(\ref{eq:theta1}) dominates, leading to a well-calibrated classifier.}
\end{proof}

\begin{figure}[tb]
    \centering
    \includegraphics[width=\linewidth]{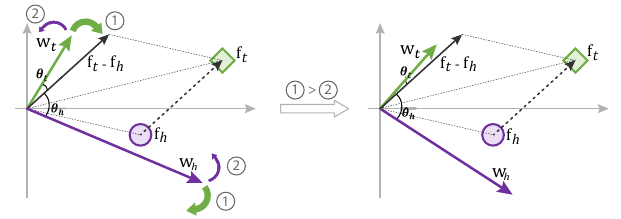}
    \caption{\redtext{Rationale analysis of H2TF}. Forces {\footnotesize\textcircled{\scriptsize{1}}} and {\footnotesize\textcircled{\scriptsize{2}}} are generated by Eqs.~(\ref{eq:theta1}) and (\ref{eq:theta2}), respectively. {\footnotesize\textcircled{\scriptsize{1}}} $>$ {\footnotesize\textcircled{\scriptsize{2}}} makes the tail sample to ``pull" closer to $w_t$ and ``push" further away from $w_h$, leading to the adjustment of decision boundary and enlargement of the tail class space.}
\label{fig:rationale}
\end{figure}

Fig.~\ref{fig:rationale} geometrically interprets the rationale of H2TF.

As highlighted in Remark~\ref{re:PIF}, PIF provides flexibility by automatically assigning classification margins, which allows the subsequent H2TF to more effectively adjust the decision boundary.
\redtext{Remark~\ref{re:H2TF} further reveals that H2TF generates two mutually restraining forces. 
While the dominant force shifts the boundary toward better tail recognition, the opposing force prevents excessive adjustment, maintaining overall balance.
Section~\ref{sec:ablation} provides experimental visualizations that support these theoretical insights.
Consequently, PI-H2T enhances the performance of existing methods by jointly leveraging the adaptive margins from PIF and the reciprocal regularization of H2TF.}

%% file: tab/sym_tab.tex
\begin{table}[!t]
\centering
\caption{Summary of Notation and Symbols}
\label{tab:symbols}
 \resizebox{1.\linewidth}{!}{
\renewcommand{\arraystretch}{1.3}
\begin{tabular}{c | l  }
\hline
\textbf{Symbol} & \textbf{Description} \\
\hline
\redtext{$x, y$} & \redtext{Input image and its ground-truth label, $y \in \{0, \dots, C-1\}$.} \\
\redtext{$C, N, n_i$} & \redtext{Total number of classes, samples, and samples per class $i$.} \\
$\mathcal{F}$ & Feature maps before final pooling layer, $\mathcal{F} \in \mathbb{R}^{w_F \times h_F \times d}$. \\
$\mathbf{f}$ & Pooled feature representation, $\mathbf{f} \in \mathbb{R}^{d}$. \\
$d$ & Feature dimension (embedding size). \\
$\mathbf{W}$ & Classifier weights matrix, $\mathbf{W} \in \mathbb{R}^{d \times C}$. \\
$\mathbf{w}_i$ & Weight vector of the $i$-th class, $\mathbf{w}_i \in \mathbb{R}^{d}$. \\
$z_i$ & Predicted logit of class $i$, $z_i = \mathbf{w}_i^\top \mathbf{f}$. \\
$h, m, t$ & Indices for head/medium/tail classes. \\
\hline
\end{tabular}
}
\end{table}

%% file: Sec/4_Experiment.tex
\section{Experiments}

\subsection{Datasets and Metrics}
Our proposed PI-H2T is evaluated on four widely-used benchmarks: CIFAR100-LT~\cite{Kaidi2019}, ImageNet-LT~\cite{LiuZW19LTOW}, iNaturalist 2018~\cite{Horn_2018_CVPR}, and Places-LT~\cite{LiuZW19LTOW, Liu2022OLTR2}. 

\noindent\textit{CIFAR100-LT.} It is a subsampled version of the balanced CIFAR100 dataset~\cite{krizhevsky2009learning}, containing RGB images from 100 natural image classes with a feature dimension of $32 \times 32$. 
Following~\cite{LiuZW19LTOW}, the sampling rate is modeled by an exponential function, $n_i = N \times \lambda^i$, where $\lambda \in \left(0,1 \right)$, $N$ represents the class size of the original balanced training set, and $n_i$ denotes the sampling quantity for the $i$-th class in the long-tailed version.
We use widely adopted imbalance factor, calculated as $\rho=\dfrac{n_{max}}{n_{min}}$, with values of 200, 100, and 50~\cite{Peng2020Feature}. 

\vspace{0.5em}
\noindent\textit{ImageNet-LT and Places-LT.}
Long-tailed versions of the original balanced ImageNet-2012~\cite{ILSVRC15} and Places 365~\cite{zhou2017places} datasets are utilized.
For imageNet-LT, we adopt the same settings as Liu et al.~\cite{LiuZW19LTOW, Liu2022OLTR2}, where a subset is sampled according to a Pareto distribution with a power value of $\alpha = 6$.
Overall, this long-tailed version comprises 115.8K images spanning 1,000 categories, with each class containing a maximum of 1,280 images and a minimum of 5 images.
Places-LT consists of 184.5K images from 365 categories, with class sizes ranging from a maximum of 4,980 images to a minimum of 5 images.

\vspace{0.5em}
\noindent\textit{iNaturalist 2018.} The iNaturalist datasets are inherently imbalanced due to their global and nature-based collection.
We employ the 2018 version~\cite{Horn_2018_CVPR} for our experimental analysis.
The dataset comprises 8,142 classes with a total of 437.5K images, and class cardinality ranges from 2 to 1,000.

\vspace{0.5em}
\noindent\textit{Evaluation metrics.} In addition to top-1 classification accuracy, following Liu~et~al.~\cite{LiuZW19LTOW}, the accuracy on three partitions: head ($n_i > 100$), medium ($20<n_i\leq 100$) and tail ($n_i\leq 20$) for large-scale datasets are also compared.  

\subsection{Basic Settings}
We validate the effectiveness of the proposed Pi-H2T on two primary backbone architectures.

\vspace{0.5em}
\noindent\textit{CNN-based methods.}
SGD with a momentum of 0.9 is adopted for all datasets. 
Stage II is trained with 10 epochs.
For a fair comparison, we employ MixUp~\cite{Hongyi2018} as the sole augmentation strategy for all methods.
For CIFAR100-LT, we refer to the settings in Cao~et~al.~\cite{Kaidi2019} and Zhong~et~al.~\cite{mislas21}. 
The backbone network is ResNet-32~\cite{he2016deep}. 
The maximum training epoch is 200 in Stage I.
The initial learning rate is 0.1 and is decayed in the $160^{th}$ and $180^{th}$ epochs by 0.1.  
The batch size is 128. 
For imageNet-LT and iNaturalist 2018, we use the commonly used ResNet-50. 
For Places-LT, we utilize the ResNet-152 pre-trained on imageNet.

\vspace{0.5em}
\noindent\textit{MHSA-based method.} 
We mainly follow the settings in Dong et al.~\cite{Dong23LPT}. 
The ViT-B/16 model, pre-trained on ImageNet-21K, serves as the backbone architecture.
For the augmentation strategy, consistent with widely adopted practices among mainstream methods, we adopt the same data augmentation strategies in \cite{RenJW2020Balanced, liJ2022nested, Jin2023shike, Dong23LPT}.  
Specifically, for CIFAR100-LT, AutoAugment and Cutout are adopted. 
RandomAugment is adopted for Places-LT and iNaturalist 2018. 
We employ SGD as an optimizer with a batch size of 128 and an initial learning rate of 0.01, following a cosine annealing schedule.

For comparison methods, we reproduce prior approaches using the hyper-parameters provided to ensure a fair comparison. 
For methods where hyper-parameters or official codes are not available, we report the results as presented in the original papers.

\subsection{Comparisons to Existing Methods} 
\label{sec:com}

\subsubsection{Compared Methods.}\label{sec:com_method}
For single models with CNN as the backbone, we compare the proposed PI-H2T with the following three kinds of methods:
\textit{two-stage methods}, i.e., LDAM-DRW~\cite{Kaidi2019}, decoupling representation (DR)~\cite{decouple20}, MisLAS~\cite{zhang2021bag}, and GCL~\cite{LiMK2022GCL};
\textit{decision boundary adjustment method}, i.e., Adaptive Bias Loss (ABL)~\cite{jin2023optimal};
\textit{data augmentation methods}, i.e., FSA~\cite{chu2020feature}, MBJ~\cite{Liu2022Memory}, CMO~\cite{ParkS2022Majority} and LCReg~\cite{liu2024lcreg}. 
MixUp~\cite{Hongyi2018} can enhance representation learning et al.~\cite{zhang2021bag,mislas21}, therefore, we reproduce DR cooperated with MixUp for comparison. 
LCReg

We report the results of CE loss and balanced softmax cross-entropy loss (BSCE)~\cite{RenJW2020Balanced} with CMO (abbreviated as CE+CMO and BSCE+CMO, respectively). 
The results of H2T accompanied by DR (utilizing CE loss), MisLAS, and GCL are reported.
In the evaluation of multi-expert models, we compare BBN~\cite{bbn20}, RIDE~\cite{WangXD21RIDE}, and ACE~\cite{Cai2021ACE}.
In particular, both RIDE and ACE employ 3-expert architectures. 
For Places-LT, we include the results of H2T integrated with RIDE for comparative analysis. 
Regarding MHSA-based models, our comparison focuses on methods that involve fine-tuning foundational models. 
We compare with \redtext{LIFT+DBM~\cite{son2025difficulty},} LIFT~\cite{shi2024LIFT}, VPT~\cite{jia2022visual}, LPT~\cite{Dong23LPT}, and Decoder~\cite{wang2024exploring}.

\input{tab/cifar}
\input{tab/img}

\input{tab/inat}
\input{tab/pla}

\subsubsection{Comparison Results.}
\label{sec:resluts}

\vspace{0.5em}
\noindent\textbf{CIFAR100-LT.} The results are presented in Table~\ref{tab:cifar_results}. 
PI-H2T is capable of further enhancing commonly used backbones, including single-expert, multi-expert, and MHSA-based models, demonstrating superior performance over other methods in handling varying degrees of class imbalance. 
Applying PI-H2T to decoupled representation with CE loss (DR+PI-H2T) significantly enhances performance, surpassing the latest methods as well as H2T. 
The greatest improvement exceeds 3\% over DR+MU and 2\% over H2T.
DR+PI-H2T can even surpass H2T with GCL loss (e.g., 45.58\% vs. 45.24\% with $\rho = 100$), which is one of the state-of-the-art two-stage methods.
GCL with PI-H2T (GCL+PI-H2T) further amplifies this advantage and achieves performance comparable to multi-expert models within a single-model framework.
For example, GCL+PI-H2T attains 46.34\% top-1 accuracy with $\rho = 200$, surpassing RIDE, which achieves 45.84\%.
For multi-expert models, PI-H2T remains highly effective in multi-expert and MHSA-based models.
For example, PI-H2T increases RIDE performance by 1.97\% and LPT performance by 1.11\% with $\rho = 200$, respectively.

\vspace{0.5em}
\noindent\textbf{ImageNet-LT\protect\footnotemark[3] and iNaturalist 2018.}
\footnotetext[3]{Given that the model pre-trained on ImageNet-21K inherently contains information from ImageNet-1K, which is the balanced version of ImageNet-LT, we opt not to implement LPT (or with PI-H2T) on ImageNet-LT.} 
Tables~\ref{tab:img} and \ref{tab:iNat} compare PI-H2T on large-scale datasets with existing methods.
It can be observed that, compared to H2T, PI-H2T achieves more substantial performance improvements in existing methods, particularly when adapted with GCL, demonstrating its effectiveness on large-scale datasets.
The baseline model using CE loss achieves high accuracy on frequent classes, but its performance on rare classes remains unsatisfactory. 
CMO leverages the background of head classes to augment tail classes without reducing the number of training samples for head classes, thereby enhancing model performance across both head and tail classes.
However, CMO combined with CE (CE+CMO) shows less competitive results in improving tail-class performance.
ABL also adjusts the decision boundary, but its effectiveness is not as pronounced as H2T and PI-H2T.
For example, on imageNet-LT, DR+PI-H2T achieves 53.37\%, already surpassing the comparison methods.
GCL+PI-H2T reaches 55.08\%, the highest among all single-model comparison methods, and even exceeds multi-expert models such as BBN, RIDE, and ResLT. 
Additionally, RIDE+PI-H2T significantly outperforms the original RIDE model, achieving 57.45\% compared to 55.72\%.
On iNaturalist 2018, GCL+PI-H2T achieved a top-1 classification accuracy of 75.02\%, significantly improving upon single-model approaches. 
In MHSA-based methods, PI-H2T also demonstrated a notable enhancement, with an increase of nearly 1\% (77.07\% vs. 76.10\%).

\vspace{0.5em}
\noindent\textbf{Places-LT.}
The results for Places-LT are presented in Table~\ref{tab:pla}.
PI-H2T can still enhance the performance of existing methods, although it shows less improvement compared with H2T.
Notably, H2T requires a hyperparameter for the fusion ratio. 
In contrast, PI-H2T not only further improves existing methods over H2T but also eliminates the need for manual hyperparameter selection.
One possible reason is that the backbone of the CNN-based method for Places-LT is pre-trained on ImageNet, with only a subset of parameters fine-tuned. 
This results in a relatively fixed feature space that lacks the flexibility of a model trained from scratch. 
In addition, Eq.~(\ref{eq:PIF_margin}) indicates that for the margin to be effective as a positive value, the model should correctly classify the majority of target classes. 
For example, in the single model case, the training accuracy of Places-LT is 63.16\%, whereas for CIFAR100-LT with an imbalance factor of 100, the training accuracy reaches 84.70\%. 
The improvement on CIFAR100-LT is more noticeable in comparison.
On Places-LT, the automatic margin effect introduced by PIF is minimal, and the class boundary adjusted by H2TF is less pronounced compared to those observed on other datasets.
The enhancement effect of both PI-H2T and H2T is less obvious.
\redtext{
To further understand this behavior, we recorded the margin values after Stage-1 training with PIF, as shown in Fig.~\ref{fig:margin_comp}. 
On Places-LT, the margins are smaller and more concentrated, reflecting limited flexibility to adjust class boundaries. 
In contrast, CIFAR100-LT shows wider margins with higher averages, allowing more effective class separation. 
}
Notably, PI-H2T remains effective for MHSA-based models.
The results in Table~\ref{tab:pla} for LPT+PI-H2T indicate that PI-H2T can effectively improve MHSA-based models, enhancing LPT performance from 49.70\% to 50.12\%.

\begin{figure}[t]
    \centering
    \includegraphics[width=0.95\linewidth]{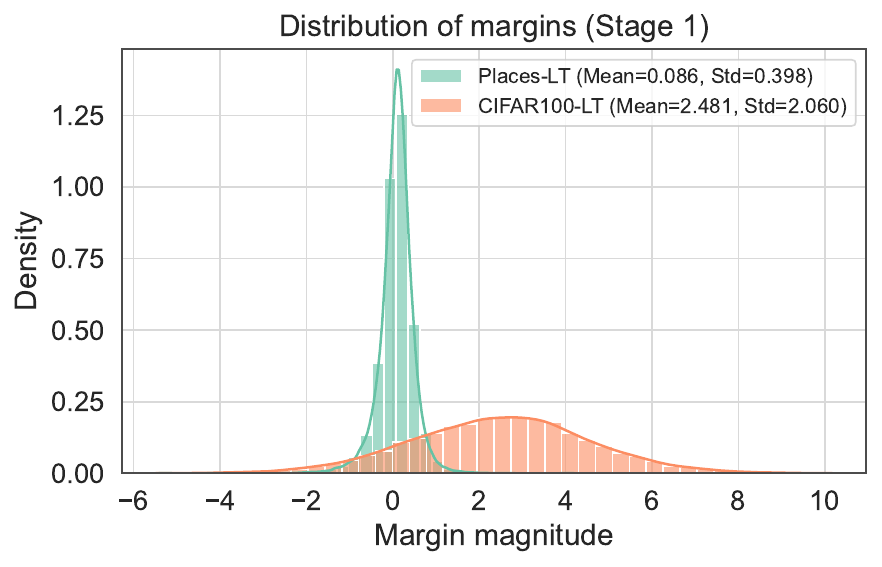}
    \vspace{-1em}
    \caption{\redtext{Distribution comparison of class margins on CIFAR100-LT ($IF$=100) and Places-LT after Stage 1 training of PI-H2T.}}
    \label{fig:margin_comp}
\end{figure}

\input{tab/comp_w_h2t}

\begin{figure}[t]
    \centering 
    \subfigure[CE]{\label{fig:abl_CE}
    \includegraphics[width=1.\linewidth]{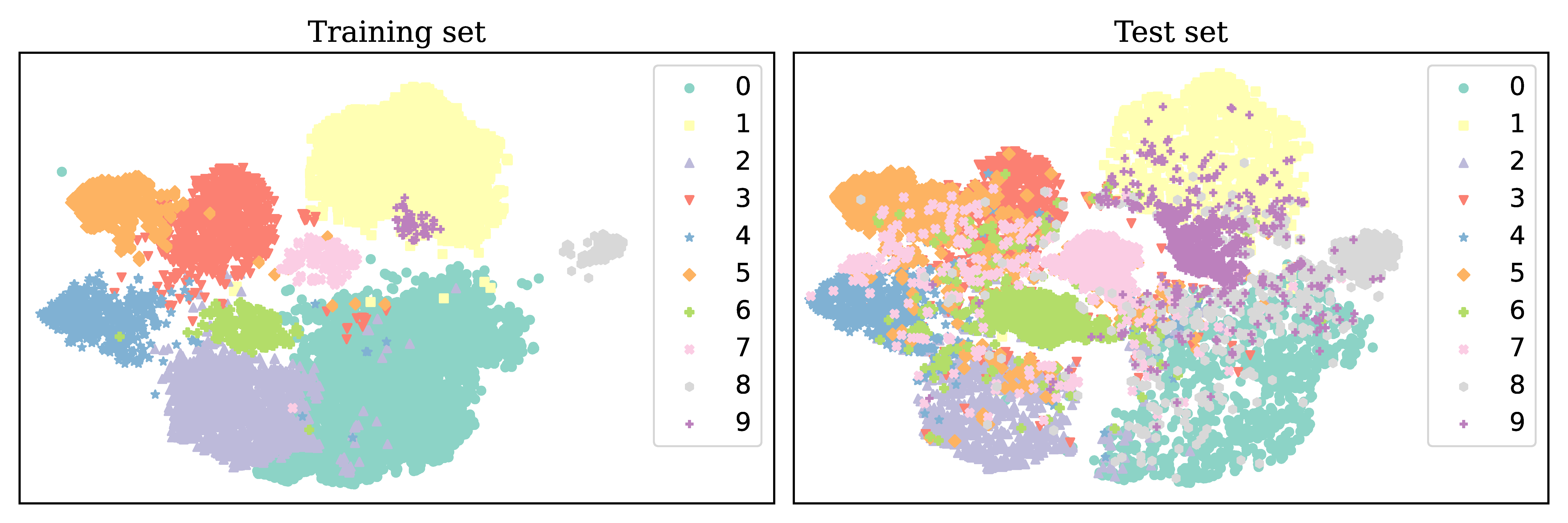} 
    }
    \subfigure[CE+PIF]{\label{fig:abl_H2TF}
    \includegraphics[width=1.\linewidth]{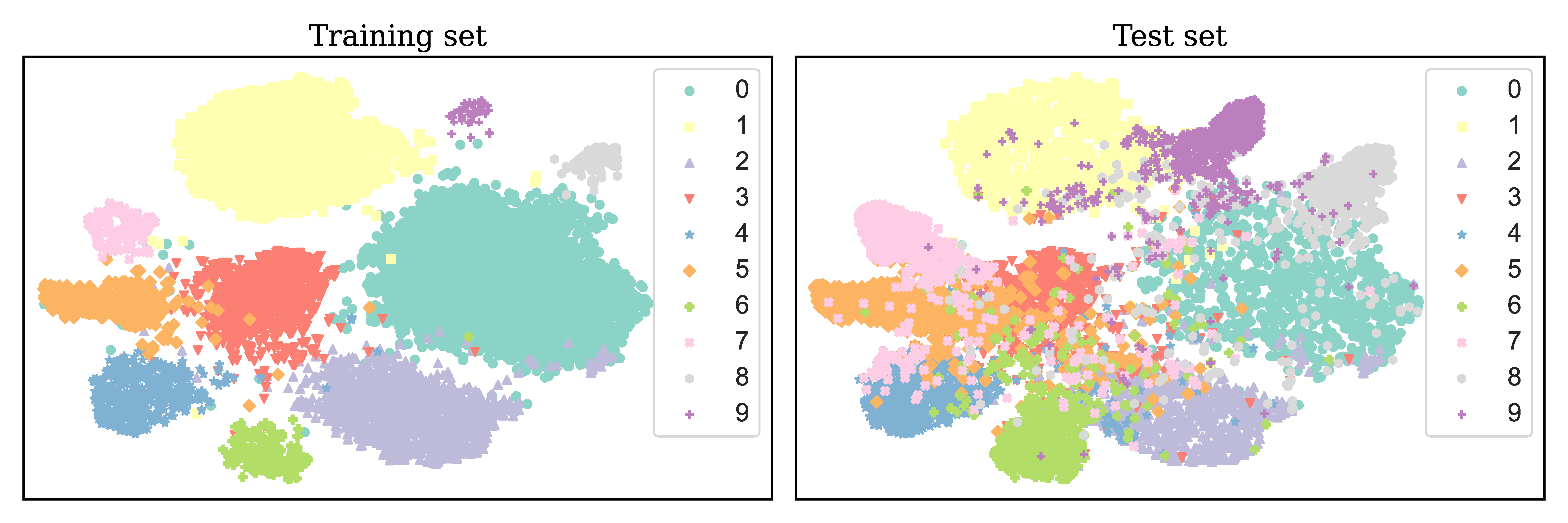} 
    }
\caption{Comparison of feature distributions between CE and CE+PIF via t-SNE visualization. The dataset is CIFAR10-LT. CE loss is adopted. }
\label{fig:abl_tsne}
\vspace{-0.5em}
\end{figure}

\vspace{0.5em}
\noindent\textbf{\redtext{Quantitative Comparison with H2T~\cite{LiMK2024H2T} (the Conference Version).}}
\redtext{
Tables~\ref{tab:comp_h2T-ce} and ~\ref{tab:comp_h2T-gcl} present the quantitative comparison between the proposed PI-H2T and the conference version H2T under two loss settings, i.e., CE loss and GCL loss~\cite{LiMK2022GCL}.
Overall, PI-H2T achieves competitive or superior performance across most datasets and training stages, verifying the effectiveness of the permutation-invariant feature fusion mechanism.
Under the CE loss, PI-H2T yields notable improvements on CIFAR100-LT and ImageNet-LT, while maintaining comparable results on Places-LT and iNaturalist 2018.
When combined with the GCL loss, PI-H2T further enhances performance on large-scale datasets, achieving over 2\% accuracy gain on iNaturalist 2018.
However, slight fluctuations can be observed in a few cases (e.g., Places-LT under CE loss), which stem from the relatively small inter-class variation and limited benefit from feature permutation in scene-centric datasets.
Despite these minor variations, the overall trend consistently supports that PI-H2T inherits the structural merits of H2T while providing improved adaptability and robustness across diverse data distributions and optimization objectives.
}

\input{fig/boundary/class09-2}
\input{fig/ablation/PIF-tsne-mini}

\subsection{Further Analysis}\label{sec:ablation}
This section visualizes the feature distributions produced by PIF, the decision boundaries learned by H2TF across all classes, and the feature distributions after applying H2TF. 
Unless otherwise specified, all experiments are conducted using the cross-entropy loss on the CIFAR10-LT dataset with an imbalance factor of 100 to facilitate clear visualization. 
The backbone network employed is a single ResNet-32. 
Additionally, we analyze the impact of different permutation-invariant operations and balanced datasets to further validate the effectiveness of PIF.

\vspace{0.5em}
\noindent\textbf{Analysis of PIF Effectiveness via Feature Distribution Visualization.}\label{sec:vis_feat}
Fig.~\ref{fig:abl_tsne} visualizes the embedding space of a single ResNet-32 model trained with both vanilla CE and CE+PIF.
From the results of the training set, compared to CE loss, the feature distribution of each class obtained by CE+PIF exhibits clearer class margins. 
For example, the margins between Classes 8 and 0, as well as between Classes 9 and 0, are more pronounced with CE+PIF.
\redtext{Quantitatively, the intra-class distance of tail classes decreases substantially from 7.13 to 4.50 after applying PIF, indicating reduced feature dispersion for under-represented categories. 
The head and medium classes remain relatively stable (0.85 → 1.11 and 1.54 → 1.72, respectively), showing that PIF improves feature quality without negatively affecting well-represented classes. }
These observations are consistent with the t-SNE results shown in Fig.~\ref{fig:abl_H2TF} and further support the theoretical motivation of ``automatic margins'' discussed in Remark~\ref{re:PIF} of Section~\ref{sec:ana}.

\vspace{0.5em}
\noindent\textbf{Impact of Different Choices in Permutation-Invariant Operations for PIF.}
To investigate how different choices in permutation-invariant (PI) operations affect the performance of the proposed PIF module, we conduct an ablation study on CIFAR100-LT with an imbalance factor of 100.
The results are presented in Table~\ref{tab:PI_abl}.
We compare the top-1 classification accuracy achieved using different PI operations, including mean, maximum (Max), and minimum (Min) operations, as well as a non-PI baseline (Base).
For all variants, H2TF is applied during stage 2.
All three PI operations yield consistent improvements over the base model, validating the effectiveness of permutation-invariant representations in enhancing feature representation.
Notably, the Max operation achieves the highest accuracy in Stage 2 (49.53\%), while Min performs slightly better in Stage 1 (42.37\%).
Overall, these results demonstrate the robustness of PIF to the choice of aggregation function and demonstrate its general efficacy.

\input{tab/PI-abl}
\input{tab/cifar10-LT}

\vspace{0.5em}
\noindent\textbf{Visualization of Decision Boundary.}\label{sec:vis_boundary}
Fig.~\ref{fig:boundary09} presents the t-SNE visualizations of the distribution in embedding space and the decision boundary between the head and tail classes.
The results validate our primary motivation: PIF enhances representations and automatically provides sufficient and appropriate margins. 
Meanwhile, H2TF further calibrates the decision boundary.
For clarity and ease of presentation, the experiment is conducted on the CIFAR10-LT dataset with an imbalance factor of 100, focusing on the most frequently misclassified classes, specifically Class 0 versus Class 9.

It can be observed that the feature distribution of the tail classes is sparser compared to that of the head class.
The limitation of H2T is that the improvement potential is constrained by the unchanged feature distribution.
In contrast, PI-H2T requires PIF to train the model to achieve a better feature distribution. 
As shown in Fig.~\ref{fig:PI-H2T09}, the intra-class features of PI-H2T are more tightly clustered. 
Notably, even without explicitly adding a class margin to the logit, i.e., by using only the basic CE loss, the feature distribution with class margin is automatically obtained. 
On this improved feature space, H2TF can more effectively adjust the class boundaries. 
The accompanying drawback, however, is that the classifier adjustment must be trained after the representation learning stage.
In addition, some failure cases still remain. 
\redtext{A small portion of head-class features can be misassigned to tail-class areas, because H2TF slightly sacrifices head-class feature space during training to enrich tail classes. 
Additionally, some tail-class samples in the test set are inherently difficult to distinguish due to feature similarity or minor variations, leading to occasional overlap with head-class regions.
}
\redtext{The decision boundary for the baselines (DR and DR+H2T) are present in Appendix~C.} 

Moreover, we present the class boundary achieved without using H2TF, relying solely on PIF (PIF st.1), as illustrated in Fig~\ref{fig:PI-st109}.
It can be observed that the uncalibrated classification boundaries fit the training set well but lack generalization to the test set.
Using PIF to enhance representation while employing the calibration strategy in DR (DR+PIF) can improve overall performance; however, it often sacrifices the performance of the head class, as shown in Fig~\ref{fig:PI-DR09}.
In contrast, H2TF effectively balances the performance of both the head and tail classes, as illustrated in Fig~\ref{fig:PI-H2T09}.
Table~\ref{tab:cifar10} provides the average accuracy across all classes for CIFAR10-LT for reference.
PIF achieves a superior representation space over DR (73.65\% vs. 72.96\% in stage 1), enabling more effective classifier correction and greater performance gains in stage 2.

\input{fig/ablation/fusion_ratio_his}

\begin{figure}[t]
    \centering
    \subfigure[\redtext{Average $r$ on CIFAR10-LT (IF=100)}]{
        \includegraphics[width=0.95\linewidth, height=0.38\linewidth]{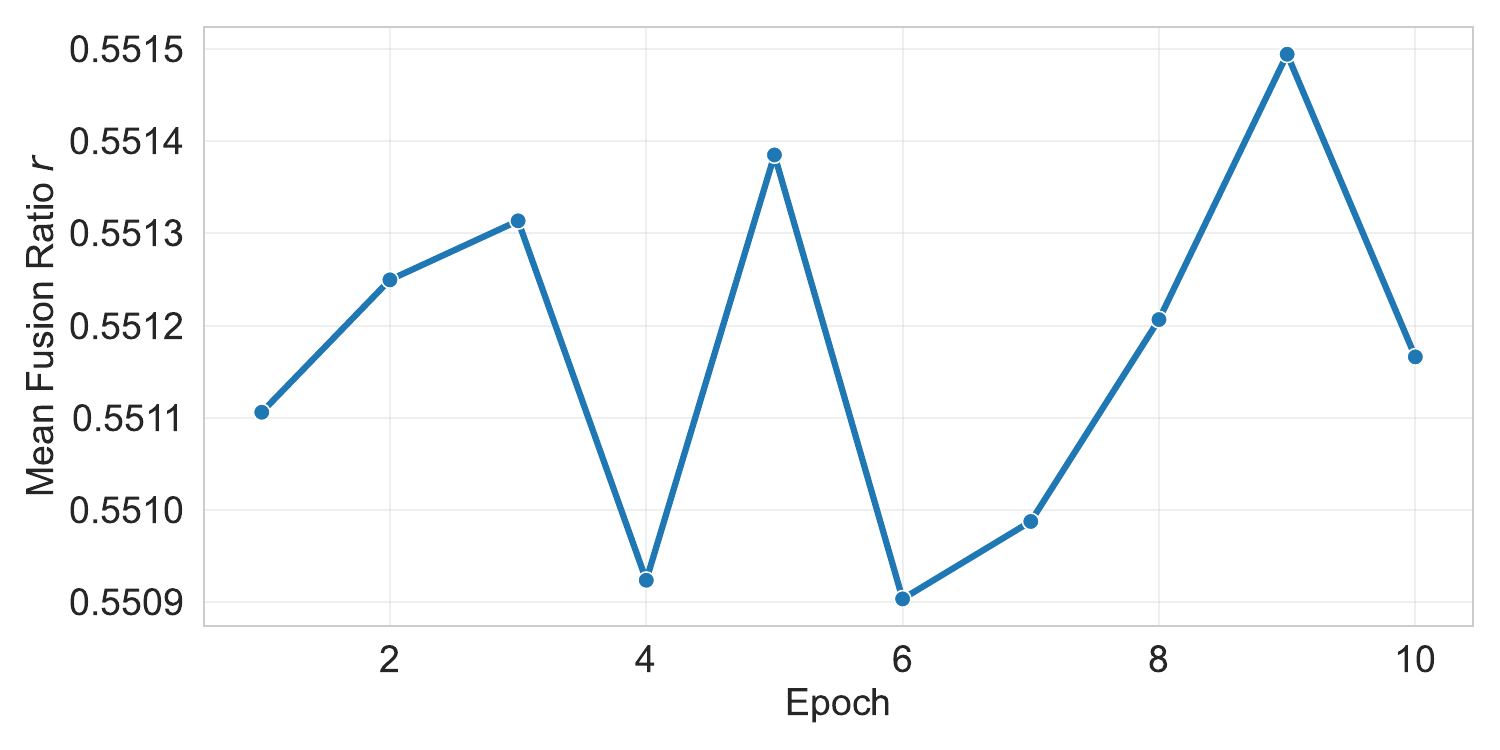}
        \label{fig:fusion_line-cifar10}
    }
    \subfigure[\redtext{Average $r$ on CIFAR100-LT (IF=100)}]{
        \includegraphics[width=0.95\linewidth, height=0.38\linewidth]{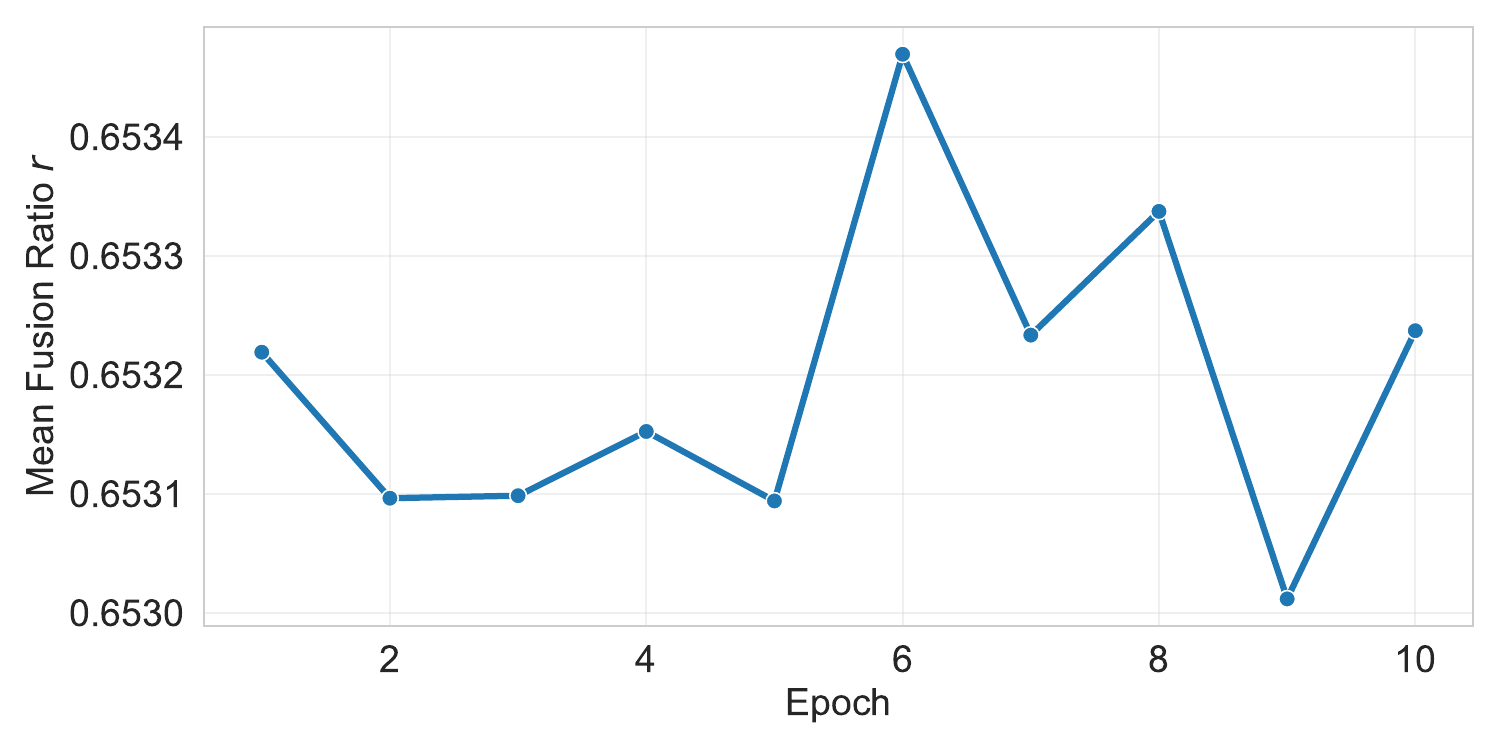}
        \label{fig:fusion_line-cifar100}
    }
    \caption{\redtext{Evolution of the average fusion ratio ($r$) across training epochs.}}
    \vspace{-0.5em}
    \label{fig:fusion_line}
\end{figure}

\vspace{0.5em}
\noindent\textbf{Feature Distribution for H2TF.}\label{sec:vis_H2TF}
To further examine the effect of fusion ratio in the H2TF strategy, we visualize the feature distributions under varying degrees of sample integration.
We compare different fusion strategies, including fixed fusion weights and the adaptive fusion weight proportional to the similarity-based distance, as formulated in Eq.~\ref{eq:r_proto_d}.
The results are shown in Fig.~\ref{fig:tsne_PIH2T}.
In Fig.~\ref{fig:sub2}, we apply balanced sampling without incorporating feature fusion from other samples, namely, H2TF with $r=1$. 
Compared to the original instance-wise sampling in long-tailed data (Fig.~\ref{fig:sub1}), this strategy partially alleviates class imbalance. 
However, the feature distribution of tail classes remains sparse and fragmented.
As $r$ decreases (Figs.~\ref{fig:sub4} and \ref{fig:sub6}), the samples progressively fill the representation space, effectively mitigating the dominance of head class features and addressing the motivation described in Fig.~\ref{fig:dec_bnd} of Section~\ref{sec:intro}.
In particular, Fig.~\ref{fig:sub8} demonstrates that an appropriately chosen fusion ratio leads to a more uniform distribution, with features being better aligned within their corresponding clusters.
Conversely, excessively reducing the self-feature component and overly relying on fused features introduces semantic noise, leading to confusion in the feature space, as shown in Fig.~\ref{fig:sub10}.
It is important to note that H2TF is applied exclusively during the classifier rectification phase. 
Specifically, the decision boundary is adjusted based on features perturbed by H2TF, while the original backbone features remain unchanged. 
During inference, H2TF is not used.
Therefore, the feature representations are unaffected.
The full visualization results can be found in Appendix~D. 

\redtext{Fig.~\ref{fig:fusion_hist} and Fig.~\ref{fig:fusion_line} illustrate the distribution of the automatically determined fusion ratio $r$ and the evolution of its average value during training.
Both CIFAR10-LT and CIFAR100-LT reach their best performance at epoch 6, while the results at the final epoch remain comparable, suggesting that the proposed adaptive fusion mechanism is robust and converges reliably.
For CIFAR10-LT, $r$ is primarily concentrated around 0.55, whereas for CIFAR100-LT, it centers around 0.65.
The mean value of $r$ remains relatively stable throughout training, indicating stability in the adaptive fusion process.
The difference in the optimal fusion ratios across datasets further demonstrates that distinct datasets require different trade-offs between balanced sampling and instance-wise sampling, underscoring the necessity of automatically determining $r$.
}

\input{tab/PI-baselines}

\vspace{0.5em}
\redtext{\noindent\textbf{Ablation Study of PIF and Related Baselines.}  
To clarify that the observed performance gain is primarily driven by the PIF module rather than additional mixing capacity, we conducted controlled ablations, summarized in Table~\ref{tab:pif_ablation}.  
Adding a standard Conv1D without the PI term (+Conv1D) only slightly improves accuracy (39.78\% vs. 39.55\% baseline), despite having the same minimal parameter increase as PIF. 
Standard attention/pooling variants (SE, LP) achieve moderate gains (40.55\% and 40.39\%, respectively), but require substantially more parameters and FLOPs. 
When combining these blocks with PIF, the gains are additive, further confirming that PIF provides complementary improvement.  
These results demonstrate that PIF enhances performance with minimal computational and parameter overhead (+2 parameters and negligible floating-point operations, FLOPs).
Appendix~E further illustrates the effectiveness of PIF on balanced datasets.
} 

%% file: tab/cifar.tex
\begin{table}[t]
 \centering  
 \caption{Comparison results on CIFAR100-LT with respect to top-1 classification accuracy (\%). 
 $\star$ denotes the results quoted from the corresponding papers.
 The other methods are reproduced using the same setup to maintain consistency in evaluation.
 The best and the second-best results are shown in \underline{\textbf{underline bold}} and \textit{\textbf{italic bold}}, respectively. } %
 \label{tab:cifar_results}
 \vspace{-0.5em}
 \setlength{\tabcolsep}{12pt}
 \renewcommand{\arraystretch}{1.2}
  {
  \begin{tabular}{l| c |c c c}
  \hlinew{1pt} 
  \multicolumn{1}{l|}{Imbalance factor} & Year &200 &100 &50 \\
  \hline 
  \multicolumn{5}{c}{CNN-based method (single model, backbone: ResNet-32.)} \\
  \hline 
  CE loss & - &35.99 & 39.55 &45.40\\
  LDAM-DRW \cite{Kaidi2019} & 2019 &38.91 &42.04 &47.62\\
  FSA~\cite{chu2020feature} & 2020 & 41.46 & 48.51 & 52.17\\  
  DR+MU~\cite{decouple20} & 2020 & 41.73 & 45.68 & 50.86\\
  MisLAS~\cite{mislas21} & 2021 & 43.45 &46.71 &52.31 \\
  MBJ$^\star$~\cite{Liu2022Memory} & 2022& - & 45.80 & 52.60 \\
  GCL~\cite{LiMK2022GCL} &2022 & 44.76 & 48.61 & 53.55 \\
  CE+CMO$^\star$~\cite{ParkS2022Majority} & 2022 & - & 43.90 & 48.30 \\
  BSCE+CMO$^\star$~\cite{ParkS2022Majority}& 2022 & - & 46.60 & 51.40 \\
  ABL$^\star$~\cite{jin2023optimal} & 2023 & - & 46.80 & 52.10 \\
  LCReg~\cite{liu2024lcreg} & 2024 &  43.60 &  47.35 & 53.09 \\
  \hdashline
  {DR+H2T} & 2024 & 43.95 & 47.73&  52.95 \\
  {MisLAS+H2T} & 2024 & 43.84  & 47.62  & 52.73 \\
  {GCL+H2T} & 2024 & 45.24 & 48.88  & 53.76 \\
  \hdashline
  \rowcolor{mygray}
  {DR+PI-H2T} & ours & \textit{\textbf{45.58}} & \textit{\textbf{49.22}}&  \textit{\textbf{54.05}} \\
  \rowcolor{mygray}
  {MisLAS+PI-H2T} & ours & 45.41  & 49.04  & 53.78 \\
  \rowcolor{mygray}
  {GCL+PI-H2T} & ours & \underline{\textbf{46.34}} &  \underline{\textbf{49.88}} & \underline{\textbf{54.13}} \\
  \hline
    \multicolumn{5}{c}{CNN-based method (multi-expert model, backbone: ResNet-32.)} \\
  \hline
  BBN~\cite{bbn20} & 2020  & 37.21 & 42.56 & 47.02 \\
  RIDE~\cite{WangXD21RIDE} & 2021  & 45.84 & 50.37 & 54.99 \\
  ACE$^\star$~\cite{Cai2021ACE}  & 2021 &- & 49.40 & 50.70 \\
  RIDE+CMO$^\star$~\cite{ParkS2022Majority} & 2022 &- & 50.00 & 53.00\\
  \hdashline 
  RIDE+H2T & 2024 & \textit{\textbf{46.64}}  & \textit{\textbf{51.38}} & \textit{\textbf{55.54}} \\
  ResLT+H2T  & 2024 & 46.18 & 49.60  & 54.39 \\
  \hdashline
  \rowcolor{mygray}
  RIDE+PI-H2T & ours & \underline{\textbf{47.81}}  & \underline{\textbf{52.51}} & \underline{\textbf{57.22 }}\\
  \hline 
  \multicolumn{5}{c}{MHSA-based method (backbone: ViT-B/16)} \\
  \hline 
  VPT~\cite{jia2022visual} & 2022 & 72.81 & 81.05 & 84.83 \\
  LPT~\cite{Dong23LPT} & 2023 & \textit{\textbf{87.91}} & \textit{\textbf{89.12}} & \textit{\textbf{90.00}} \\ 
  LIFT~\cite{shi2024LIFT} & 2024 &  79.03 & 81.35 & 82.00 \\ 
  \redtext{LIFT+DBM~\cite{son2025difficulty}} & \redtext{2025} &  \redtext{79.00} & \redtext{81.40} & \redtext{82.50} \\ 
  \rowcolor{mygray}
  LPT+PI-H2T & ours & \underline{\textbf{89.02}} & \underline{\textbf{90.05}} & \underline{\textbf{91.03}}   \\ 
  \rowcolor{mygray}
   LIFT+PI-H2T & ours & 79.74 &  82.40 & 83.29   \\
  \hlinew{1pt}
 \end{tabular}}
 \vspace{-9pt}
\end{table}


%% file: tab/img.tex
\begin{table}[tb]
 \centering  
 \caption{Comparison results on imageNet-LT\protect\footnotemark[2] with respect to top-1 classification accuracy (\%). Others are the same as Table~\ref{tab:cifar_results}.} \label{tab:img} 
 \vspace{-0.5em}
 \setlength{\tabcolsep}{8pt}
 \renewcommand{\arraystretch}{1.2}
  {
  \begin{tabular}{l | c | c c c |c}  
  \hlinew{1pt}
  Method & Year  &Head &Med &Tail & All \\
  \hline 
  \multicolumn{6}{c}{CNN-based method (single model, backbone: ResNet-50.} \\
  \hline 
  CE loss & - & 64.91 & 38.10  &11.28  &44.51 \\
  LDAM-DRW \cite{Kaidi2019} & 2019 & 58.63  &48.95  & 30.37 & 49.96 \\
  FSA$^\star$~\cite{chu2020feature} & 2020 & 47.30 & 31.60 & 14.70 & 35.20\\  
  DR~\cite{decouple20} & 2020 & 62.93 & 49.77 & 33.26 & 52.18 \\ 
  MisLAS~\cite{mislas21} & 2021 & 62.53 & 49.82 & 34.74 & 52.29 \\ 
  MBJ$^\star$~\cite{Liu2022Memory} &2022 & 61.60 & 48.40 & 39.00 & 52.10  \\
  GCL~\cite{LiMK2022GCL} &2022  & 62.24 & 48.62 & 52.12 & 54.51 \\ 
  CE+CMO$^\star$~\cite{ParkS2022Majority} &2022 & 67.00 & 42.30 & 20.50 & 49.10  \\
  BSCE+CMO$^\star$~\cite{ParkS2022Majority} &2022 & 62.00 & 49.10 & 36.70 & 52.30 \\
  ABL$^\star$~\cite{jin2023optimal} &2023 & 62.60 & 50.30 & 36.90 & 53.20 \\ 
  LCReg~\cite{liu2024lcreg} &2024 & 65.19 & 52.48 & 36.41 & 54.79 \\
  \hdashline 
  DR+H2T &2024  & 63.26 & 50.43 & 34.11 & 52.74\\
  MisLAS+H2T  &2024 & 62.42 & 51.07 & 35.36 &  52.90  \\ 
  GCL+H2T  &2024 & 62.36 & 48.75 & 52.15 & \textit{\textbf{54.62}} \\  
  \hdashline 
  \rowcolor{mygray}
  DR+PI-H2T  & ours & 63.64 & 51.50 &  34.13 & 53.37\\
  \rowcolor{mygray}
  MisLAS+PI-H2T & ours & 62.52 & 52.07 & 35.29 &  53.47  \\ 
  \rowcolor{mygray}
  GCL+PI-H2T & ours & 63.95 & 53.92 & 37.25 & \underline{\textbf{55.08}} \\    
  \hline
    \multicolumn{6}{c}{CNN-based method (multi-expert model, backbone: ResNet-50.)} \\
  \hline
  BBN~\cite{bbn20} & 2020 & - & - & - & 48.30  \\
  RIDE~\cite{WangXD21RIDE} & 2021 & 69.59 & 53.06  & 30.09 & 55.72 \\ 
  ACE$^\star$ ~\cite{Cai2021ACE} & 2021  & - & - & - & 54.70 \\
  RIDE+CMO$^\star$~\cite{ParkS2022Majority}  & 2022 & 66.40 & 53.90 & 35.60 & 56.20\\
  ResLT~\cite{Cui2022reslt} & 2023  & 59.39 & 50.97 & 41.29 & 52.66 \\ 
  \hdashline 
  RIDE+H2T & 2024 & 67.55 & 54.95 & 37.08 & \textit{\textbf{56.92}}  \\ 
  ResLT+H2T & 2024  & 62.29 & 52.29 & 35.31 & 53.39 \\  
  \hdashline 
  \rowcolor{mygray}
  RIDE+PI-H2T  & ours  & 66.68 & 55.99 & 37.77 &  \underline{\textbf{57.15}} \\ 
  \hline  
  \multicolumn{5}{c}{MHSA-based method (backbone: ViT-B/16)} \\
  \hline  
  LIFT~\cite{shi2024LIFT} & 2024 & 80.20 & 76.10 & 71.50 & \textit{\textbf{77.00 }} \\
  \redtext{LIFT + DBM \cite{son2025difficulty}} & \redtext{2025} & \redtext{78.90} & \redtext{76.00} & \redtext{73.30} & \redtext{76.80}  \\ 
  \hdashline
  \rowcolor{mygray}
  LIFT+PI-H2T & ours & 81.46 & 76.96 & 73.07 & \underline{\textbf{78.16}}\\   
  \hlinew{1pt}
 \end{tabular}}
 \vspace{-9pt}
\end{table}

%% file: tab/inat.tex
\begin{table}[tb]
\renewcommand{\thefootnote}{\fnsymbol{footnote}}
 \centering  
 \caption{Comparison results on iNaturalist 2018 with respect to top-1 classification accuracy (\%).
 Others are the same as Table~\ref{tab:cifar_results}.} \label{tab:iNat} 
 \setlength{\tabcolsep}{8pt}
 \renewcommand{\arraystretch}{1.2}
  {
  \begin{tabular}{ l | c |c c c |c }
  \hlinew{1pt}
  Method & Year &Head &Med &Tail & All \\
  \hline 
  \multicolumn{6}{c}{CNN-based method (single model, backbone: ResNet-50.)} \\
  \hline 
  CE loss & - & 76.10  & 69.05  & 62.44  & 66.86 \\
  FSA$^\star$~\cite{chu2020feature} & 2020 & - & - & - & 65.91 \\
  LDAM-DRW \cite{Kaidi2019} & 2019 & - &- &- & 68.15 \\
  DR~\cite{decouple20} & 2020 & 72.88 & 71.15 & 69.24 & 70.49\\  
  MisLAS~\cite{mislas21} & 2021  & 72.52 & 72.08 & 70.76 & 71.54  \\
  MBJ$^\star$~\cite{Liu2022Memory} &2022 & - & - & - & 70.00  \\
  GCL~\cite{LiMK2022GCL}  & 2022 & 66.43 & 71.66 & 72.47 & 71.47 \\ 
  CE+CMO$^\star$~\cite{ParkS2022Majority} & 2022 & 76.90 & 69.30 & 66.60 & 68.90 \\     
  BSCE+CMO$^\star$~\cite{ParkS2022Majority} & 2022 & 68.80 & 70.00 & 72.30 & 70.90 \\  
  ABL$^\star$~\cite{jin2023optimal} & 2023 & - & - & - & 71.60 \\  
  LCReg~\cite{liu2024lcreg} &2024 & 73.82 & 73.19 & 71.28 & \textit{72.41} \\
  \hdashline 
  DR+H2T  & 2024 & 71.73 & 72.32 & 71.30 & 71.81  \\ 
  MisLAS+H2T  & 2024 & 69.68 & 72.49 & 72.15 & \textit{\textbf{72.05}} \\       
  GCL+H2T & 2024 &67.74 & 71.92 & 72.22 & 71.62 \\
  \hdashline
  \rowcolor{mygray}   
  DR + PI-H2T  & ours&  71.77 & 71.55 & 71.03 & 71.34  \\ 
  \rowcolor{mygray}  
  MisLAS + PI-H2T  & ours & 71.66 & 72.22 & 71.03 & 71.64 \\   \rowcolor{mygray}  
  GCL+PI-H2T & ours & 70.74 & 75.41 & 75.62 & \underline{\textbf{75.02}} \\
  \hline
    \multicolumn{6}{c}{CNN-based method (single model, backbone: ResNet-50.)} \\
  \hline
  BBN~\cite{bbn20} & 2020 & - & - & - & 69.70  \\
  RIDE~\cite{WangXD21RIDE}  & 2021 & 76.52 & 74.23 & 70.45  & 72.80 \\   
  ACE~\cite{Cai2021ACE}$^\star$ & 2021 & - & - & - & 72.90 \\
  RIDE + CMO$^\star$~\cite{ParkS2022Majority} & 2022 & 68.70 & 72.60 & 73.10 & 72.80 \\ 
  ResLT~\cite{Cui2022reslt}$^\star$ & 2023 & 64.85 & 70.64 & 72.11 & 70.69 \\  
  \hdashline 
  RIDE + H2T & 2024 & 75.69 &  74.22 & 71.36  & \textit{\textbf{73.11}} \\ 
  ResLT + H2T & 2024 & 68.41 & 72.31 & 72.09 & 71.88  \\ 
  \hdashline
  \rowcolor{mygray}
  RIDE + PI-H2T & ours & 73.91 & 76.59 &  75.96 &  \underline{\textbf{76.03}}\\  
  \hline 
  \multicolumn{5}{c}{MHSA-based method (backbone: ViT-B/16)} \\
  \hline  
  LiVT$^\star$~\cite{Xu2023Learning}  & 2023  & 78.90 & 76.50 & 74.80 & 76.10 \\ 
  LPT~\cite{Dong23LPT} & 2023 & 61.45 & 77.01 & 79.12& 76.10\\
  Decoder$^\star$~\cite{wang2024exploring} &2024 & - & - & - & 59.20 \\ 
  LIFT~\cite{shi2024LIFT} & 2024 & 72.40 & 79.00 & 81.10 & 79.10 \\
  \redtext{LIFT+DBM \cite{son2025difficulty}} & \redtext{2025} & \redtext{70.80} & \redtext{78.90} & \redtext{83.50} & \redtext{\textit{\textbf{79.90}}}  \\
  \hdashline
  \rowcolor{mygray}
  LPT+PI-H2T & ours & 65.68 & 77.60 & 79.19 &  77.07\\  
  \rowcolor{mygray}
  LIFT+PI-H2T & ours & 74.11 & 80.16 & 81.77 & \underline{\textbf{80.17 }} \\
  \hlinew{1pt}
 \end{tabular}
 }
\end{table}

%% file: tab/pla.tex
\begin{table}[tb]
 \centering  
 \caption{Comparison results on Places-LT with respect to top-1 classification accuracy (\%).    
 CC, cosine classifier. LC, linear classifier. 
 Others are the same as Table~\ref{tab:cifar_results}.} 
 \label{tab:pla} 
 \setlength{\tabcolsep}{8pt}
 \renewcommand{\arraystretch}{1.2}
  {
  \begin{tabular}{l | c| c c c |c}
  \hlinew{1pt} 
  Method & Year  &Head &Med &Tail & All \\
  \hline 
  \multicolumn{6}{c}{CNN-based method (single model, backbone: ResNet-152.)} \\
  \hline 
  CE loss &- & 46.48  &25.66  &8.09  & 29.43 \\  
  FSA$^\star$~\cite{chu2020feature} & 2020 & 42.80 & 37.50 & 22.70 & 36.40 \\
  DR~\cite{decouple20} &2020 & 41.66 &37.79 &  32.77 & 37.40\\
  MisLAS~\cite{mislas21}&2021 & 41.95 & 41.88 & 34.65 & 40.38 \\ 
  MBJ$^\star$~\cite{Liu2022Memory} &2022 & 39.50 & 38.20 & 35.50 & 38.10  \\   
  GCL~\cite{LiMK2022GCL}  &2022 & 38.64 &  42.59 & 38.44 &  40.30  \\
  ABL$^\star$~\cite{jin2023optimal} &2023 & 41.50 & 40.80 & 31.40 & 39.40 \\ 
  LCReg~\cite{liu2024lcreg} & 2024 & 43.65 & 41.15 & 31.46 & \textit{40.01} \\
  \hdashline 
  DR+H2T &2024 & 41.96 & 42.87 & 35.33 & \textit{\textbf{40.95}} \\ 
  MisLAS+H2T &2024 & 41.40 & 43.04 & 35.95 & \underline{\textbf{41.03}} \\  
  GCL+H2T  &2024 & 39.34 & 42.50 & 39.46 & 40.73 \\   
  \hdashline 
  \rowcolor{mygray}  
  DR+PI-H2T &ours & 42.36 & 42.60 & 34.23 & 40.75 \\  
  \rowcolor{mygray}  
  MisLAS+PI-H2T &ours & 41.12 & 42.90 & 35.44 & 40.69 \\
  \rowcolor{mygray}  
  GCL+PI-H2T  &ours & 41.29 & 43.31 & 33.74& 40.56 \\    
  \hline
    \multicolumn{6}{c}{CNN-based method (multi-expert model, backbone: ResNet-152.)} \\
  \hline
  RIDE~\cite{WangXD21RIDE} (LC)  &2021 & 44.79 & 40.69 & 31.97  & 40.32 \\ 
  RIDE~\cite{WangXD21RIDE} (CC)  &2021 & 44.38 & 40.59 & 32.99 & 40.35   \\ 
  ResLT$^\star$~\cite{Cui2022reslt} &2023 & 40.30 & 44.40 & 34.70 & 41.00 \\ 
  \hdashline 
  RIDE + H2T (LC)  &2024 & 42.99 & 42.55 & 36.25 & \underline{\textbf{41.38}} \\
  RIDE + H2T (CC)  &2024 & 42.34 & 43.21 & 35.62 & 41.30 \\  
  \hdashline
  \rowcolor{mygray}
  RIDE+PI-H2T (LC) & ours &  42.80 &  41.75 &37.57 & \textit{\textbf{41.31}} \\
  \rowcolor{mygray}
  RIDE+PI-H2T (CC)  & ours & 42.66 & 42.00 & 37.33 & 41.25 \\  
  \hline 
  \multicolumn{5}{c}{MHSA-based method (backbone: ViT-B/16)} \\
  \hline  
  LiVT$^\star$~\cite{Xu2023Learning}  & 2023  & 48.10 & 40.60 & 27.50 & 40.80 \\ 
  LPT~\cite{Dong23LPT} & 2023 & 47.60 & 52.10 &48.40 & 49.70 \\
  Decoder$^\star$~\cite{wang2024exploring} &2024 & - & - & - & 46.80 \\ 
  LIFT~\cite{shi2024LIFT} & 2024 &51.30 & 52.20 & 50.50 & \textit{\textbf{51.50 }} \\
  \redtext{LIFT+DBM \cite{son2025difficulty}} & \redtext{2025}  & \redtext{50.40} & \redtext{52.00} & \redtext{52.50} & \redtext{\textit{\textbf{51.50}}} \\ 
  \hdashline
  \rowcolor{mygray}
  LPT+PI-H2T & ours & 46.93 & 52.61 & 50.57&  50.12\\  
  \rowcolor{mygray}
  LIFT+PI-H2T & ours & 51.24 & 52.06 & 51.37 & \textit{\textbf{51.63 }}\\   
  \hlinew{1pt}

 \end{tabular}
}
\end{table}

%% file: tab/comp_w_h2t.tex
\begin{table}[t]
\renewcommand{\thefootnote}{\fnsymbol{footnote}}
 \centering  
 \caption{\redtext{Quantitative comparison with H2T (the conference version).
 All numbers denote top-1 accuracy (\%).
 The backbone is based on ResNet and the loss function is CE.}}
 \label{tab:comp_h2T-ce}
 \resizebox{1.\linewidth}{!}
 {\redtext{
 \renewcommand{\arraystretch}{1.2}
  \begin{tabular}{ l|c  c c | c c }
   \hlinew{1pt}
  \multirow{2}{*}{Dataset} &\multicolumn{3}{c|}{H2T} &  \multicolumn{2}{c}{PI-H2T}\\ 
  \cline{2-6}
          & St.1 & St.2 & St.2 (Auto $r$) & St.1 & St. 2 \\
  \hline
  CIFAR100-LT ($IF=100$) & 39.55 & 47.73 & 48.37 & 42.19  & 49.22\\ 
  imageNet-LT   & 44.51 & 52.74 & 52.99  & 47.48 & 53.37\\ 
  Naturalist 2018 & 66.86 & 71.81 & 71.56  & 67.65  & 71.34\\  
  Places-LT & 29.43 & 40.95 &  40.74 & 29.92 & 40.75\\ 
  \hlinew{1pt}
 \end{tabular}}
 }
\end{table}

\begin{table}[t]
\renewcommand{\thefootnote}{\fnsymbol{footnote}}
 \centering  
 \caption{\redtext{Quantitative comparison with H2T (the conference version).
 All numbers denote top-1 accuracy (\%).
 The backbone is based on ResNet and the loss function is GCL.}}
 \label{tab:comp_h2T-gcl}
 \resizebox{1.\linewidth}{!}
 {\redtext{
 \renewcommand{\arraystretch}{1.2}
  \begin{tabular}{ l|c  c c | c c }
   \hlinew{1pt}
  \multirow{2}{*}{Dataset} &\multicolumn{3}{c|}{H2T} &  \multicolumn{2}{c}{PI-H2T}\\ 
  \cline{2-6}
          & St.1 & St.2 & St.2 (Auto $r$) & St.1 & St. 2 \\
  \hline
  CIFAR100-LT ($IF=100$) & 46.57  & 48.88 & 49.34  & 47.54   & 49.88\\ 
  imageNet-LT  &  51.93 & 54.62 &  54.79  & 52.88  & 55.08\\ 
  Naturalist 2018 & 70.33 & 71.62 & 71.87  &  73.39  & 75.02\\  
  Places-LT & 33.93 & 40.73 &  40.79  &  35.00 & 40.56\\ 
  \hlinew{1pt}
 \end{tabular}}
 }
\end{table}

%% file: fig/boundary/class09-2.tex
\begin{figure}
\begin{minipage}[c]{1.\linewidth} 
    \centering  
    \subfigure[Decision boundary by PIF st.1]{\label{fig:PI-st109} 
    \includegraphics[width=.48\linewidth, height=.42\linewidth]{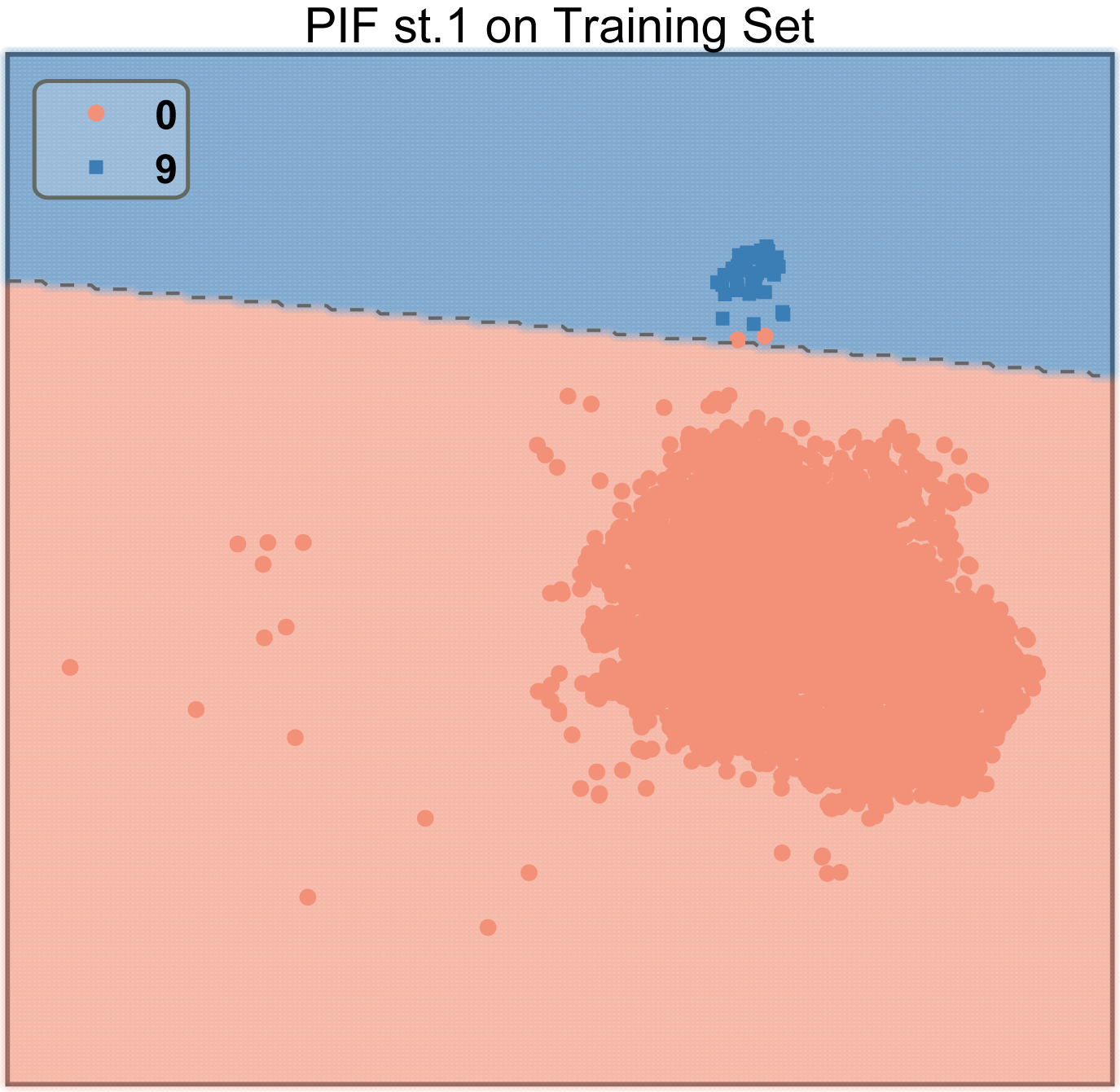}\label{fig:PI-st109-train09} 
    \includegraphics[width=.48\linewidth, height=.42\linewidth]{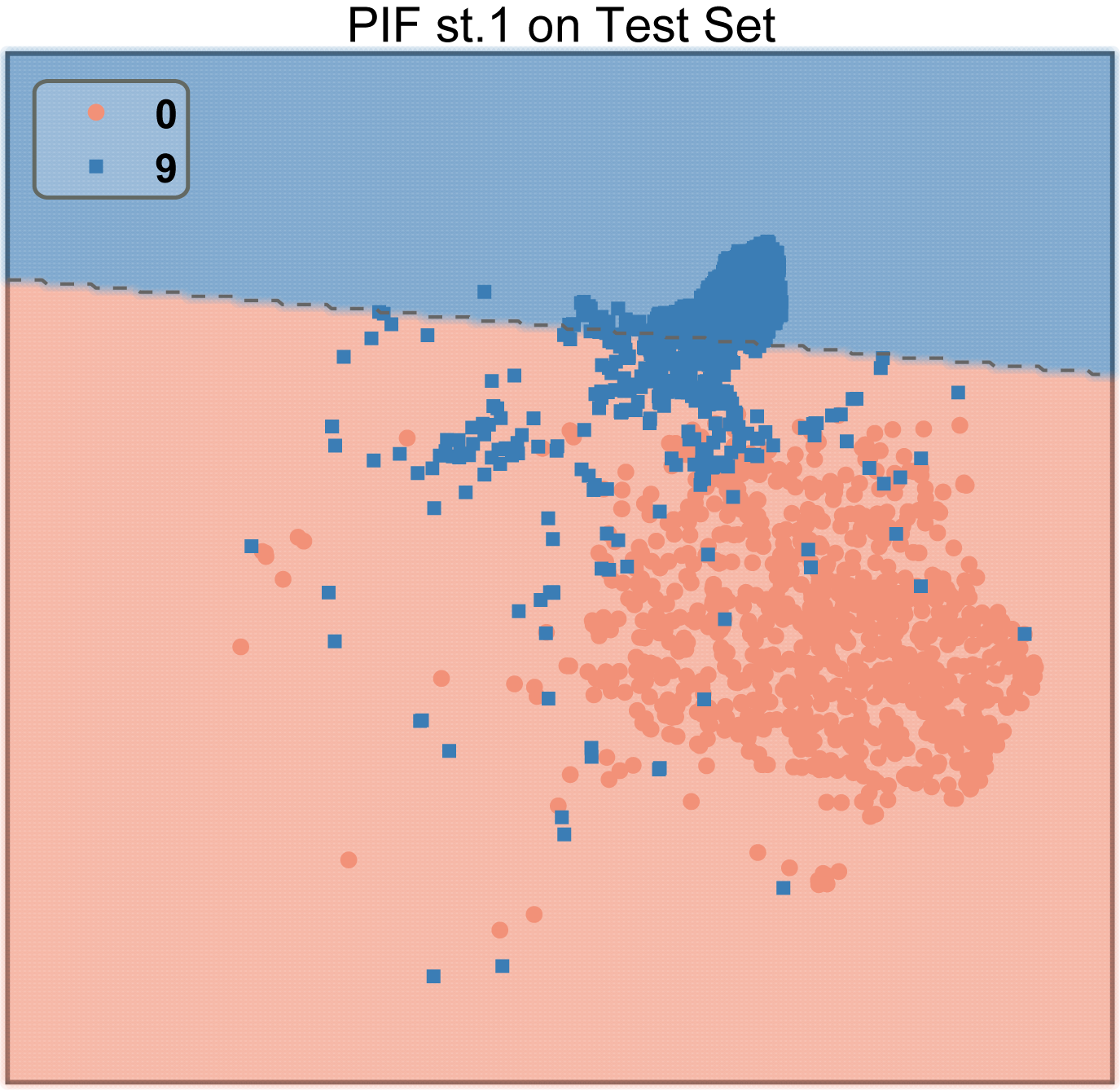}\label{fig:PI-st109-test09}  
    }
    \\
    \subfigure[Decision boundary by PIF+DR]{\label{fig:PI-DR09} 
    \includegraphics[width=.48\linewidth, height=.42\linewidth]{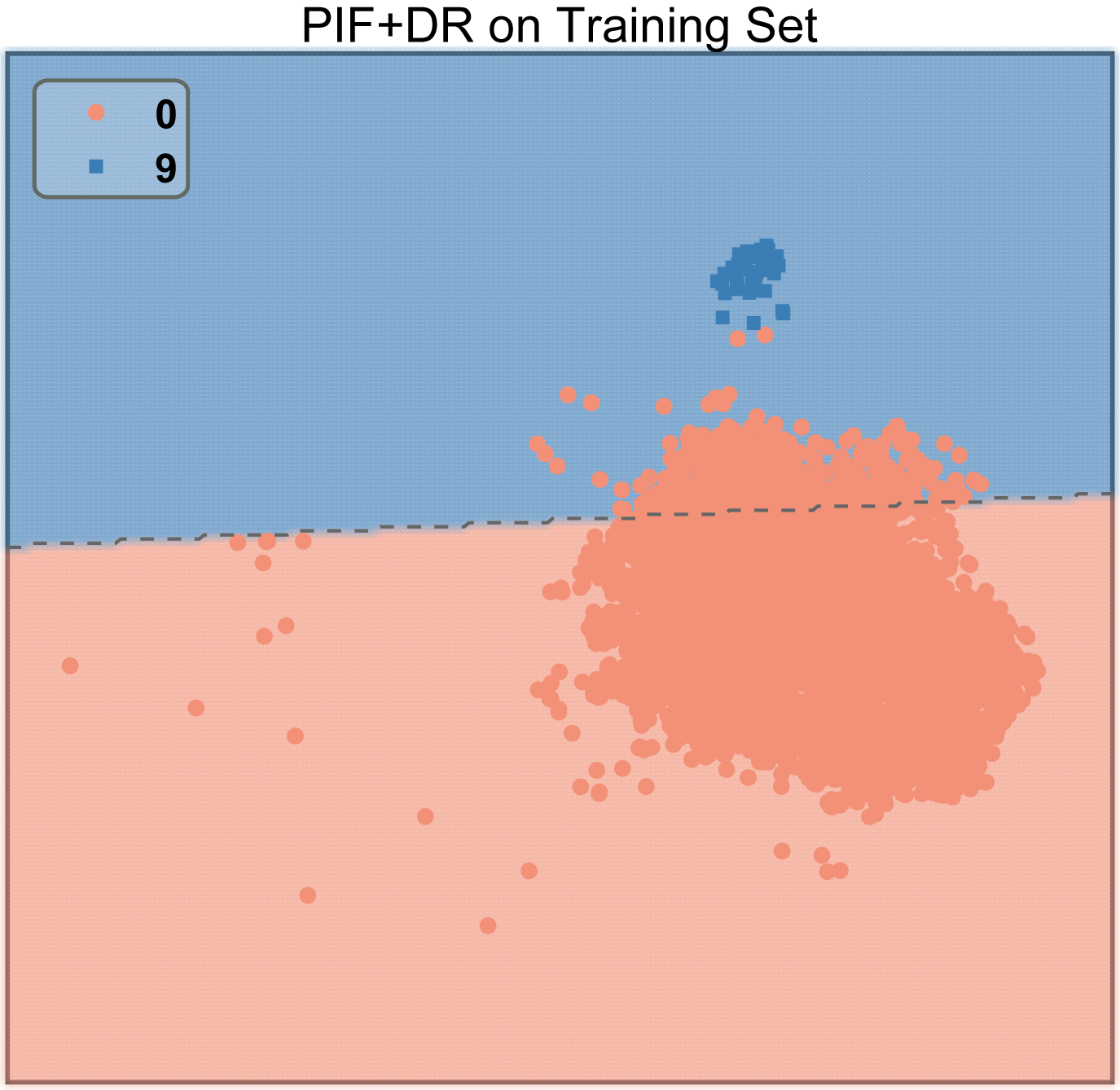}\label{fig:PI-DR-train09} 
    \includegraphics[width=.48\linewidth, height=.42\linewidth]{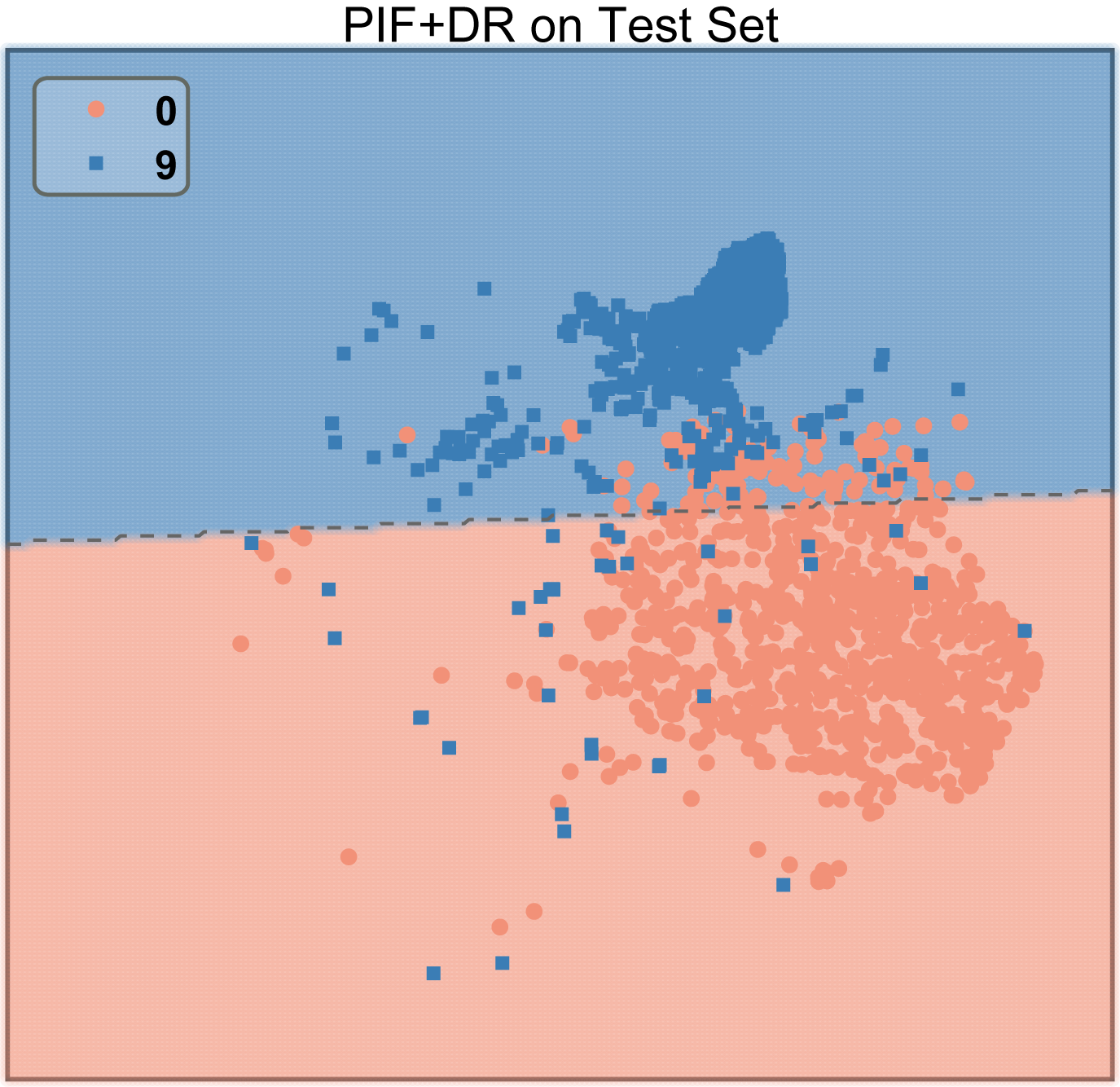}\label{fig:PI-DR-test09}  
    }
    \\
    \subfigure[Decision boundary by PI-H2T]{\label{fig:PI-H2T09} 
    \includegraphics[width=.48\linewidth, height=.42\linewidth]{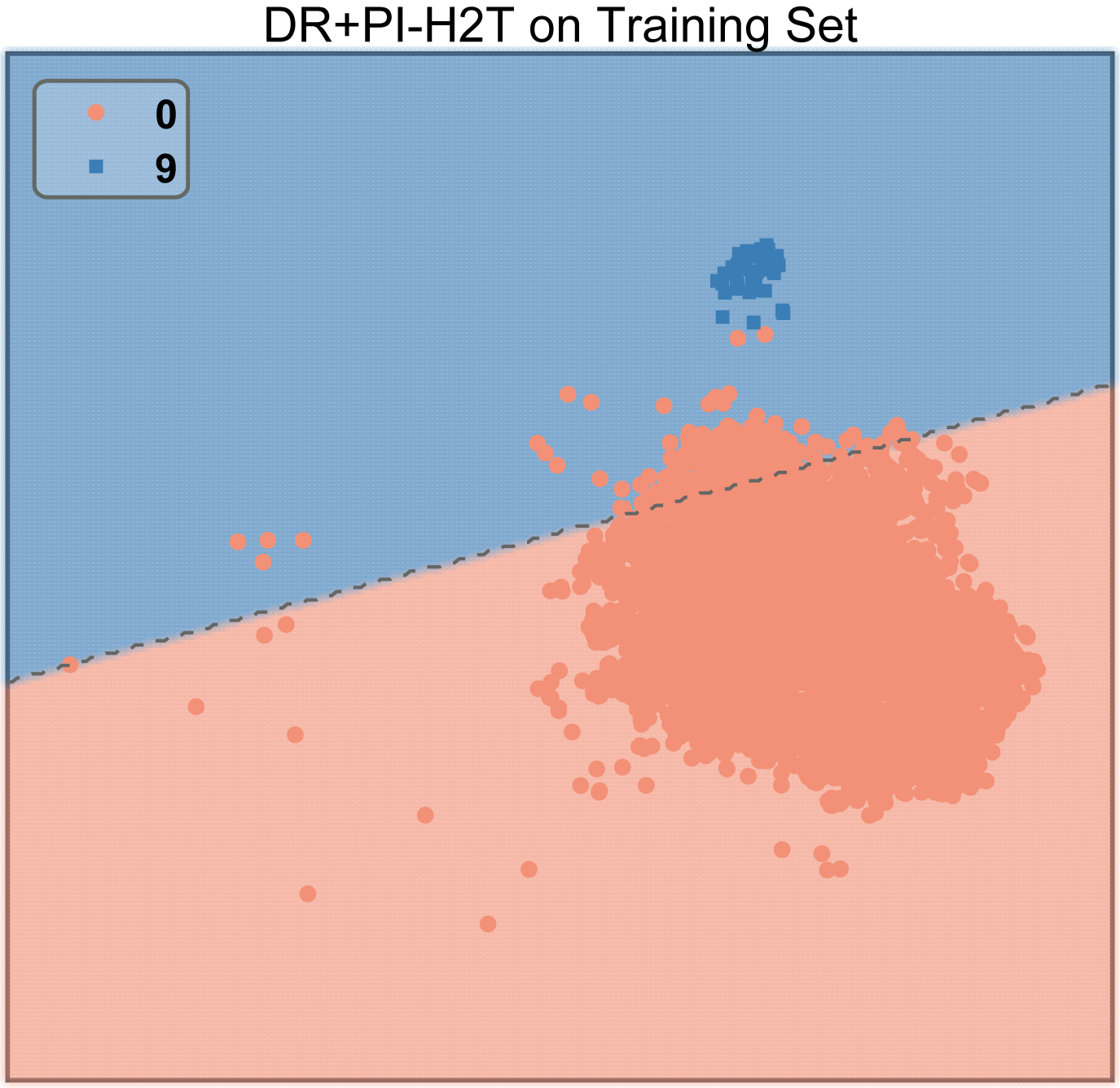}\label{fig:PI-H2T-train09} 
    \includegraphics[width=.48\linewidth, height=.42\linewidth]{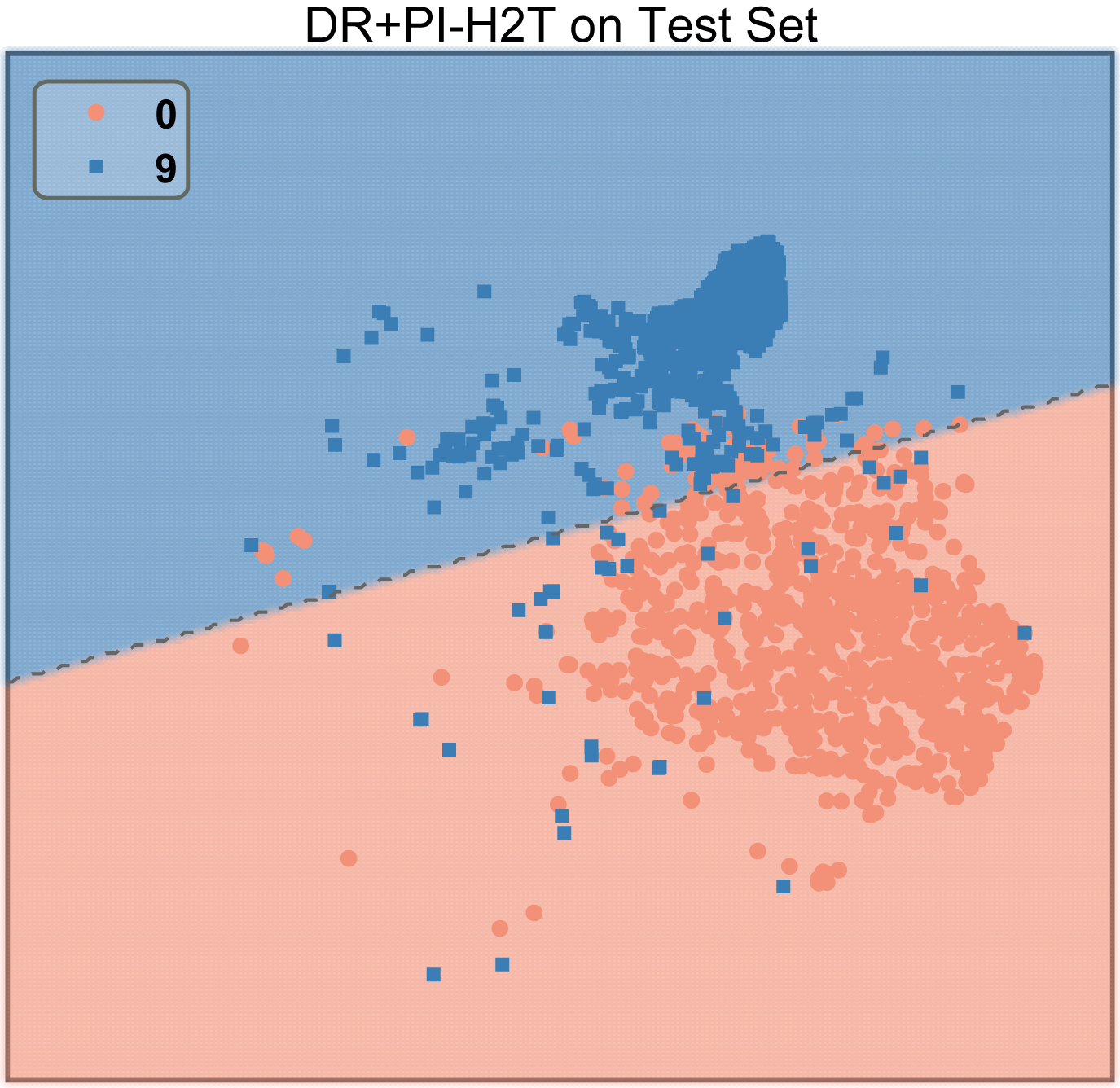}\label{fig:PI-H2T-test09}  
    }
\caption{Decision boundary between Class 0 and Class 9 on CIFAR10-LT with an imbalance factor of 100.}
\label{fig:boundary09}
\end{minipage}
\end{figure}

%% file: fig/ablation/PIF-tsne-mini.tex
\begin{figure}[t]
\centering
  \subfigure[w.o. H2TF\label{fig:sub1}]{
    \includegraphics[width=0.48\linewidth]{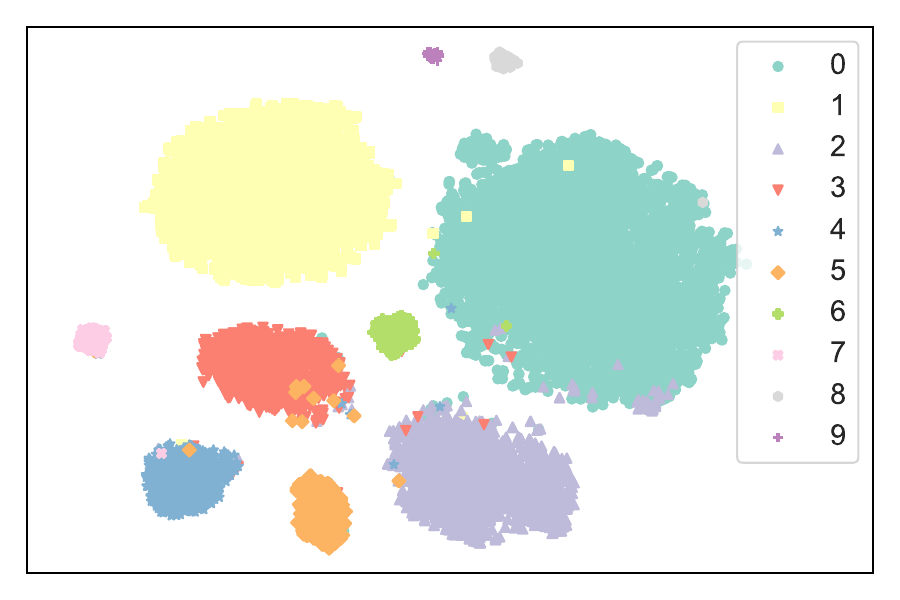}
  } \hspace{-1.3em}
  \subfigure[w. H2TF ($r=1$)\label{fig:sub2}]{
    \includegraphics[width=0.475\linewidth]{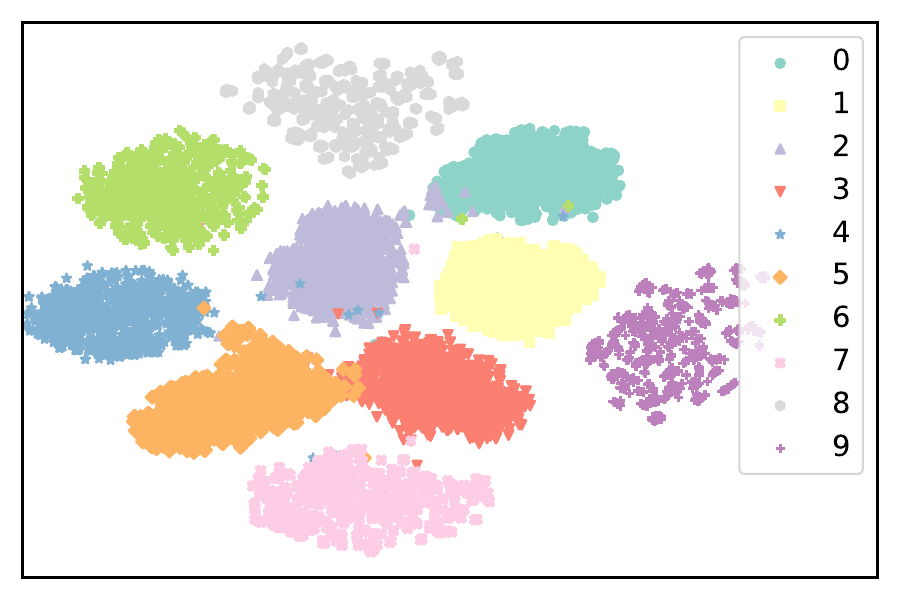}
  } 
  \vspace{-1em}
  \subfigure[w. H2TF ($r=0.8$)\label{fig:sub4}]{
    \includegraphics[width=0.48\linewidth]{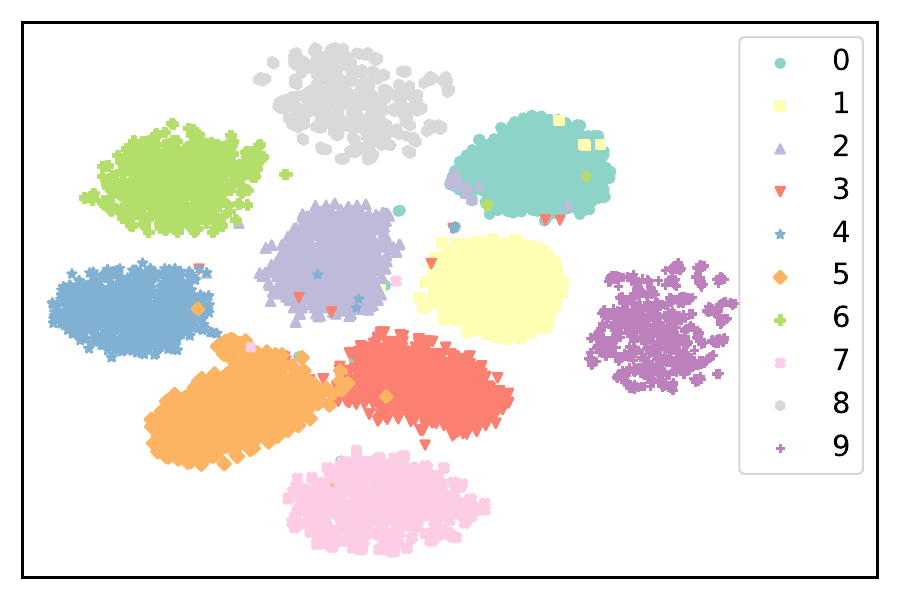}
  } \hspace{-1.3em}
  \subfigure[w. H2TF ($r=0.6$)\label{fig:sub6}]{
    \includegraphics[width=0.48\linewidth]{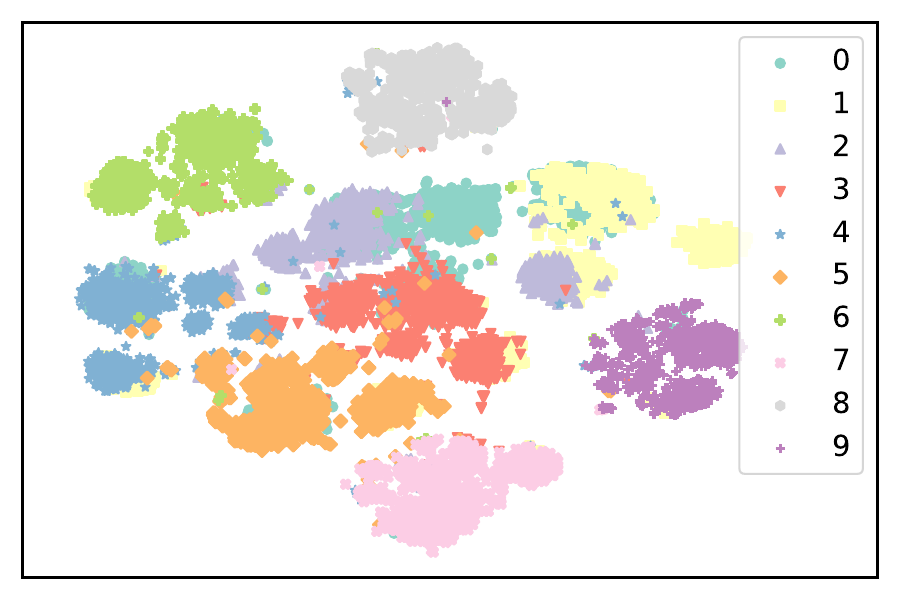}
  }  
  \vspace{-1em}
  \subfigure[w. H2TF ($r \propto d_x$)\label{fig:sub8}]{
    \includegraphics[width=0.48\linewidth]{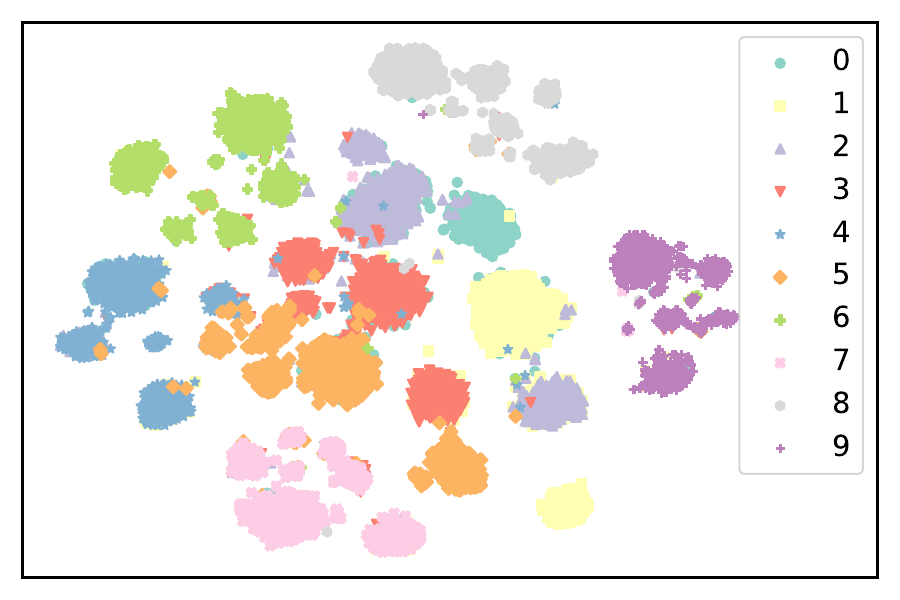} 
    }\hspace{-1.0em}
  \subfigure[w. H2TF ($r=0.3$)\label{fig:sub10}]{
    \includegraphics[width=0.48\linewidth]{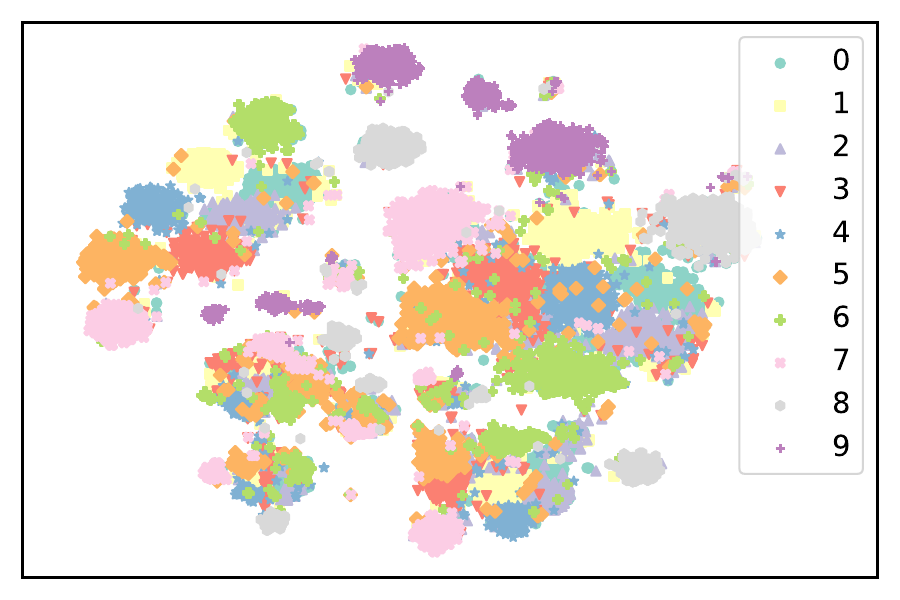}
  } 
  \caption{t-SNE visualization of feature distributions with varying fusion ratios ($r$ in Eq.~\ref{eq:fh2t}) in H2TF strategy} \label{fig:tsne_PIH2T}
  \vspace{-0.5em}
\end{figure}

%% file: tab/PI-abl.tex
\begin{table}[t]
\renewcommand{\thefootnote}{\fnsymbol{footnote}}
 \centering 
 \caption{Top-1 classification accuracy (\%) comparison for PI operation.}
 \label{tab:PI_abl}
 \resizebox{0.75\linewidth}{!}  
 {\setlength{\tabcolsep}{16pt}
 \renewcommand{\arraystretch}{1.2}
  \begin{tabular}{>{\color{black}}c | >{\color{black}}c >{\color{black}}c }
   \hlinew{1pt}
  PI Operation  &Stage 1 & Stage 2 \\
  \hline
  Baseline & 39.55  & 45.68  \\ 
  \hdashline
  Mean & 42.19  & 49.22  \\ 
  Max.  & 42.23  & \underline{\textbf{49.53}} \\ 
  Min.  & \underline{\textbf{42.37}}  & 48.53 \\ 
  \hlinew{1pt}
 \end{tabular}
 }\vspace{-0.5em}
\end{table}

%% file: tab/cifar10-LT.tex
\begin{table}[t]
\renewcommand{\thefootnote}{\fnsymbol{footnote}}
 \centering  
 \caption{Top-1 accuracy (\%) on CIFAR10-LT with an imbalance factor of 100. st.1 represents stage 1 of the corresponding methods.}
 \label{tab:cifar10}
 \resizebox{1.\linewidth}{!}
 {
 \renewcommand{\arraystretch}{1.2}
  \begin{tabular}{ c  c c | c c c}
   \hlinew{1pt}
  DR (st.1) & DR & DR+H2TF & PIF (st.1) & PIF+DR & PI-H2TF\\
  \hline
  72.66 & 79.15 & 82.87 & 73.65 & 82.50 & 83.04  \\ 
  \hlinew{1pt}
 \end{tabular}
 }
\end{table}

%% file: fig/ablation/fusion_ratio_his.tex
\begin{figure}[htpb]  
    \centering
    \vspace{-1em}
    \subfigure[\redtext{Distribution of $r$ on CIFAR10-LT at best epoch (Acc.: 83.04\%).}]{
        \includegraphics[width=0.95\linewidth]{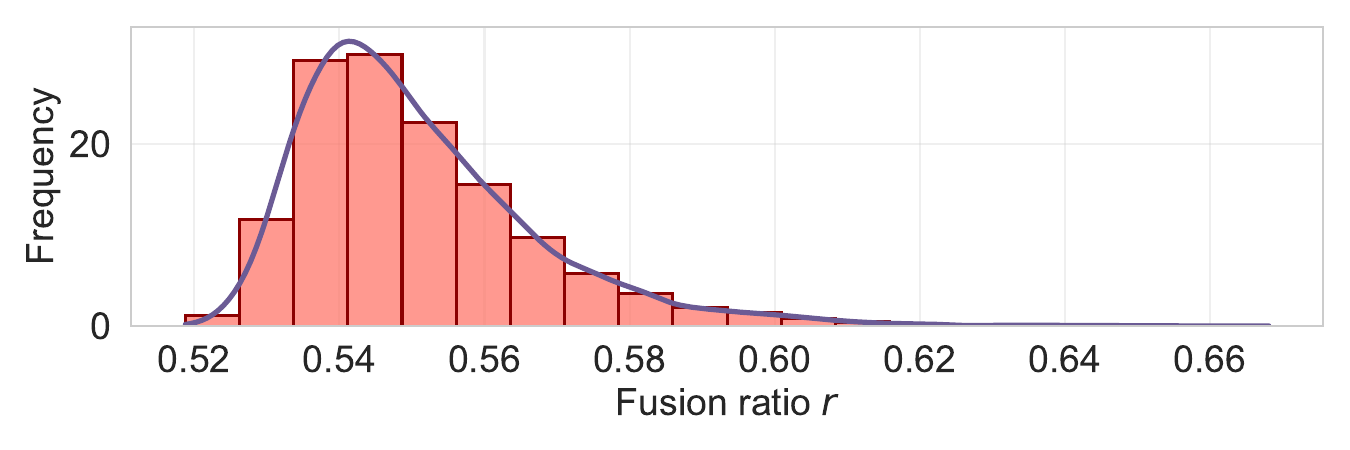}
        \label{fig:sub_hist_best-cifar10}
    }
    \subfigure[\redtext{Distribution of $r$ on CIFAR10-LT at last epoch (Acc.: 82.94\%).}]{
        \includegraphics[width=0.95\linewidth]{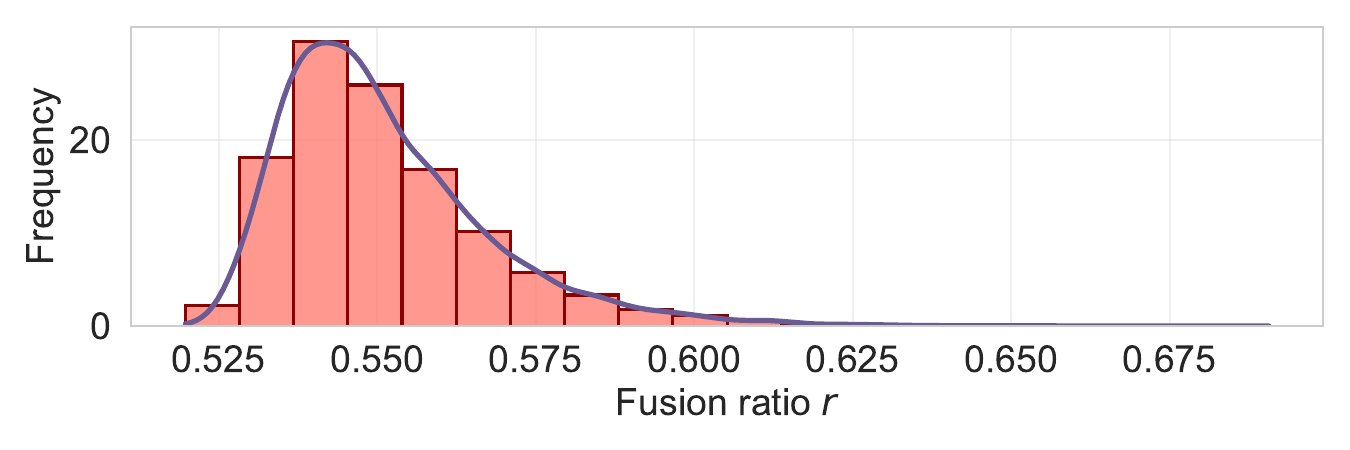}
        \label{fig:sub_hist_last-cifar10}
    }    
    \begin{tikzpicture}
        \draw[dashed, gray, very thick] (0,0) -- (9,0);
    \end{tikzpicture}
    \subfigure[\redtext{Distribution of $r$ on CIFAR100-LT at best epoch (Acc.: 49.22\%).}]{
        \includegraphics[width=0.95\linewidth]{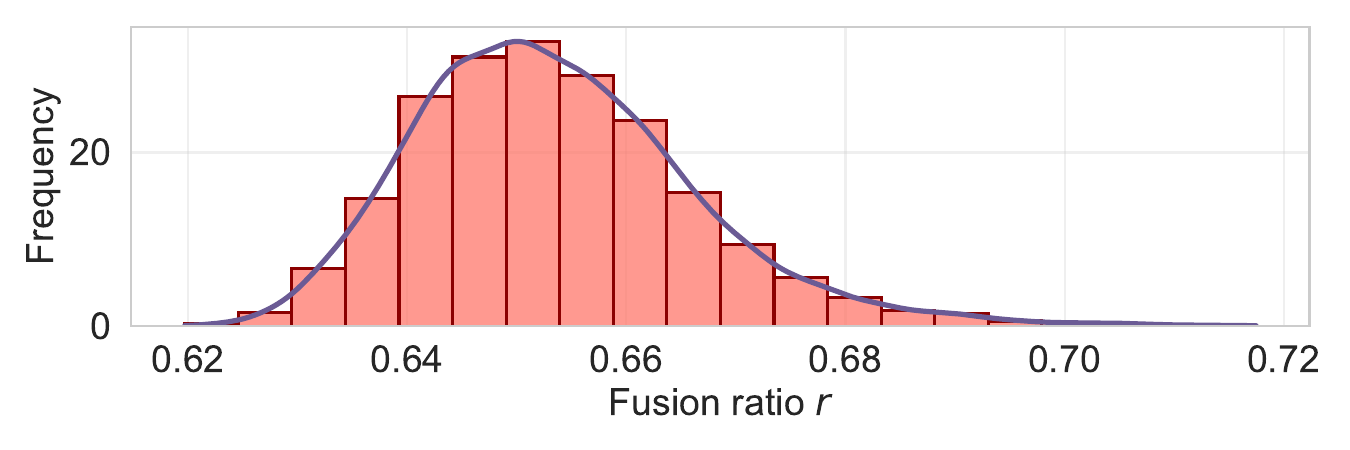}
        \label{fig:sub_hist_best}
    }
    \subfigure[\redtext{Distribution of $r$ on CIFAR100-LT at last epoch (Acc.: 49.04\%).}]{
        \includegraphics[width=0.95\linewidth]{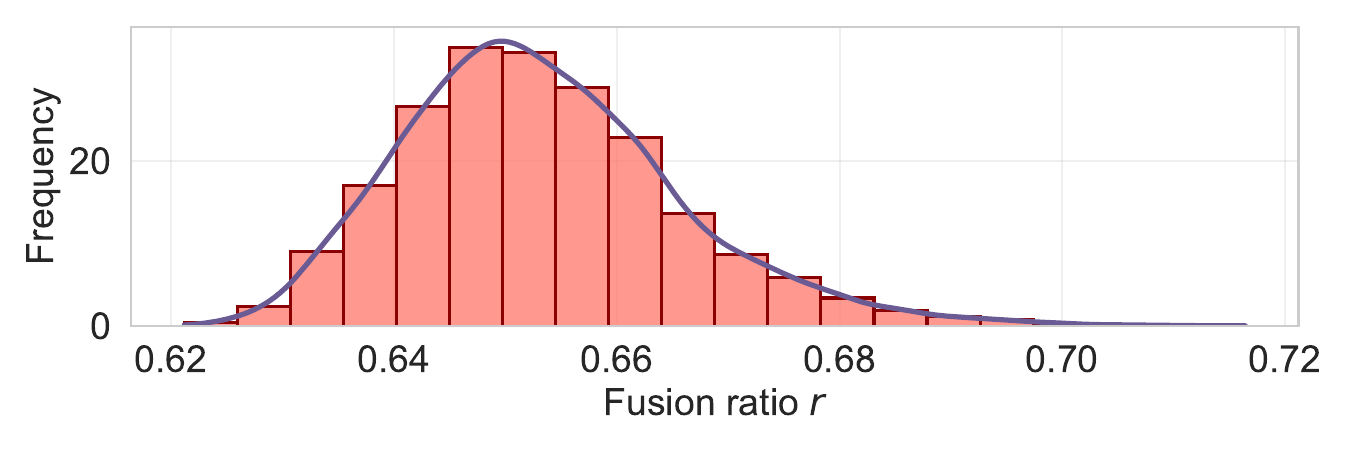}
        \label{fig:sub_hist_last}
    }
    \caption{\redtext{Distribution of automatically determined fusion ratios ($r$) in the proposed PI-H2T framework (IF = 100).}}
    \label{fig:fusion_hist}
    \vspace{-1em}
\end{figure}

%% file: tab/PI-baselines.tex
\begin{table}[t]
\renewcommand{\thefootnote}{\fnsymbol{footnote}}
\centering
\redtext{
\caption{Ablation study of PIF and related baselines.}
\label{tab:pif_ablation}
\resizebox{1.0\linewidth}{!}{
\renewcommand{\arraystretch}{1.2}
\begin{tabular}{l|l|c|c|c}
\hlinew{1pt}
Variant & Discription & Params & FLOPs & Acc. (\%) \\
\hline
Baseline  & Original ResNet-32 & -- & -- & 39.55 \\
+ Conv1D & Without PI feature  & \textbf{+2} & \textbf{$\approx$ 0} & 39.78 \\
+ SE (r=16) & Squeeze-and-excitation block  & +3.36K & + 0.146M &  40.55  \\
+ LP  & Learned pooling block & +0.56K & +0.143M  & 40.39  \\
\textbf{+ PIF} & PI feature fusion & \textbf{+2} & \textbf{$\approx$ 0} & 42.19 \\
\textbf{+ SE, + PIF}  & - & +3.36K & + 0.146M  & 42.31 \\
\textbf{+ LP, + PIF}  & - & +0.56K & + 0.143M  & 42.29 \\
\hlinew{1pt}
\end{tabular}}
}
\end{table}

%% file: Sec/5_Conclusion.tex
\section{\redtext{Concluding Remarks}}
Severe imbalance in long-tailed data leads to two major challenges in visual recognition: a deformed representation space and a biased classifier. 
In this paper, we propose a general approach to enhance existing methods and address these issues, introducing permutation-invariant and head-to-tail feature fusion (PI-H2T).
First, PI-H2T re-extracts representative features through permutation-invariant representation fusion (PIF) to improve representation quality. 
Next, it re-trains the classifier by grafting partial semantics from head to tail classes, calibrating the biased classifier through head-to-tail fusion (H2TF). 
This method introduces \redtext{only two additional learnable parameters, resulting in negligible computational and memory overhead}, while maintaining the backbone structure.
As a result, PI-H2T is adaptable and compatible with a wide range of existing techniques, making integration seamless.
Extensive experiments demonstrate that PI-H2T improves upon state-of-the-art methods.
Moreover, we apply it to a multi-head self-attention backbone, which also yields significant performance gains.

\redtext{
Despite these advantages, PI-H2T occasionally exhibits a slightly higher Expected Calibration Error (ECE), suggesting that confidence calibration could be further improved.
In addition, the theoretical interaction between representation fusion and classifier calibration remains an open question that deserves deeper study.
Future work will focus on calibration-aware extensions of H2TF to reduce ECE without sacrificing accuracy, and on extending PI-H2T to multi-modal, continual, and open-world recognition scenarios.
}

%% file: tab/algorithm.tex
\begin{algorithm}[htpb]
\renewcommand{\algorithmicrequire}{\textbf{Input:}}
\renewcommand{\algorithmicensure}{\textbf{Output:}}
\caption{PI-H2T}\label{alg:h2t}
\begin{algorithmic}[1]
\REQUIRE {Training set, fusion ratio $p$ \;}
\ENSURE {Trained model\;}
\STATE Initialize the model $\phi$\ randomly 
\FOR {$iter=1$ to $E_1$}
\STATE Obtain the feature map $\mathcal{F} = \phi_{\theta}(x)$ 
\STATE Fuse feature maps by Eq.~(\ref{eq:PIF}) to obtain PI fused feature maps $\mathcal{F}_{fuse}$ \; 
\STATE Input $\mathcal{F}$ to pooling layer to obtain $\mathbf{f}$, then calculate the logits by $z = \mathbf{W}^T \mathbf{f}$ and the loss by $\mathcal{L}_1(x,y)$ \;
\STATE $\phi = \phi - \alpha \nabla_{\phi} \mathcal{L}_1((x,y);\phi)$
\ENDFOR
\FOR {$iter=E_1+1$ to $E_2$}
\STATE Sample batches of data $(x^B,y^B)\sim \mathcal{T}^B$ from the class-balanced sampling data and $(x,y)\sim \mathcal{T}^I$ from the instance-wise sampling data\;
\STATE Obtain feature maps $\mathcal{F}^B = \phi_{\theta}(x^B)$ and $\mathcal{F}^I = \phi_{\theta}(x)$ \;
\STATE Fuse feature maps by Eq.~(\ref{eq:PIF}) to obtain PI fused feature maps $\mathcal{F}^B_{fuse}$ and $\mathcal{F}^I_{fuse}$ \; 
\STATE Input $\mathcal{F}^B_{fuse}$ and $\mathcal{F}^I_{fuse}$ to pooling layer to obtain $\mathbf{f}^B$ and $\mathbf{f}^I$ \;
\STATE Fuse features by Eq~(\ref{eq:fh2t}) to obtain H2T fused feature $\tilde{\mathbf{f}}$ \;
\STATE Calculate the logits by $\tilde{z} = \mathbf{W}^T \tilde{\mathbf{f}}$ and the loss by $\mathcal{L}_2(x^B,y^B)$ \;
\STATE Froze the parameters of representation learning $\phi^r$, and finetune the classifier parameters $\phi^c$:
$\phi^c = \phi^c- \alpha \nabla_{\phi^c} \mathcal{L}_2((x^B,y^B);\phi^c)$.
\ENDFOR
\end{algorithmic}  
\end{algorithm}

%% file: fig/ablation/PIF-tsne.tex
\begin{figure*}[!htpb]
\centering
  \hspace*{-2em}
  \subfigure[w.o. H2TF\label{appfig:sub1}]{
    \includegraphics[width=0.2\linewidth]{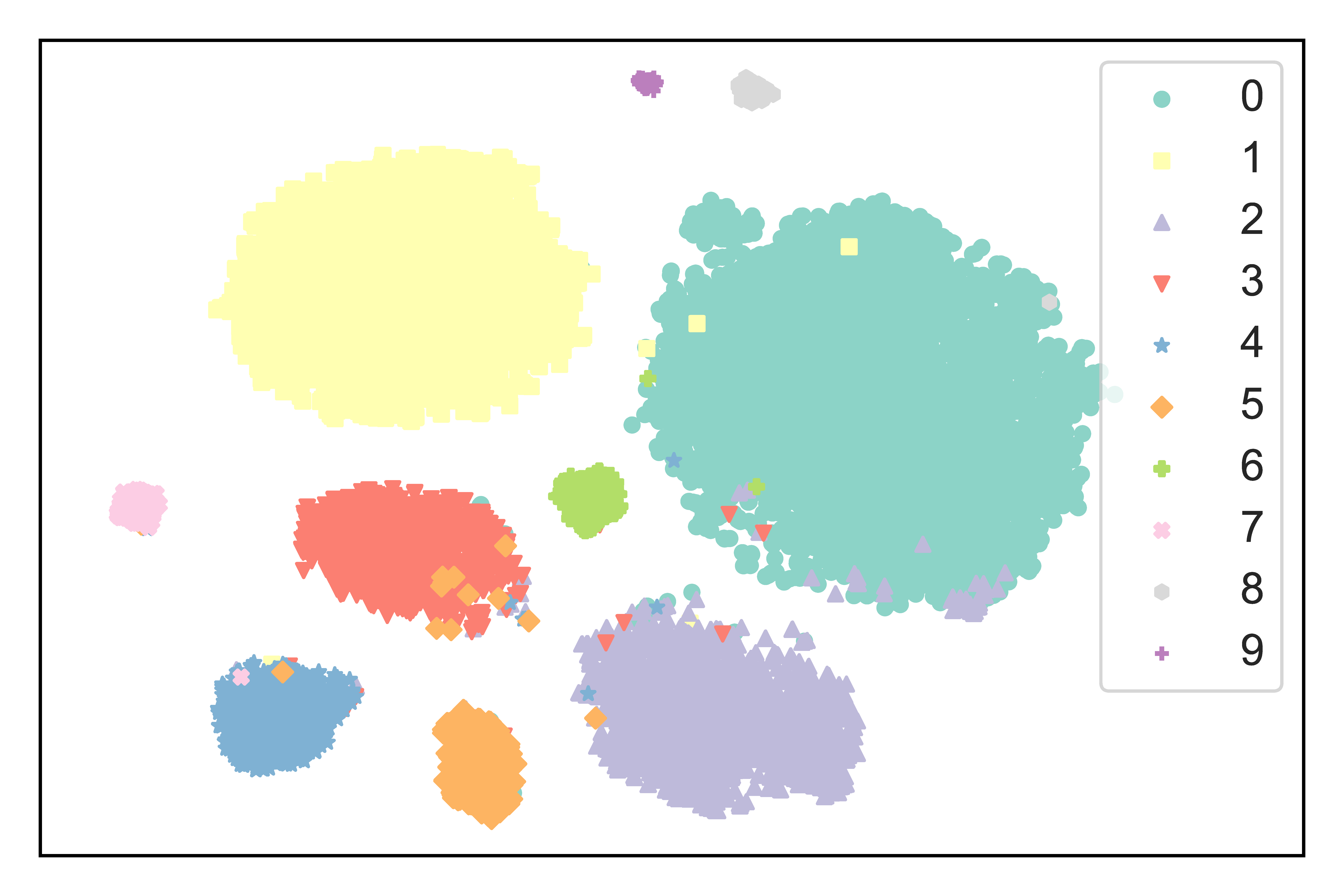}
  } \hspace{-1.2em}
  \subfigure[w. H2TF ($r=1$)\label{appfig:sub2}]{
    \includegraphics[width=0.2\linewidth]{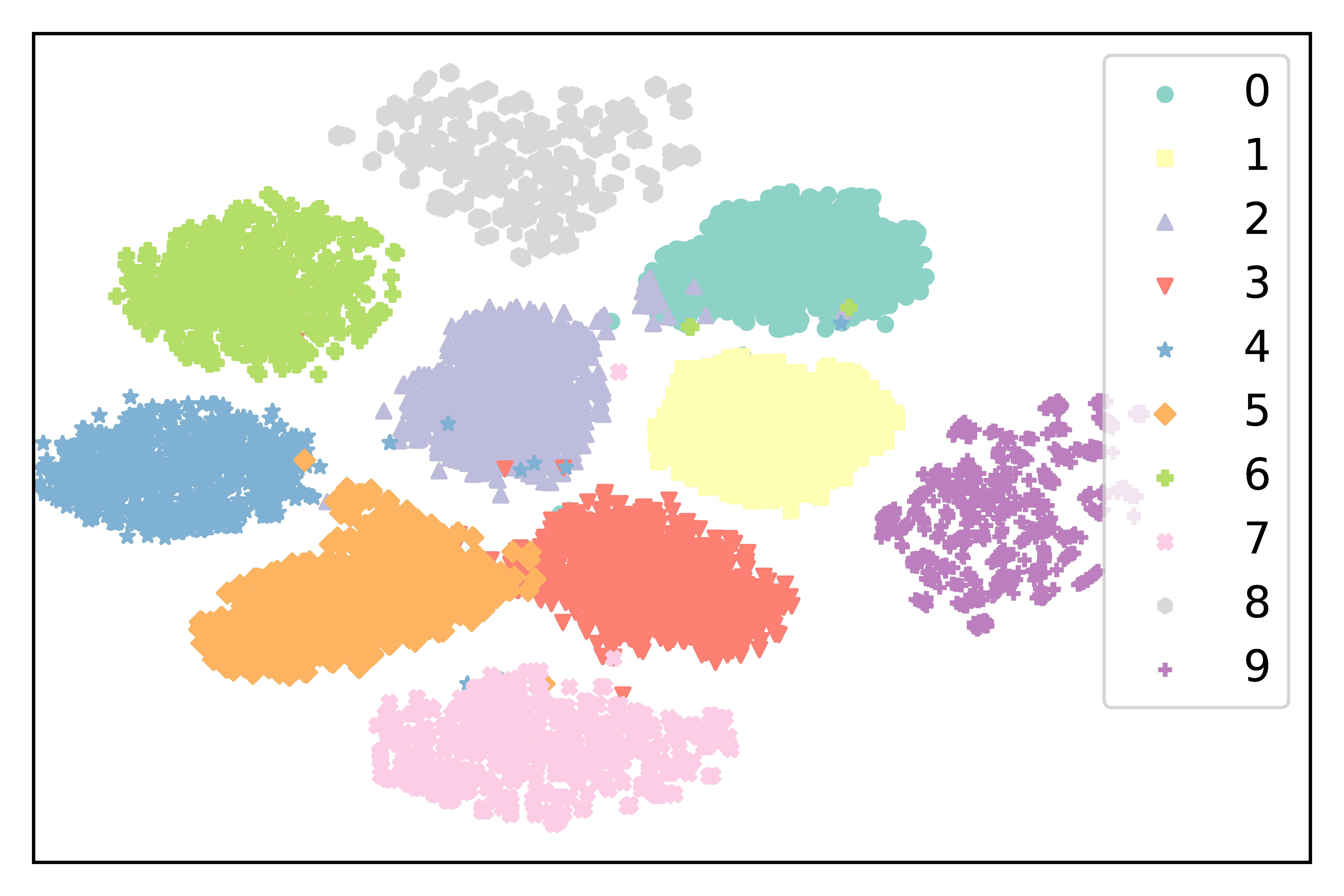}
  } \hspace{-1.2em}
  \subfigure[w. H2TF ($r=0.9$)\label{appfig:sub3}]{
    \includegraphics[width=0.2\linewidth]{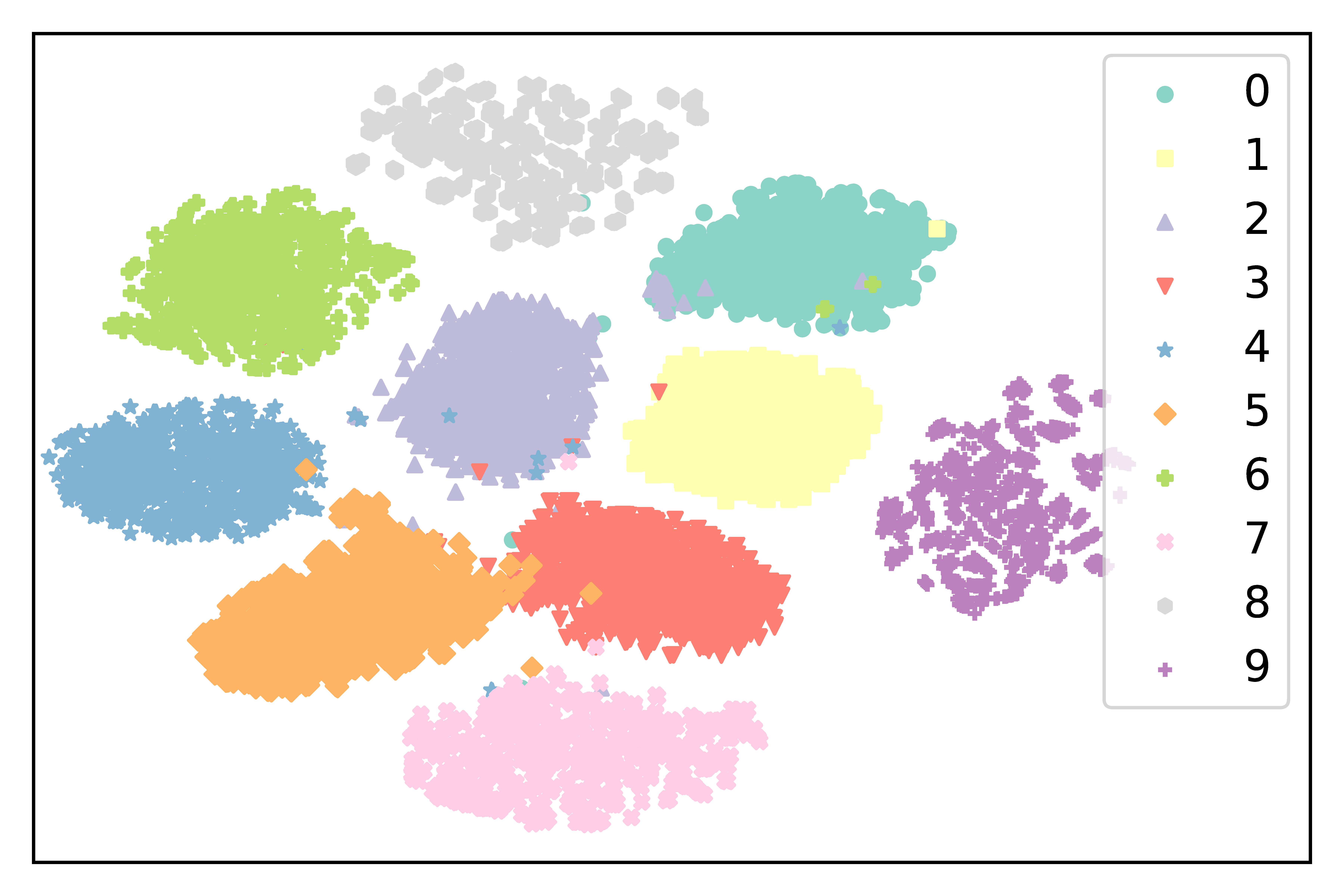}
  } \hspace{-1.2em}
  \subfigure[w. H2TF ($r=0.8$)\label{appfig:sub4}]{
    \includegraphics[width=0.2\linewidth]{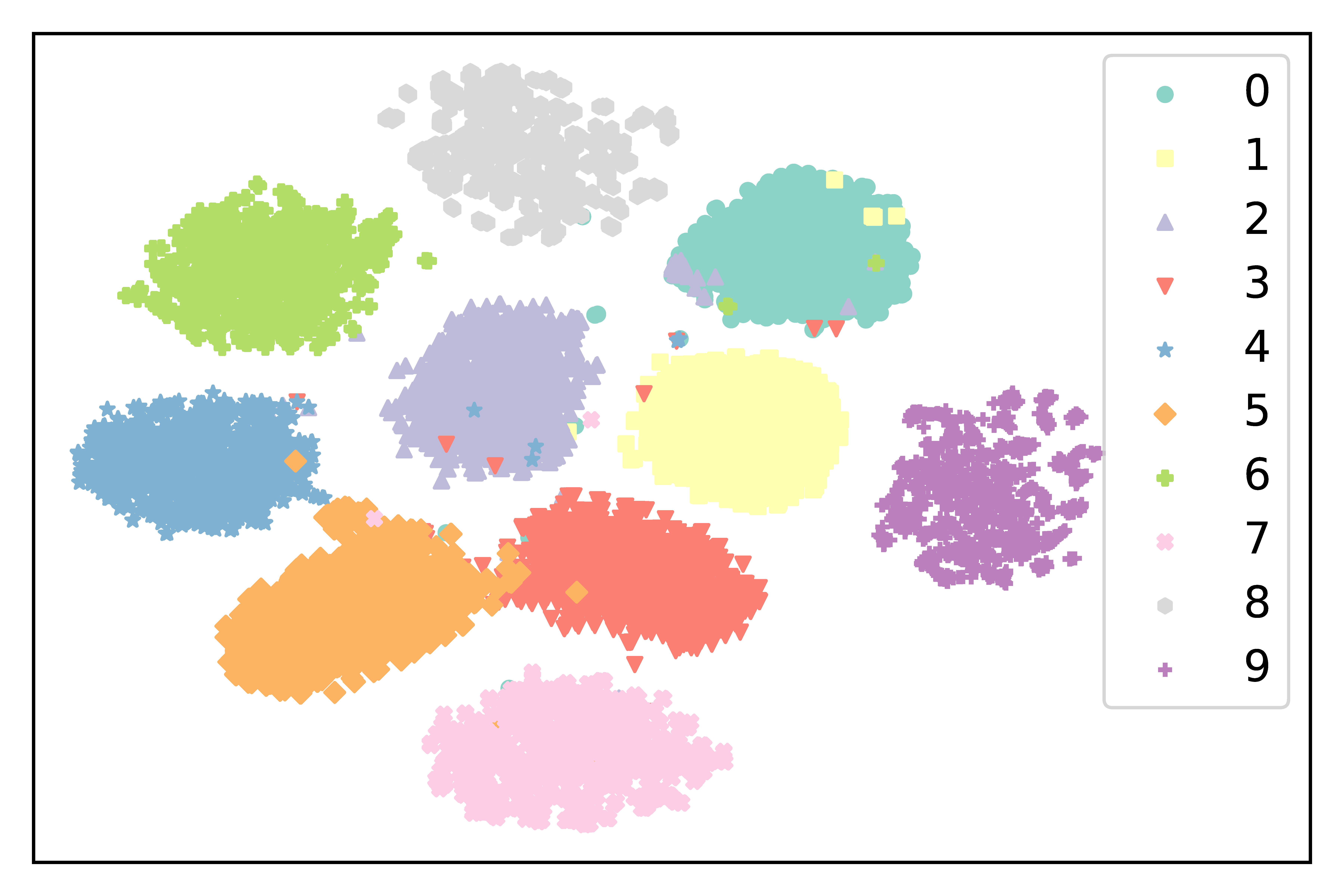}
  } \hspace{-1.2em}
  \subfigure[w. H2TF ($r=0.7$)\label{appfig:sub5}]{
    \includegraphics[width=0.2\linewidth]{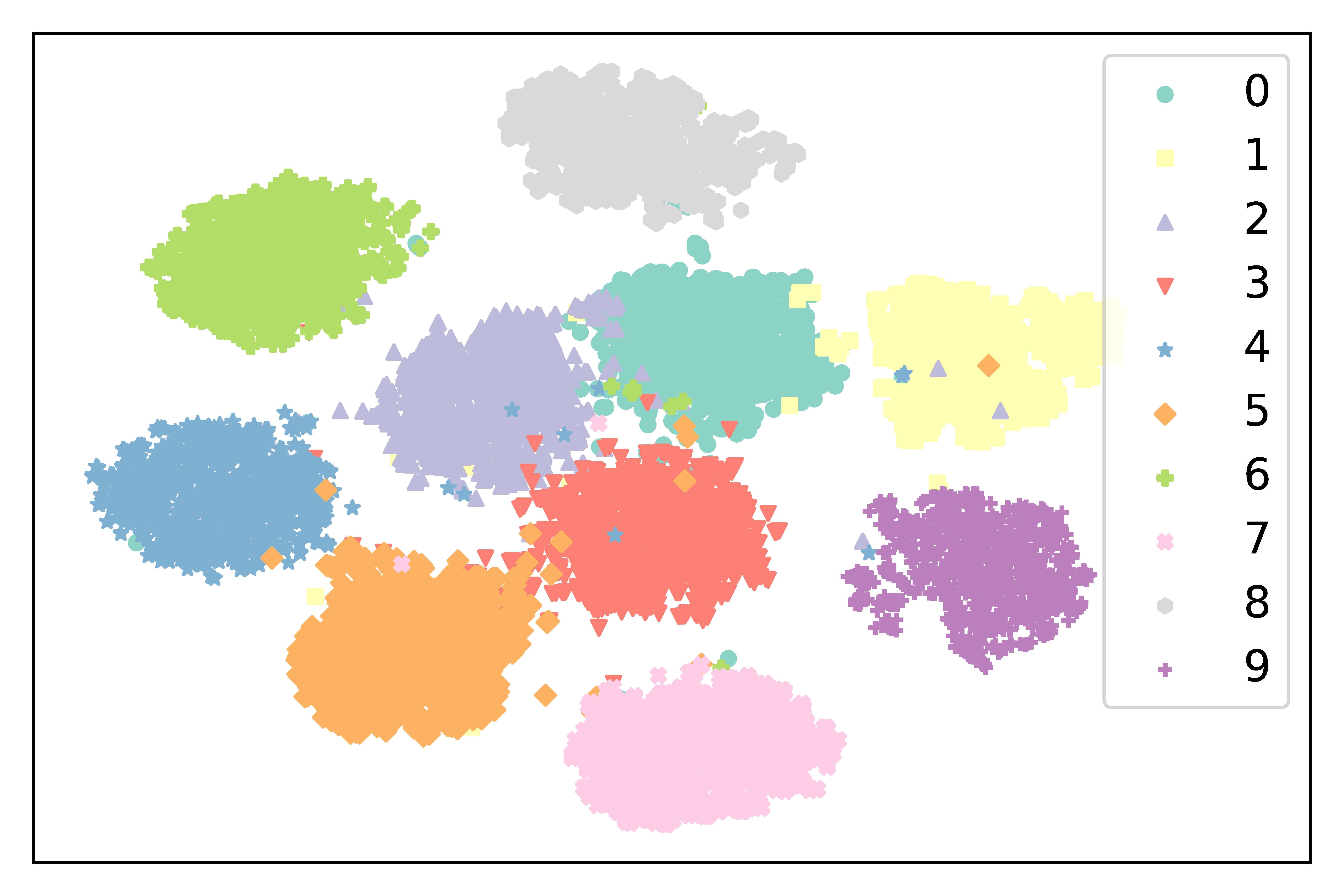}
  }
  \vspace{0.5em}  
  \hspace*{-2.4em}
  \subfigure[w. H2TF ($r=0.6$)\label{appfig:sub6}]{
    \includegraphics[width=0.2\linewidth]{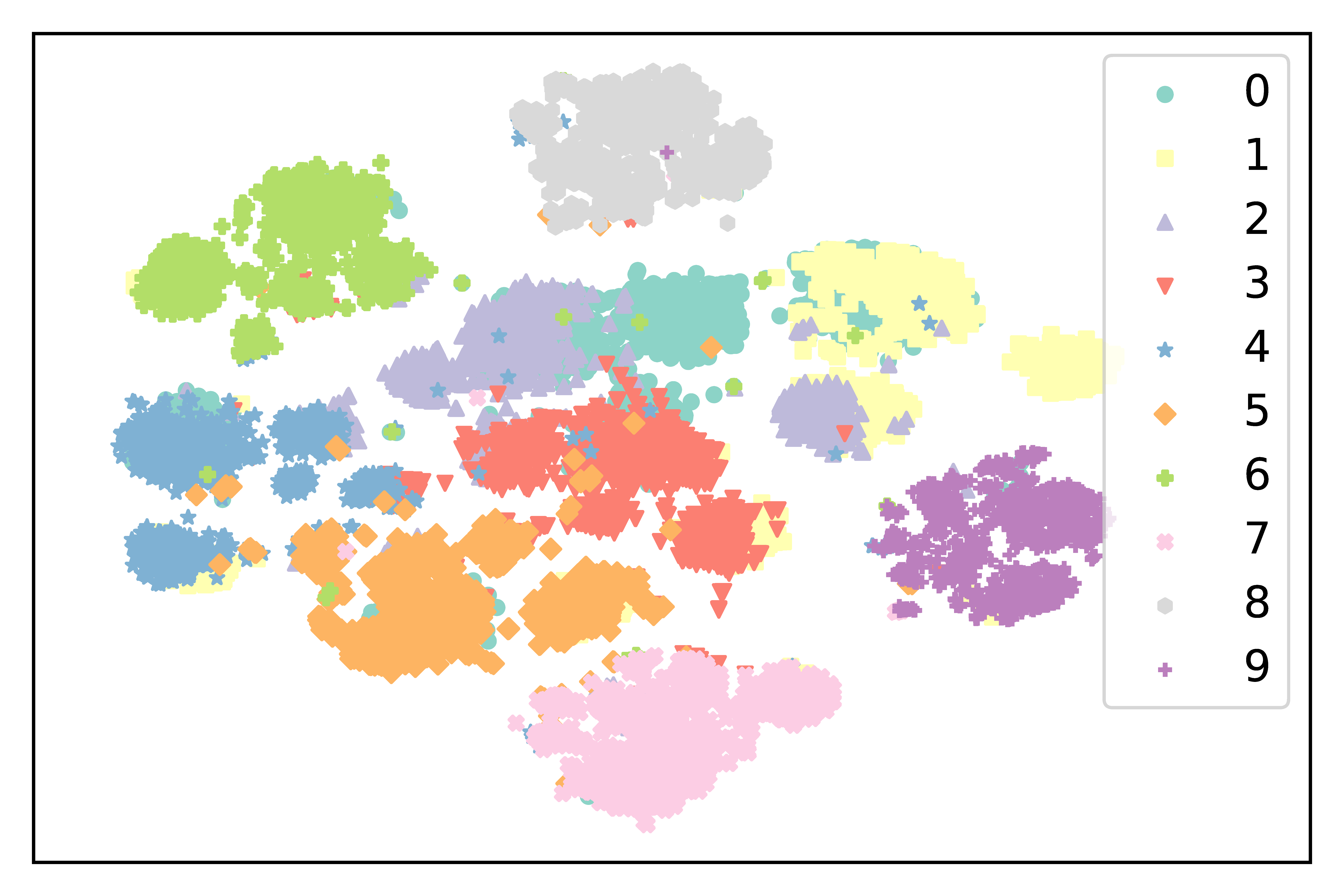}
  } \hspace{-1.2em}
  \subfigure[w. H2TF ($r=0.5$)\label{appfig:sub7}]{
    \includegraphics[width=0.2\linewidth]{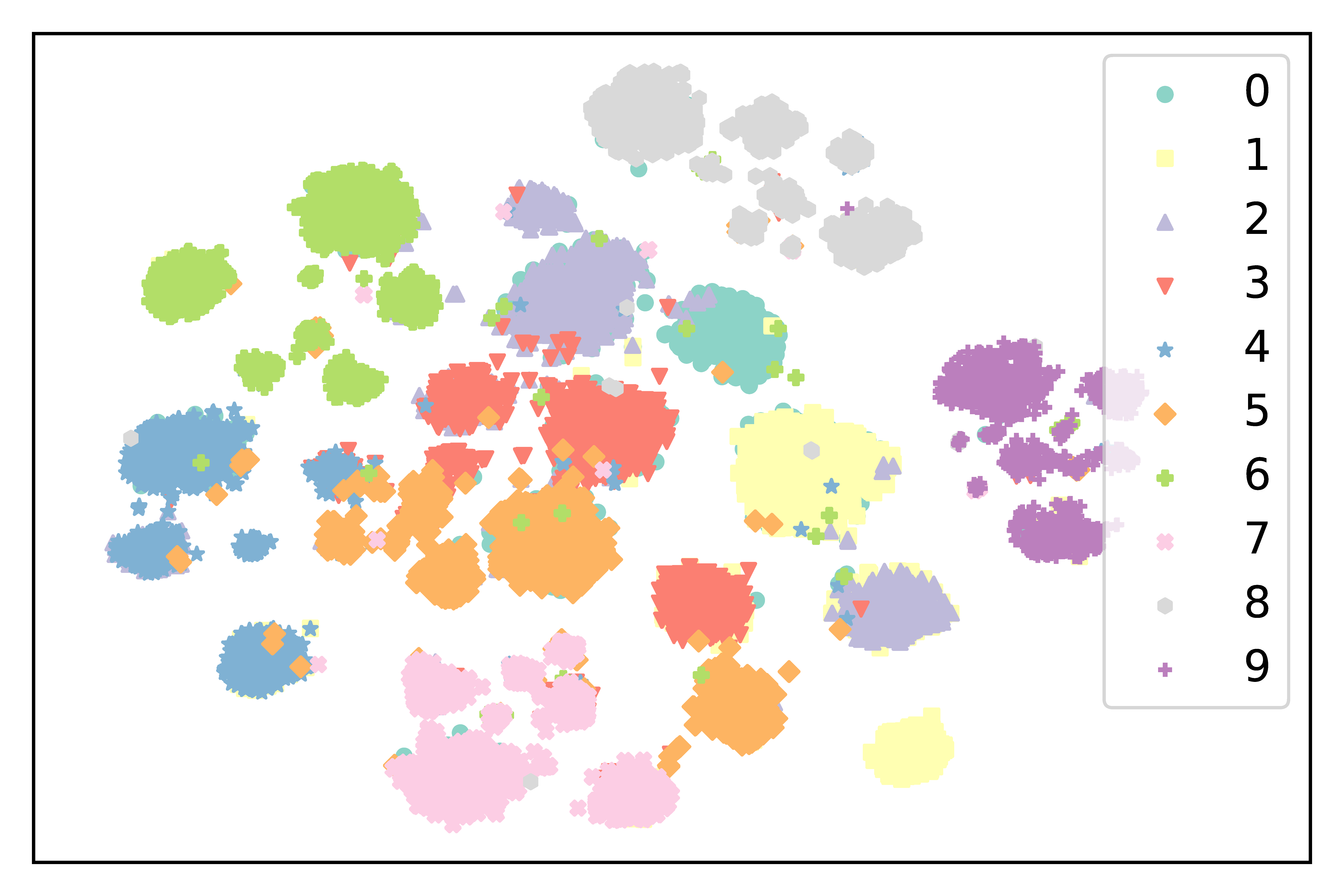}
  } \hspace{-1.2em}
  \subfigure[w. H2TF ($r \propto d_x$)\label{appfig:sub8}]{
    \includegraphics[width=0.2\linewidth]{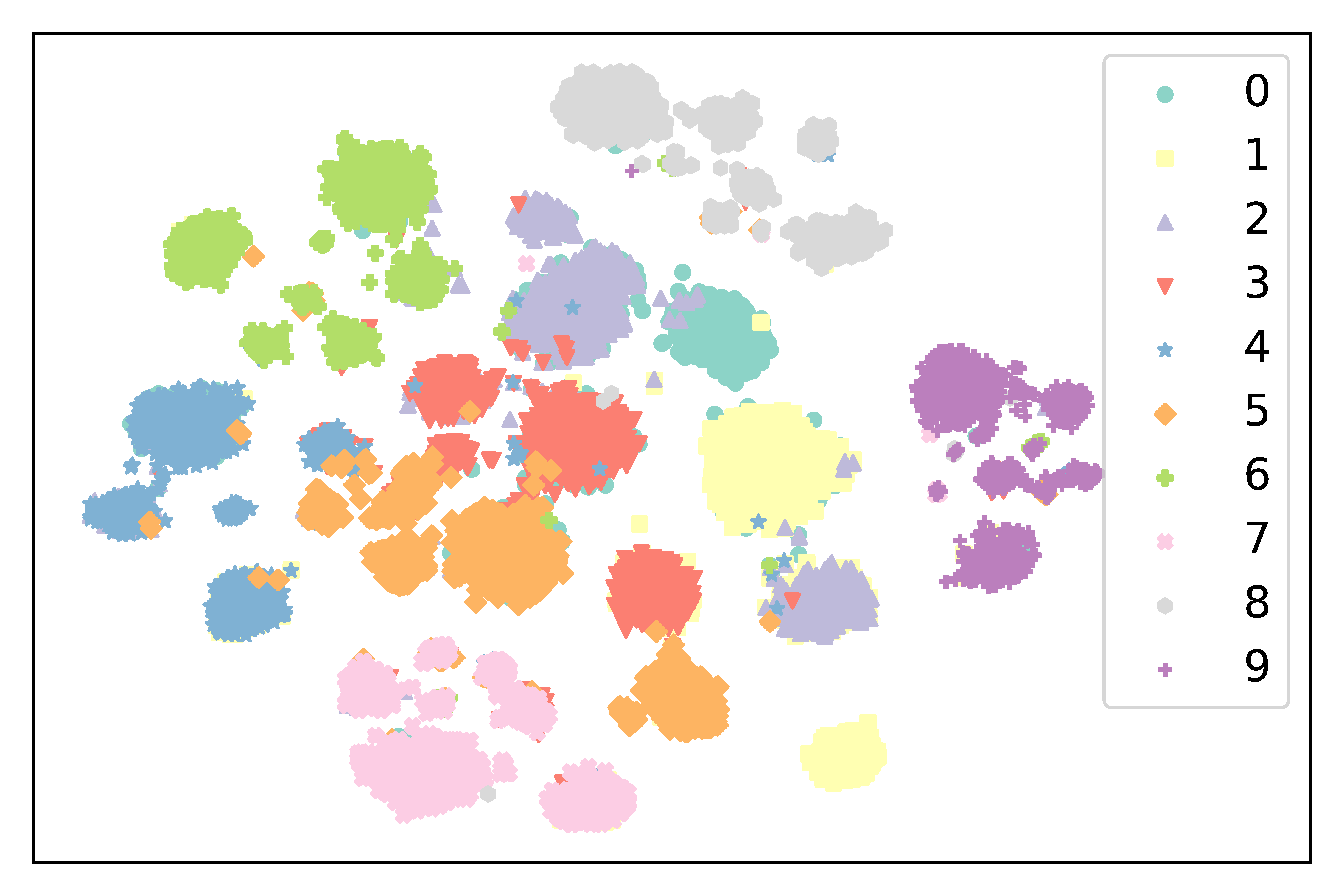} 
    }\hspace{-1.2em}
  \subfigure[w. H2TF ($r=0.4$)\label{appfig:sub9}]{
    \includegraphics[width=0.2\linewidth]{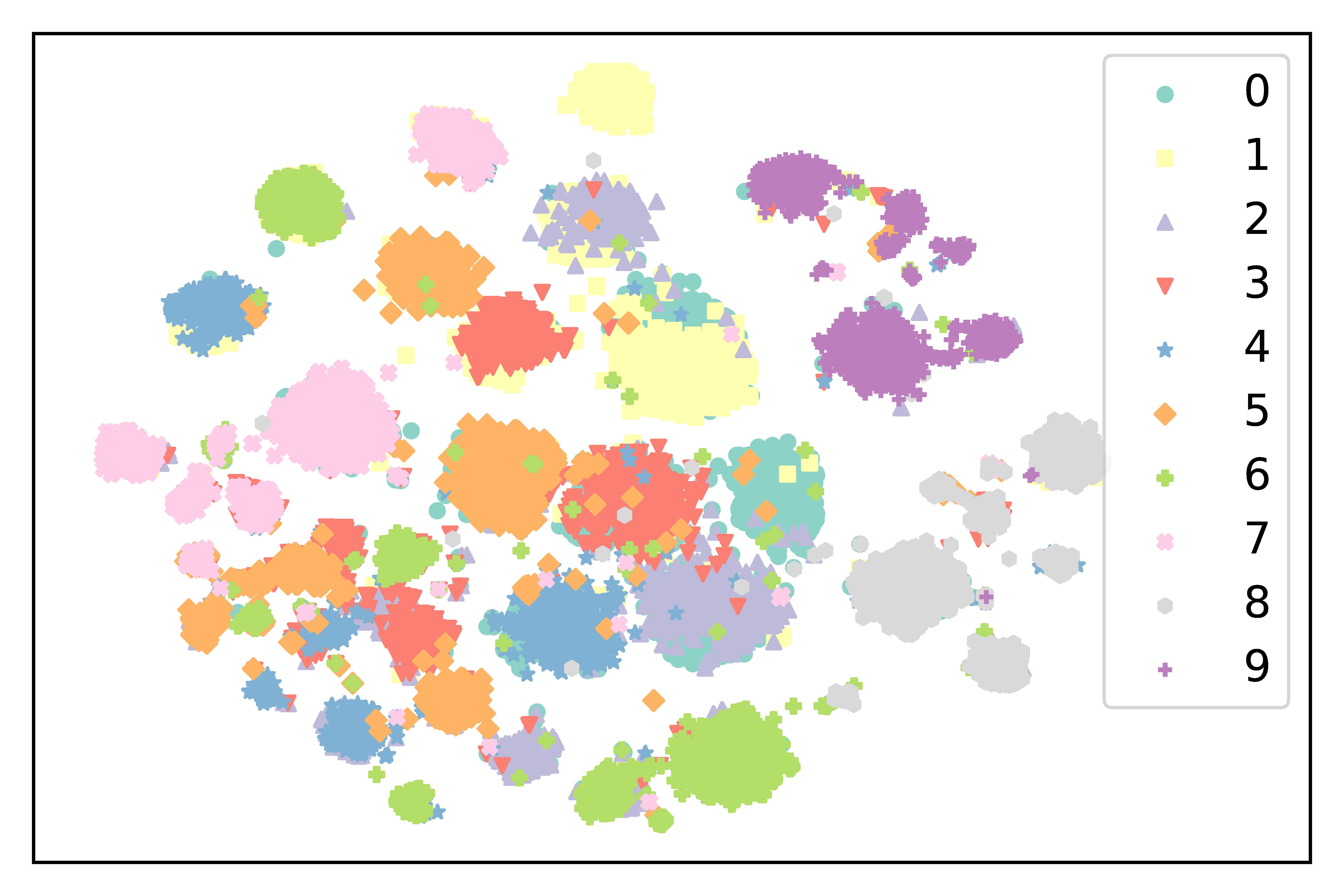}
  } \hspace{-1.2em}
  \subfigure[w. H2TF ($r=0.3$)\label{appfig:sub10}]{
    \includegraphics[width=0.2\linewidth]{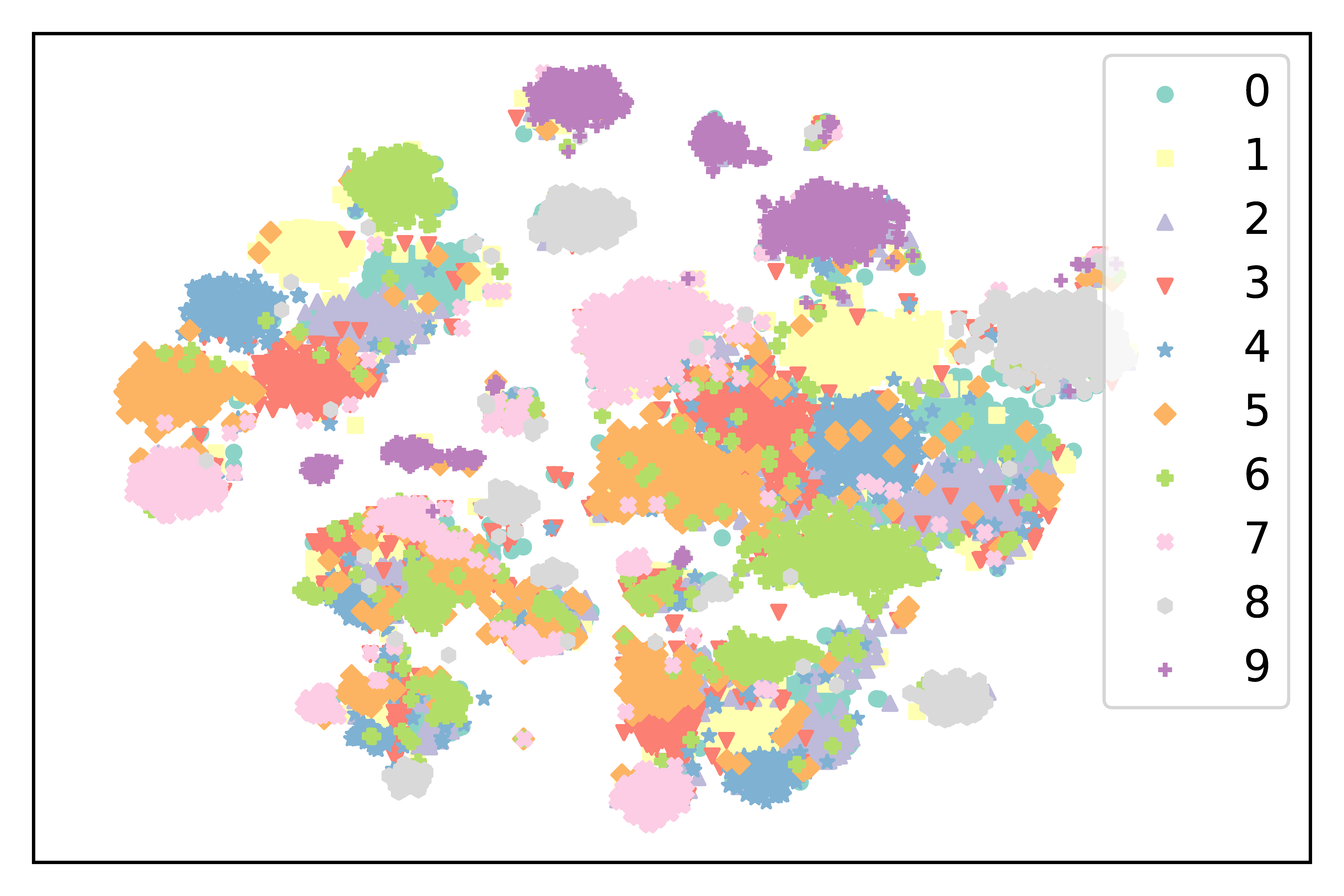}
  } 
  \caption{More t-SNE visualization Results of feature distributions with varying fusion ratios ($r$ in Eq.~\ref{eq:fh2t}) in H2TF strategy} \label{appfig:tsne_PIH2T}
  \vspace{-0.5em}
\end{figure*}

%% file: tab/cifar-bal.tex
\begin{table}[htpb]
\renewcommand{\thefootnote}{\fnsymbol{footnote}}
 \centering  
 \caption{Top-1 classification accuracy (\%) comparison on balanced datasets.}
 \label{tab:cifar_bal}
 \setlength{\tabcolsep}{15pt}
 \renewcommand{\arraystretch}{1.2}
  {
  \begin{tabular}{l| c c }
   \hlinew{1pt}
  Dataset  &CIFAR10 & CIFAR100 \\
  \hline
    \multicolumn{3}{c}{Single model, backbone: ResNet-32} \\
  \hline
  CE & 93.33 & 71.05  \\ 
  CE+PIF  & 93.87  ({\color{purple}{$\uparrow$ 0.21\%}}) & 72.06   ({\color{purple}{$\uparrow$ 1.01\%}})\\ 
  \hlinew{1pt}
 \end{tabular}
 }\vspace{-0.5em}
\end{table}